\documentclass[nohyperref]{article}
\usepackage[accepted]{icml2022}

\usepackage{microtype,bm}
\usepackage{graphicx}
\usepackage{booktabs} 

\usepackage{hyperref}

\usepackage{amsmath}
\usepackage{amssymb}
\usepackage{mathtools}
\usepackage{amsthm}
\usepackage{caption}

\usepackage[capitalize,noabbrev]{cleveref}

\theoremstyle{plain}

\usepackage[textsize=tiny]{todonotes}

\usepackage{color}
\usepackage[utf8]{inputenc} 
\usepackage[T1]{fontenc}    
\usepackage{natbib}
\usepackage{url}            
\usepackage{booktabs}       
\usepackage{amsfonts}       
\usepackage{nicefrac}       
\usepackage{microtype}      
\usepackage{booktabs}
\usepackage[rightcaption]{sidecap}
\sidecaptionvpos{figure}{c}
\usepackage{bbm}
\usepackage{latexsym, epsfig, amssymb, amsmath, amsthm, mathrsfs, booktabs,mathtools}
\usepackage{enumerate}
\usepackage{multirow}
\usepackage{multicol}
\usepackage{array,booktabs}
\usepackage{enumitem}
\usepackage{epstopdf}
\usepackage{subcaption}
\usepackage{wrapfig}
\usepackage{adjustbox}
\usepackage{multirow}
\usepackage{multicol}
\usepackage{array,booktabs}
\usepackage{algpseudocode} 
\usepackage{enumitem}
\usepackage{epstopdf}
\usepackage{xspace}
\usepackage{comment}
\usepackage{multirow}
\usepackage{makecell}

\usepackage{enumitem,amssymb}
\usepackage{pifont}
\newcommand{\cmark}{\ding{51}}%
\newcommand{\checkbox}{\rlap{$\square$}{\raisebox{2pt}{\large\hspace{1pt}\cmark}}%
\hspace{2.5pt}}

\newenvironment{myitemize}{\begin{list}{$\bullet$}
		{\setlength{\topsep}{1mm}
			\setlength{\itemsep}{0.25mm}
			\setlength{\parsep}{0.25mm}
			\setlength{\itemindent}{0mm}
			\setlength{\partopsep}{0mm}
			\setlength{\labelwidth}{15mm}
			\setlength{\leftmargin}{4mm}}}{\end{list}}

\renewcommand{\algorithmicrequire}{ \textbf{Input:}}
\renewcommand{\algorithmicensure}{ \textbf{Output:}}

\newcommand{\mc}{\mathcal}
\newcommand{\E}{\mb{E}}
\newcommand{\m}[1]{{\bm{#1}}}
\newcommand{\wh}[1]{{\widehat{#1}}}

\newcommand{\fedurl}{github.com/mc-nya/FedNest}
\renewcommand{\mc}[1]{\ensuremath{\mathcal{#1}}} 
\newcommand{\g}[1]{\mbox{\boldmath $#1$}}
\newcommand{\mb}[1]{{\mathbb{#1}}}

\newcommand{\minmax}{minimax}
\newcommand{\Minmax}{Minimax}
\DeclarePairedDelimiterX{\norm}[1]{\lVert}{\rVert}{#1}

\definecolor{darkred}{RGB}{150,0,0}
\definecolor{darkgreen}{RGB}{0,150,0}
\definecolor{darkblue}{RGB}{0,0,200}
\hypersetup{colorlinks=true, linkcolor=darkred, citecolor=darkgreen, urlcolor=darkblue}

\let\phi\varphi

\newtheorem{theorem}{Theorem}[section] 
\newtheorem{lemma}{Lemma}[section] 
\newtheorem{corollary}{Corollary}[section] 
\newtheorem{remark}{Remark}[section]
\newtheorem{example}{Example} 
\newtheorem{assumption}{Assumption}

\newcommand{\order}[1]{{\cal{O}}(#1)}
\newcommand{\fedblo}{\textsc{FedNest}\xspace}
\newcommand{\lfedblo}{\textsc{LFedNest}\xspace}
\newcommand{\Blo}{Bilevel}

\newcommand{\fedavg}{\textsc{FedAvg}\xspace}

\newcommand{\scaffold}{\textsc{SCAFFOLD}\xspace}

\newcommand{\fedinn}{\textsc{FedInn}\xspace}
\newcommand{\fedout}{\textsc{FedOut}\xspace}
\newcommand{\lfedinn}{\textsc{LFedInn}\xspace}
\newcommand{\lfedout}{\textsc{LFedOut}\xspace}
\newcommand{\fednest}{\textsc{FedNest}\xspace}
\newcommand{\fedhess}{\textsc{FedIHGP}\xspace}

\newcommand{\vct}[1]{\bm{#1}}

\newcommand{\x}{\vct{x}}
\newcommand{\y}{\vct{y}}

\newcommand{\funcin}{\mathbb{G}}
\newcommand{\funco}{\mathbb{F}}

\icmltitlerunning{\fedblo: Federated Bilevel, Minimax, and Compositional Optimization}

\begin{document}

\twocolumn[
\icmltitle{\fedblo: Federated~\Blo,~\Minmax, and Compositional~Optimization}
\icmlsetsymbol{equal}{*}
\begin{icmlauthorlist}
\icmlauthor{Davoud Ataee Tarzanagh}{umich}
\icmlauthor{Mingchen Li}{ucr}
\icmlauthor{Christos Thrampoulidis}{ubc}
\icmlauthor{Samet Oymak}{ucr}
\end{icmlauthorlist}
\icmlaffiliation{umich}{Email: \texttt{tarzanaq@umich.edu},
University of Michigan}
\icmlaffiliation{ucr}{Emails: \texttt{\{mli176@,oymak@ece.\}ucr.edu},
University of California, Riverside}
\icmlaffiliation{ubc}{Email: \texttt{cthrampo@ece.ubc.ca},
University of British Columbia}
\icmlcorrespondingauthor{Davoud Ataee Tarzanagh}{\texttt{tarzanaq@umich.edu}}
\icmlkeywords{Machine Learning, ICML}
\vskip 0.3in
]
\printAffiliationsAndNotice{}
\newcommand{\magenta}[1]{\textcolor{magenta}{#1}}
\newcommand{\err}[1]{\textcolor{red}{?~#1~?}}
\newcommand{\ct}[1]{\textcolor{magenta}{C: #1}}

\begin{abstract}
Standard federated optimization methods successfully apply to stochastic problems with \textit{single-level} structure. However, many contemporary ML problems -- including adversarial robustness, hyperparameter tuning, actor-critic -- fall under nested bilevel programming that subsumes minimax and compositional optimization. In this work, we propose \fedblo: A federated alternating stochastic gradient method to address general nested problems. We establish provable convergence rates for \fedblo in the presence of heterogeneous data and introduce variations for bilevel, minimax, and compositional optimization. \fedblo introduces multiple innovations including federated hypergradient computation and variance reduction to address inner-level heterogeneity. We complement our theory with experiments on hyperparameter \& hyper-representation learning and minimax optimization that demonstrate the benefits of our method in practice.  Code is available at \url{https://github.com/ucr-optml/FedNest}.
\end{abstract}
\vspace{-10pt}
\section{Introduction}
In the federated learning (FL) paradigm, multiple clients cooperate to learn a model under the orchestration of a central server \citep{mcmahan17fedavg} without directly exchanging local client data with the server or other clients. The locality of data distinguishes FL from traditional distributed optimization and also motivates new methodologies to address heterogeneous data across clients. Additionally, cross-device FL across many edge devices presents additional challenges since only a small fraction of clients participate in each round, and clients cannot maintain \emph{state} across rounds~\citep{kairouz2019advances}. 

Traditional distributed SGD methods are often unsuitable in FL and incur high communication costs. To overcome this issue, popular FL methods, such as \fedavg~\citep{mcmahan17fedavg}, use local client updates, i.e. clients update their models multiple times before communicating with the server (aka,~local SGD). Although \fedavg has seen great success, recent works have exposed convergence issues in certain settings~\citep{karimireddy2020scaffold,hsu2019measuring}. This is due to a variety of factors, including \emph{client drift}, where local models move away from globally optimal models due to objective and/or systems heterogeneity. 
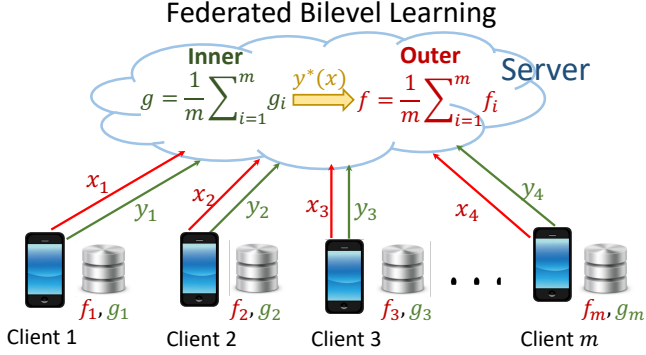
\begin{figure*}[t]
\centering
\begin{tikzpicture}
    \node at (-3.2,-.4)[anchor=west] {\includegraphics[width=\linewidth/2]{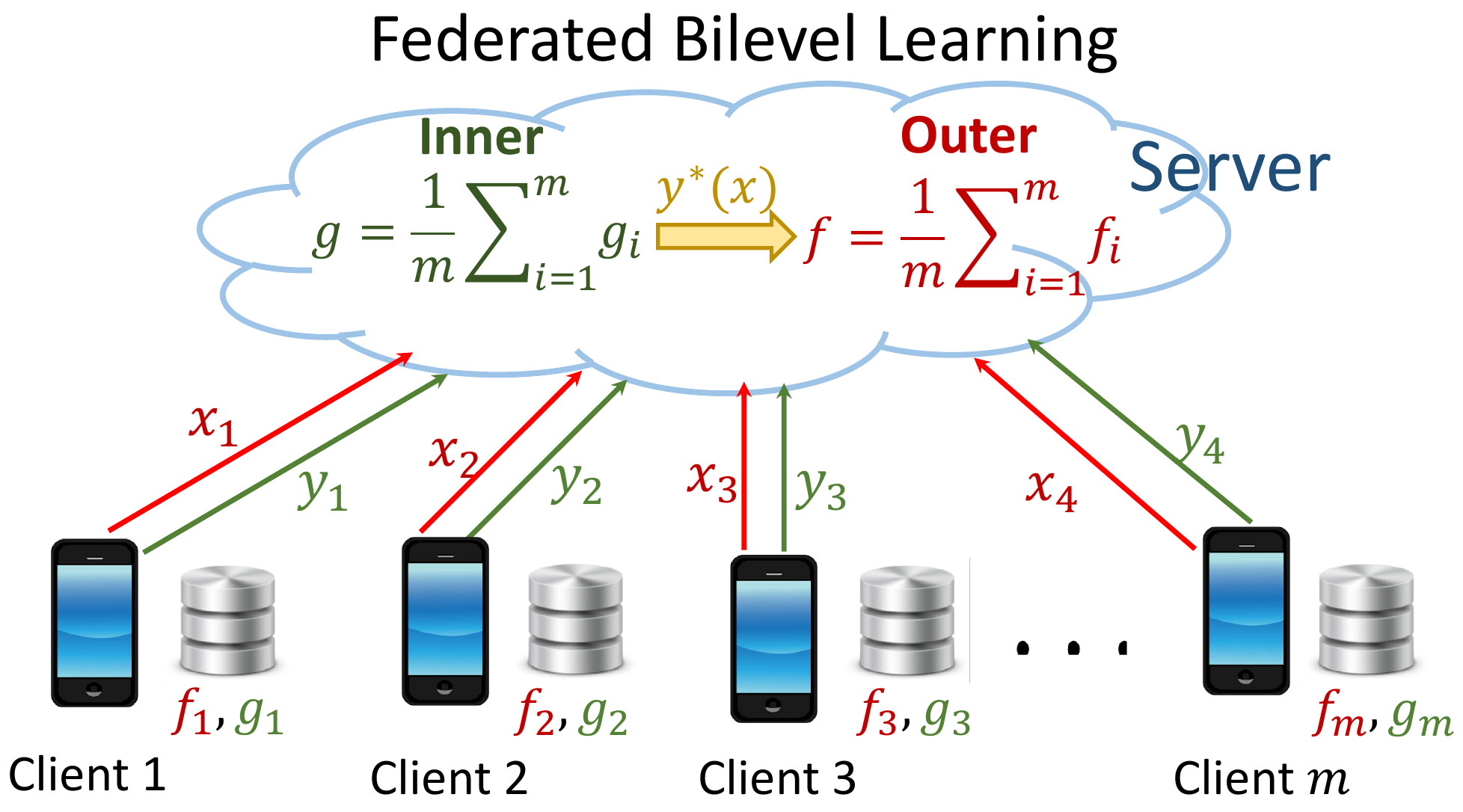}};
\node at (6,1.8)[anchor=west] [scale=1]{{Communication-efficiency:}};
    \node at (6.3,1.4)[anchor=west] {\checkbox \fedhess avoids explicit Hessian};
    \node at (6.3,1.0)[anchor=west] {\checkbox \lfedblo for local hypergradients};
    
    \node at (6,0.5)[anchor=west] [scale=1]{{Client heterogeneity:}};
    \node at (6.3,0.15)[anchor=west] {\checkbox \fedinn avoids client drift};
    
    \node at (6,-0.3)[anchor=west] [scale=1]{{Finite sample bilevel theory:}};
    \node at (6.3,-0.7)[anchor=west] {\checkbox Stochastic inner \& outer analysis};
    
    \node at (6,-1.1)[anchor=west] [scale=1]{{Specific nested optimization problems:}};
    \node at (6.3,-1.5)[anchor=west] {\checkbox {Bilevel}};
    \node at (6.3,-1.9)[anchor=west] {\checkbox {Minimax}};
    \node at (6.3,-2.3)[anchor=west] {\checkbox {Compositional}};
\end{tikzpicture}
\vspace{-10pt}\caption{Depiction of federated bilevel nested optimization and high-level summary of \fedblo (Algorithm~\ref{alg:fednest}). At outer loop, \fedhess uses multiple rounds of matrix-vector products to facilitate hypergradient computation while only communicating vectors. At inner loop, \fedinn uses \textsc{FedSVRG} to avoid client drift and find the unique global minima. Both are crucial for establishing provable convergence of \fedblo.}
\label{fig_main}
\vspace{-5pt}
\end{figure*}

Existing FL methods, such as \fedavg, are widely applied to stochastic problems with \emph{single-level} structure. Instead, many machine learning tasks -- such as adversarial learning \citep{madry2017towards}, meta learning~\citep{bertinetto2018meta}, hyperparameter optimization~\citep{franceschi2018bilevel}, reinforcement/imitation learning~\citep{wu2020finite,arora2020provable}, and neural architecture search~\citep{liu2018darts} -- admit \emph{nested} formulations that go beyond the standard single-level structure. Towards addressing such nested problems, bilevel optimization has received significant attention in the recent literature
~\citep{ghadimi2018approximation,hong2020two,ji2021bilevel}; albeit in non-FL settings. On the other hand, federated versions have been elusive perhaps due to the additional challenges surrounding heterogeneity, communication, and inverse Hessian approximation. 

\noindent\textbf{Contributions:} This paper addresses these challenges and develops~\textbf{\fedblo}: A federated machinery for nested problems with provable convergence and lightweight communication. \fedblo is composed of \textbf{\fedinn}: a federated stochastic variance reduction algorithm (\textsc{FedSVRG}) to solve the inner problem while avoiding client drift, and \textbf{\fedout}: a communication-efficient federated hypergradient algorithm for solving the outer problem. Importantly, we allow both inner \& outer objectives to be finite sums over heterogeneous client functions. \fedblo runs a variant of \textsc{FedSVRG} on inner \& outer variables in an alternating fashion as outlined in Algorithm~\ref{alg:fednest}. We make multiple algorithmic and theoretical contributions summarized below.%
\vspace{-.2cm}
\begin{myitemize}
\item \textbf{The variance reduction} of \fedinn enables robustness in the sense that local models converge to the globally optimal inner model {despite client drift/heterogeneity} unlike \fedavg. While \fedinn is similar to \textsc{FedSVRG}~\cite{konevcny2016federated} and \textsc{FedLin}~\cite{mitra2021linear}, we make two key contributions: (i) We leverage the global convergence of \fedinn to ensure accurate hypergradient computation which is crucial for our bilevel proof.  (ii) We establish new convergence guarantees for single-level stochastic non-convex \textsc{FedSVRG}, which are then integrated within our \fedout.
 \item \textbf{Communication efficient bilevel optimization}: Within \fedout, we develop an efficient federated method for hypergradient estimation that bypass Hessian computation. Our approach approximates the global Inverse Hessian-Gradient-Product (IHGP) via computation of matrix-vector products over few communication rounds.
 \item \textbf{\lfedblo:} To further improve communication efficiency, we additionally propose a Light-\fedblo algorithm, which computes hypergradients locally and only needs a single communication round for the outer update. Experiments reveal that \lfedblo becomes very competitive as client functions become more homogeneous.
 \item \textbf{Unified federated nested theory:} We specialize our bilevel results to \emph{minimax} and \emph{compositional} optimization with emphasis on the former. For these, \fedblo significantly simplifies and leads to faster convergence. Importantly, our results are on par with the state-of-the-art non-federated guarantees for nested optimization literature without additional assumptions (Table~\ref{table:blo2:results}).
 \item We provide extensive \textbf{numerical experiments} on bilevel and minimax optimization problems. These demonstrate the benefits of \fedblo, efficiency of \lfedblo, and shed light on tradeoffs surrounding communication, computation, and heterogeneity.  
 \end{myitemize}

\begin{table}[t]\vspace{-7pt}
\begin{minipage}[b]{1\linewidth}\centering
\scalebox{0.75}{\hspace*{-5pt}
\begin{tabular}{|l|c|c|c | c| }
\multicolumn{5}{c}{\large{Stochastic Bilevel Optimization}} \\
 \cline{3-5}
\multicolumn{2}{c}{}&\multicolumn{3}{|c|}{Non-Federated}\\
 \cline{2-5}
\multicolumn{1}{c|}{}&\textbf{\fednest} & \textbf{ALSET}  & \textbf{BSA} &  \textbf{TTSA}  \\
\hline
\textbf{batch size} &  \multicolumn{4}{c|}{${\cal O}(1)$}  \\
\hline
 \makecell{\makecell{\textbf{samples in $\xi$}\\ \textbf{samples in $\zeta$}}}  & \makecell{ ${\cal O}(\kappa_g^5\epsilon^{-2})$ \\ ${\cal O}(\kappa_g^9 \epsilon^{-2})$ }& \makecell{ ${\cal O}( \kappa_g^5\epsilon^{-2})$ \\ ${\cal O}( \kappa_g^9\epsilon^{-2})$ } & \makecell{ ${\cal O}( \kappa_g^6\epsilon^{-2})$  \\ ${\cal O}(\kappa_g^9\epsilon^{-3})$ }  & \makecell{ ${\cal O}( \kappa_g^p\epsilon^{-2.5})$  \\ ${\cal O}( \kappa_g^p\epsilon^{-2.5})$ }
 \\ 
 \hline
 \multicolumn{5}{c}{} \\
 \multicolumn{5}{c}{\large{Stochastic Minimax Optimization}} \\
 \cline{3-5}
\multicolumn{2}{c|}{}&\multicolumn{3}{c|}{Non-Federated}\\
 \cline{2-5}
\multicolumn{1}{c|}{}&\textbf{\fednest}	& \textbf{ALSET}  & \textbf{SGDA} &  \textbf{SMD}      \\ 
\hline
\textbf{batch size} & \mc{O}(1) & ${\cal O}(1)$   & ${\cal O}(\epsilon^{-1})$  & N.A.  \\ 
\hline
\textbf{samples }  & \multicolumn{4}{c|}{${\cal O}(\kappa^3_f\epsilon^{-2})$}
\\ 
 \hline
 \multicolumn{5}{c}{} \\
\multicolumn{5}{c}{\large{Stochastic Compositional Optimization}} \\
 \cline{3-5}
\multicolumn{2}{c}{}&\multicolumn{3}{|c|}{Non-Federated}\\
 \cline{2-5}
\multicolumn{1}{c|}{}&\textbf{\fednest}	& \textbf{ALSET}  & \textbf{SCGD} & \textbf{NASA}      
\\ \hline
\textbf{batch size} & \multicolumn{4}{c|}{${\cal O}(1)$} 
\\ \hline
\textbf{samples}  &$ {\cal O}(\epsilon^{-2})$& ${\cal O}(\epsilon^{-2})$  & ${\cal O}(\epsilon^{-4})$  & ${\cal O}(\epsilon^{-2})$ \\ 
\hline
 \end{tabular}
}
\vspace{-0.1cm}
\caption{\small{Sample complexity of \fedblo and comparable non-FL methods to find an $\epsilon$-stationary point of $f$: $\kappa_g:=\ell_{g,1}/\mu_g$ and $\kappa_f:=\ell_{f,1}/\mu_f$.  $\kappa^p_g$ denotes a polynomial function of $\kappa_g$.  {ALSET} \citep{chen2021closing}, {BSA} \citep{ghadimi2018approximation}, {TTSA} \citep{hong2020two}, SGDA \citep{lin2020gradient}, SMD \citep{rafique2021weakly}, SCGD~\citep{wang2017stochastic}, and NASA~\citep{ghadimi2020single}. 
}}
\label{table:blo2:results}
\vspace{-0.35cm}
\end{minipage}
\end{table}

%
\section{Federated Nested Problems \& \fedblo}\label{sec:fedavg:full}
We will first provide the background on bilevel nested problems and then introduce our general federated method.

\noindent\textbf{Notation.}
$\mb{N}$ and $\mb{R}$ denotes the set of natural and real numbers, respectively. 
For a differentiable function $h(\m{x},\m{y}):\mb{R}^{d_1}\times\mb{R}^{d_2}\rightarrow\mb{R}$ in which $\m{y}=\m{y}(\m{x}):\mb{R}^{d_1}\rightarrow\mb{R}^{d_2}$, we denote $\nabla{h}\in\mb{R}^{d_1}$ the gradient of $h$ as a function of $\m{x}$ and  $\nabla_{\m{x}}{h}$, $\nabla_{\m{y}} h$ the partial derivatives of $h$ with respect to $\m{x}$ and $\m{y}$, respectively. We let $\nabla_{\m{xy}}^2 h$ and $\nabla_{\m{y}}^2 h$ denote the Jacobian and Hessian of $h$, respectively. We consider FL optimization over $m$ clients and we denote $\mathcal{S}=\{1,\ldots,m\}$. For vectors $\m{v} \in \mb{R}^d$ and matrix $\m{M} \in \mb{R}^{d \times d}$, we denote $\|\m{v}\|$ and $\|\m{M}\|$ the respective Euclidean and spectral norms.
\subsection{Preliminaries on Federated Nested Optimization}
In \textbf{federated bilevel learning}, we consider the  following nested optimization problem as depicted in Figure \ref{fig_main}:
\begin{subequations}\label{fedblo:prob}
\begin{align}
\begin{array}{ll}
\underset{\m{x} \in \mb{R}^{{d}_1}}{\min} &
\begin{array}{c}
f(\m{x})=\frac{1}{m} \sum_{i=1}^{m} f_{i}\left(\m{x},\m{y}^*(\m{x})\right) 
\end{array}\\
\text{subj.~to} & \begin{array}[t]{l} \m{y}^*(\m{\m{x}})
\in \underset{ \m{y}\in \mb{R}^{{d}_2}}{\textnormal{argmin}}~~\frac{1}{m}\sum_{i=1}^{m} g_i\left(\m{x},\m{y}\right). 
\end{array}
\end{array}
\end{align}
Recall that $m$ is the number of clients. Here, to model objective heterogeneity, each client $i$ is allowed to have its own individual outer \& inner functions $(f_i,g_i)$. Moreover, we consider a general stochastic oracle model, access to local functions $(f_i,g_i)$ is via stochastic sampling as follows:
\begin{align}
\nonumber 
 f_i(\m{x},\m{y}^*(\m{x})) &:= \mb{E}_{\xi \sim \mc{C}_i}\left[f_i(\m{x}, \m{y}^*(\m{x}); \xi)\right],\\g_i(\m{x},\m{y}) &:= \mb{E}_{\zeta \sim \mc{D}_i}\left[g_i(\m{x}, \m{y}; \zeta)\right],
\end{align}
\end{subequations}
where $(\xi, \zeta)\sim(\mc{C}_i, \mc{D}_i)$ are  outer/inner sampling distributions for the $i^{\text{th}}$ client. We emphasize that for $i \neq j$, the tuples $(f_i,g_i,\mc{C}_i, \mc{D}_i)$ and $(f_j,g_j,\mc{C}_j, \mc{D}_j)$ can  be different.

\begin{example}[Hyperparameter tuning] Each client has local validation and training datasets associated with objectives $(f_i,g_i)_{i=1}^m$ corresponding to validation and training losses, respectively. The goal is finding hyper-parameters $\x$ that lead to learning model parameters $\y$ that minimize the (global) validation loss. 
\end{example}
%

\begin{algorithm}[t]
\caption{\pmb{\fedblo}}
\label{alg:fedblo}
\begin{algorithmic}[1]
\State \textbf{Inputs:} $K,T \in \mb{N}$;~~$(\m{x}^0, \m{y}^0)\in \mb{R}^{d_1+d_{2}}$;~~\textbf{\fedinn},\\
\hspace{33pt}\textbf{\fedout} with stepsizes $\{(\alpha^k, \beta^k)\}_{k=0}^{K-1}$
\For{$k = 0, \cdots, K-1$}
\State $\m{y}^{k,0}=\m{y}^k$
\For{$t = 0, \cdots, T-1$}
\State $\m{y}^{k,t+1}=\pmb{\fedinn}\left(\m{x}^k,\m{y}^{k,t},\beta^k \right)$~~~
\EndFor
\State $\m{y}^{k+1}=\m{y}^{k,T}$
\State $\m{x}^{k+1}=\pmb{\fedout}\left(\m{x}^k,\m{y}^{k+1}, \alpha^k \right)$
\EndFor
\end{algorithmic}\label{alg:fednest}
\end{algorithm}
%
The stochastic bilevel problem \eqref{fedblo:prob} subsumes two popular problem classes with the nested structure: \emph{Stochastic MiniMax} \& \emph{Stochastic Compositional}. Therefore, results on the general nested problem \eqref{fedblo:prob}  also imply the results in these special cases. Below, we briefly describe them.

\textbf{Minimax optimization.} If $g_i(\x,\y;\zeta):=-f_i(\x,\y;\xi)$ for all $i \in \mc{S}$, the stochastic bilevel problem \eqref{fedblo:prob} reduces to the stochastic minimax problem
\begin{align}\label{fedminmax:prob}
\min_{\m{x}\in\mb{R}^{d_1}}~f(\m{x}):=\frac{1}{m} \max_{\m{y}\in\mb{R}^{d_2}}\sum_{i=1}^{m} \E[f_{i}\left(\m{x},\m{y}; \xi\right)].
\end{align}
Motivated by applications in fair beamforming, training generative-adversarial networks (GANs) and robust machine learning, significant efforts have been made for solving \eqref{fedminmax:prob}~including \citep{daskalakis2018limit,gidel2018variational,mokhtari2020unified,thekumparampil2019efficient}.

\begin{example}[GANs] We train a generative model $g_{\x}(\cdot)$ and an adversarial model $a_{\y}(\cdot)$ using client datasets ${\mc{C}}_i$. The local functions may for example take the form $f_i(\x,\y)=\E_{s\sim {\mc{C}}_i}\{\log a_{\y}(s)\} +\E_{z\sim {\cal{D}}_{\text{noise}}}\{\log[1-a_{\y}(g_{\x}(z))]\}$.
\end{example}

\textbf{Compositional optimization.} Suppose $f_i(\m{x},\m{y};\xi):=f_i(\m{y};\xi)$ and $g_i$ is quadratic in $\m{y}$ given as $g_i(\m{x},\m{y};\zeta):=\|\m{y}-\m{r}_i(\m{x};\zeta)\|^2$. Then, the bilevel problem \eqref{fedblo:prob} reduces to 
\begin{align}\label{fedcompos:prob2}
\begin{array}{ll}
\underset{\m{x} \in \mb{R}^{{d}_1}}{\min} &
\begin{array}{c}
f(\m{x})=\frac{1}{m} \sum_{i=1}^{m} f_{i}\left(\m{y}^*(\m{x})\right) 
\end{array}\\
\text{subj.~to} & \begin{array}[t]{l} \m{y}^*(\m{\m{x}})=\underset{ \m{y}\in \mb{R}^{{d}_2}}{\textnormal{argmin}}~~\frac{1}{m}\sum_{i=1}^{m} g_i\left(\m{x},\m{y}\right) 
\end{array}
\end{array}
\end{align}
with $f_i(\m{y}^*(\m{x})) := \mb{E}_{\xi \sim \mc{C}_i}[f_i( \m{y}^*(\m{x}); \xi)]$ and $g_i(\m{x},\m{y}):= \mb{E}_{\zeta \sim \mc{D}_i}[g_i(\m{x},\m{y};\zeta)]$.
Optimization problems in the form of \eqref{fedcompos:prob2} occur for example in model agnostic meta-learning and policy evaluation in reinforcement learning \citep{finn2017model,ji2020theoretical,dai2017learning,wang2017stochastic}. 

\noindent \textbf{Assumptions.} Let $\m{z}=(\m{x},\m{y}) \in \mb{R}^{d_1+d_2}$. Throughout, we make the following assumptions on inner/outer objectives. 
\vspace{0.1cm}
%
\begin{assumption}[Well-behaved objectives]\label{assu:f} 
For all $i \in [m]$:
\begin{enumerate}[label={\textnormal{\textbf{(A\arabic*})}}]
\vspace{-.3cm}
\item 
$f_i(\m{z}), \nabla f_i(\m{z}), \nabla g_i(\m{z}), \nabla^2 g_i (\m{z})$ are $\ell_{f,0}$,$\ell_{f,1}$,$\ell_{g,1}$, $\ell_{g,2}$-Lipschitz continuous, respectively; and 
\vspace{-.2cm}
\item 
$g_i(\m{x},\m{y})$ is $\mu_{g}$-strongly convex in $\m{y}$ for all $\m{x}\in \mb{R}^{d_1}$.
\end{enumerate}
\end{assumption}
Throughout, we use $\kappa_g=\ell_{g,1}/\mu_g$ to denote the condition number of the inner function $g$.
\begin{assumption}[Stochastic samples]\label{assu:bound:var}
For all $i \in [m]$:
\begin{enumerate}[label={\textnormal{\textbf{(B\arabic*})}}]
\vspace{-.3cm}
\item $\nabla f_i(\m{z};\xi)$, $\nabla g_i(\m{z};\zeta)$, $\nabla^2 g_i(\m{z}; \zeta)$ are unbiased estimators of $\nabla f_i(\m{z})$, $\nabla g_i(\m{z})$, $\nabla^2g_i(\m{z})$, respectively; and
\item Their variances are bounded, i.e., $\mb{E}_{\xi}[\|\nabla f_i(\m{z};\xi)-\nabla f_i(\m{z})\|^2] \leq \sigma_f^2$,   $\mb{E}_{\zeta}[\|\nabla g_i(\m{z};\zeta) -\nabla g_i(\m{z})\|^2] \leq \sigma_{g,1}^2$, and $\mb{E}_{\zeta}[\|\nabla^2 g_i(\m{z};\zeta)-\nabla^2 g_i(\m{z})\|^2] \leq \sigma_{g,2}^2$ for some $\sigma_f^2, \sigma_{g,1}^2$, and $\sigma_{g,2}^2$.
\end{enumerate}
\end{assumption}
These assumptions are common in the bilevel optimization literature~\citep{ghadimi2018approximation,chen2021closing,ji2021bilevel}. Assumption~\ref{assu:f} requires that the inner and outer functions are well-behaved. Specifically, strong-convexity of the inner objective is a recurring assumption in bilevel optimization theory implying a unique solution to the inner minimization in \eqref{fedblo:prob}.
\subsection{Proposed Algorithm: \fedblo}

In this section, we develop \fedblo, which is formally presented in Algorithm \ref{alg:fedblo}. The algorithm operates in two nested loops. The outer loop operates in rounds $k\in\{1,\ldots,K\}$. Within each round, an inner loop operating for $T$ iterations is executed. Given estimates $\m{x}^k$ and  $\m{y}^k$, each iteration $t\in\{1,\ldots,T\}$ of the inner loop produces a new \emph{global} model $\m{y}^{k,t+1}$ of the inner optimization variable $\m{y}^*(\m{x}^k)$ as the output of an optimizer \fedinn. The final estimate $\m{y}^{k+1}=\m{y}^{k,T}$ of the inner variable is then used by an optimizer \fedout to update the outer \emph{global} model $\m{x}^{k+1}$.

The subroutines \fedinn and \fedout are gradient-based optimizers.
Each subroutine involves a certain number of \emph{local} training steps indexed by $\nu\in\{0,\ldots,\tau_i-1\}$ that are performed at the $i^{\text{th}}$ client. The local steps of \fedinn iterate over \emph{local} models $\m{y}_{i,\nu }$ of the inner variable. Accordingly, \fedout iterates over \emph{local} models $\m{x}_{i,\nu }$ of the global variable. A critical component of \fedout is a communication-efficient federated hypergradient estimation routine, which we call {\fedhess}.
The implementation of \fedinn, \fedout and \fedhess is critical to circumvent the algorithmic challenges of federated bilevel optimization. In the remaining of this section, we detail the challenges and motivate our proposed implementations. Later, in Section \ref{sec:theory}, we provide a formal convergence analysis of \fednest. 
\subsection{Key Challenge: Federated Hypergradient Estimation}\label{sec:fedhess}

\fedout is a gradient-based optimizer for the outer minimization in \eqref{fedblo:prob}; thus each iteration involves computing $\nabla f(\x)=(1/m)\sum_{i=1}^{m}\nabla f_i(\x,\y^*(\x))$. Unlike single-level FL, the fact that the outer objective $f$ depends explicitly on the inner minimizer $\y^*(\x)$ introduces a new challenge. 
A good starting point to understand the challenge is the following evaluation of $\nabla f(\x)$ in terms of partial derivatives. The result is well-known from properties of implicit functions.
\begin{lemma}\label{lem:IFT}
Under Assumption~\ref{assu:f}, for all $i\in [m]$:
\begin{subequations}\label{main_grad}
\begin{align*}
\nabla f_i(\m{x}, \m{y}^*(\m{x}))&= \nabla^{\texttt{D}}f_i\left(\m{x},\m{y}^*(\m{x})\right)+\nabla^{\texttt{I}}f_i\left(\m{x},\m{y}^*(\m{x})\right),
\end{align*}
where the \emph{direct} and \emph{indirect} gradient components are:
\begin{align}\label{main_grad:decom}
&  \nabla^{\texttt{D}} f_i(\m{x},\m{y}^*(\m{x})):= \nabla_{\m{x}}f_i\left(\m{x},\m{y}^*(\m{x})\right),\\
\nonumber
&\nabla^{\texttt{I}} f_i(\m{x},\m{y}^*(\m{x})):=- \nabla^2_{\m{x}\m{y}} g(\m{x},\m{y}^*(\m{x}))\\
&\quad \qquad\cdot \left[\nabla^2_{\m{y}}g(\m{x},\m{y}^*(\m{x}))\right]^{-1} \nabla_\m{y} f_i \left(\m{x},\m{y}^*(\m{x})\right).\label{main_grad:decom2}
\end{align}
\end{subequations}
\end{lemma}
We now use the above formula  
to describe the two core challenges of bilevel FL optimization. 

First, evaluation of any of the terms in \eqref{main_grad} requires access to the minimizer $\m{y}^*(\m{x})$ of the inner problem. On the other hand, one may at best hope for a good \emph{approximation} to $\m{y}^*(\m{x})$ produced by the inner optimization subroutine. Of course, this challenge is inherent in any bilevel optimization setting, but is exacerbated in the FL setting because of \emph{client drift}. Specifically, when clients optimize their individual (possibly different) local inner objectives, the global estimate of the inner variable produced by SGD-type methods may drift far from (a good approximation to) $\m{y}^*(\m{x})$. We explain in Section \ref{sec:fedinn} how \fedinn solves that issue.

The second challenge comes from the stochastic nature of the problem. 
Observe that the indirect component in \eqref{main_grad:decom2} is nonlinear in the Hessian $\nabla_{\m{y}}^2 g(\m{x}, \m{y}^*(\m{x}))$,   complicating an unbiased stochastic approximation of $\nabla f_i(\m{x}, \m{y}^*(\m{x}))$. As we expose here,  solutions to this complication developed in the non-federated bilevel optimization literature, are \emph{not} directly applicable in the FL setting. Indeed, existing stochastic bilevel algorithms, e.g.  \cite{ghadimi2018approximation}, define $\bar{\nabla}f(\m{x},\m{y}):= \bar{\nabla}^{\texttt{D}}f(\m{x},\m{y})+\bar{\nabla}^{\texttt{I}}f(\m{x},\m{y})$ as a surrogate of $\nabla f(\m{x},  \m{y}^*(\m{x}))$ by replacing $\m{y}^*(\m{x})$ in definition \eqref{main_grad} with an approximation $\m{y}$ and using the following stochastic approximations:
\begin{subequations}\label{eq.inverse-estimator}
\begin{align}\label{eq.inverse-estimator.a}
& \bar{\nabla}^{\texttt{D}}f(\m{x},\m{y})  \approx  \nabla_{\m{x}} f(\m{x}, \m{y};\dot{\xi}),
\\
\nonumber
& \bar{\nabla}^{\texttt{I}}f(\m{x},\m{y})  \approx  -\nabla^2_{\m{xy}}g(\m{x},\m{y};\zeta_{N'+1}) \\
& \Big[\frac{N}{\ell_{g,1}}\prod\limits_{n=1}^{N'}\big(\m{I}-\frac{1}{\ell_{g,1}} \nabla^2_{\m{y}} g(\m{x},\m{y};\zeta_n)\big)\Big] \nabla_\m{y} f(\m{x},\m{y};\dot{\xi}).
\label{eq.inverse-estimator.b}
\end{align}
\end{subequations}
Here, $N'$ is drawn from $\{0,  \ldots, N-1\}$ uniformly at random (UAR) and  $\{\dot{\xi}, \zeta_{1}, \ldots,\zeta_{N'+1}\}$ are i.i.d.~samples. 
\citet{ghadimi2018approximation,hong2020two} have shown that using \eqref{eq.inverse-estimator}, the inverse Hessian estimation bias exponentially decreases with the number of samples $N$. 

One might hope to directly leverage the above approach in a local computation fashion by replacing the global outer function $f$ with the individual function $f_i$. However, note from \eqref{main_grad:decom2} and \eqref{eq.inverse-estimator.b} that the proposed stochastic approximation of the indirect gradient involves in a nonlinear way the \emph{global} Hessian,  which is \emph{not} available at the client \footnote{We note that the approximation in \eqref{eq.inverse-estimator} is not the only construction, and bilevel optimization can accommodate other forms of gradient surrogates~\citep{ji2021bilevel}. Yet, all these approximations require access (in a nonlinear fashion) to the \emph{global} Hessian; thus, they suffer from the same challenge in FL setting.}. Communication efficiency is one of the core objectives of FL making the idea of communicating Hessians between clients and server prohibitive. \emph{Is it then possible, in a FL setting, to obtain an accurate stochastic estimate of the indirect gradient while retaining communication efficiency?} In Section \ref{sec:fedout}, we show how \fedout and its subroutine \fedhess, a matrix-vector products-based (thus, communication efficient) federated hypergradient estimator, answer this question affirmatively.
\subsection{Outer Optimizer: \fedout}\label{sec:fedout}
\begin{algorithm}[t]
\caption{$\m{x}^{+} ~=~\pmb{\fedout}~(\m{x}, \m{y}^+, \alpha)$ for stochastic \colorbox{cyan!30}{bilevel} and  \colorbox{green!30}{minimax} problems
%
} 
\begin{algorithmic}[1]
\State $\funco_i(\cdot)\gets \nabla_\m{x} f_i(\cdot,\m{y}^+;\cdot)$ 
 \State $\m{x}_{i,0}=\m{x}$ and $\alpha_i \in (0,\alpha]$
\State \colorbox{cyan!30}{Choose $N \in \mb{N}$ and set $ \m{p}_{N^{'}}=~\pmb{\fedhess}~(\m{x}, \m{y}^+,N)$}
 \For{$i \in \mc{S}$ \textbf{in parallel}} 
 \State \colorbox{cyan!30}{$\m{h}_i= \funco_i(\m{x};\xi_{i})-\nabla^2_{\m{xy}}g_i(\m{x},\m{y}^+;{\zeta}_{i})
 \m{p}_{N^{'}}$}
  \State \colorbox{green!30}{$\m{h}_i=\funco_i(\m{x};  \xi_{i})$}
 \EndFor
 \State $\m{h}=|\mathcal{S}|^{-1}\sum_{i\in\mathcal{S}}\m{h}_i$ 
 \For {$i \in \mc{S}$ \textbf{in parallel}} 
\For {$\nu= 0,\ldots,\tau_i-1$} 
 \State
$\m{h}_{i,\nu }=\funco_i(\m{x}_{i,\nu }; \xi_{i,\nu })-\funco_i(\m{x}; \xi_{i,\nu })+\m{h}$
 \State  $ \m{x}_{i,\nu +1}= \m{x}_{i,\nu }- \alpha_i\m{h}_{i,\nu }$
  \EndFor
\EndFor
\State $\m{x}^{+}=|\mc{S}|^{-1}\sum_{i\in \mc{S}} \m{x}_{i,\tau_i}$ 
\end{algorithmic}
\label{alg:fedout}
\end{algorithm}

\begin{algorithm}[t]
  \caption{$\m{p}_{N'}=~\pmb{\fedhess}~(\m{x}, \m{y}^+,N)$: Federated approximation of inverse-Hessian-gradient product}
\begin{algorithmic}[1]
\State Select $ N'\in \{0, \dots, N-1\}$ UAR.
\State Select $\mc{S}_0 \in \mc{S}$ UAR.
  \For{$i \in \mc{S}_0$ \textbf{in parallel}} 
 \State $\m{p}_{i,0} =\nabla_\m{y} f_i(\m{x}, \m{y}^+; \xi_{i,0})$
 \EndFor
 \State $ \m{p}_0 = \frac{N}{\ell_{g,1}}  |\mc{S}_0|^{-1}\sum_{i\in\mathcal{S}_0} \m{p}_{i,0}$  
\If{$N'=0$}
\State  Return $\m{p}_{N'}$
\EndIf
\State Select $\mc{S}_1, \ldots, \mc{S}_{ N'} \in \mc{S}$ UAR.
\For{$n=1,\ldots, N'$}
		\For {$i \in \mc{S}_n$ \textbf{in parallel}} 
		 \State $\m{p}_{i,n} = \left(\m{I}-\frac{1}{\ell_{g,1}}\nabla^2_\m{y} g_i(\m{x}, \m{y}^+;{\zeta}_{i,n})\right)\m{p}_{n-1}$
	   \EndFor
	   \State $\m{p}_{n}=|\mc{S}_n|^{-1}\sum_{i\in\mc{S}_n} \m{p}_{i,n}$ 
		\EndFor
\end{algorithmic}\label{fedihgp_algo}
\end{algorithm}


This section presents the outer optimizer \fedout, formally described in Algorithm~\ref{alg:fedout}. As a subroutine of \fedblo (see Line 9, Algorithm~\ref{alg:fednest}), at each round $k=0,\ldots,K-1$, \fedout takes the most recent global outer model $\x^k$ together the updated (by \fedinn) global inner model $\y^{k+1}$ and produces an update $\x^{k+1}$. To lighten notation, for a round $k$, denote the function's input as $(\m{x}, \m{y}^+)$ (instead of $(\x^k,\y^{k+1})$) and the output as $\m{x}^+$ (instead of $\x^{k+1}$).
For each client $i\in \mathcal{S}$, \fedout uses   stochastic approximations of $\bar{\nabla}^{\texttt{I}} f_i(\m{x},\m{y}^+)$ and $\bar{\nabla}^{\texttt{D}}f_i(\m{x},\m{y}^+)$, which we call  $\m{h}_i^{\texttt{I}}(\m{x}, \m{y}^+)$ and $\m{h}_i^{\texttt{D}}(\m{x}, \m{y}^+)$, respectively.  The specific choice of these approximations (see Line 5) is critical and is discussed in detail later in this section. Before that, we explain how each client uses these proxies to form  local updates of the outer variable.
In each round, starting from a common global model $\m{x}_{i,0}=\m{x}$, each client $i$ performs $\tau_i$ local steps (in parallel):
\begin{align}\label{eqn:update1:fedout}
\m{x}_{i,\nu +1} = \m{x}_{i,\nu }-\alpha_i\m{h}_{i,\nu },
\end{align}
and then the server aggregates local models via $\m{x}^+=|\mc{S}|^{-1} \sum_{i\in \mc{S}} \m{x}_{i,\tau_i}$. Here, $\alpha_i \in (0,\alpha]$ is the local stepsize, 
\begin{equation}\label{eqn:hyper:estim}
\begin{aligned}
\m{h}_{i,\nu }:=&\m{h}^{\texttt{I}}(\m{x},\m{y}^{+})+  \m{h}^{\texttt{D}}(\m{x},\m{y}^{+})\\
&-\m{h}_i^{\texttt{D}}(\m{x},\m{y}^{+})+\m{h}_i^{\texttt{D}}(\m{x}_{i,\nu },\m{y}^{+})\,,
\end{aligned}
\end{equation}
$\m{h}^{\texttt{I}}(\m{x},\m{y}) := |\mc{S}|^{-1}\sum_{i\in\mathcal{S}} \m{h}_i^{\texttt{I}} (\m{x},\m{y})$, and $\m{h}^{\texttt{D}}(\m{x},\m{y}) := |\mc{S}|^{-1}\sum_{i\in\mathcal{S}} \m{h}_i^{\texttt{D}} (\m{x},\m{y})$. 

The key features of updates~\eqref{eqn:update1:fedout}--\eqref{eqn:hyper:estim} are exploiting past gradients (variance reduction) to account for objective heterogeneity. Indeed,
the ideal update in \fedout would perform the update $\m{x}_{i,\nu +1} =\m{x}_{i,\nu }-\alpha_i\big(\m{h}^{\texttt{I}}(\m{x}_{i,\nu },\m{y}^{+})+ \m{h}^{\texttt{D}}(\m{x}_{i,\nu },\m{y}^{+})\big)$ using the global gradient estimates. But this requires each client $i$ to have access to both direct and indirect gradients of all other clients--which it does not, since clients do not communicate between rounds. To overcome this issue, each client $i$ uses global gradient estimates, i.e., $ \m{h}^{\texttt{I}}(\m{x},\m{y}^{+})+\m{h}^{\texttt{D}}(\m{x},\m{y}^{+})$ from the beginning of each round as a guiding direction in its local update rule. However, since both $\m{h}^{\texttt{D}}$ and $\m{h}^{\texttt{I}}$ are computed at a previous  $(\m{x},\m{y}^{+})$, client $i$ makes a correction by subtracting off the stale direct gradient estimate $\m{h}_i^{\texttt{D}}(\m{x},\m{y}^{+})$ and adding its own \emph{local} estimate $\m{h}_i^{\texttt{D}}(\m{x}_{i,\nu },\m{y}^{+})$. Our local update rule in Step~11 of Algorithm~\ref{alg:fedout} is precisely of this form, i.e.,  $\m{h}_{i,\nu }$ approximates
$\m{h}^{\texttt{I}}(\m{x}_{i,\nu },\m{y}^{+})+ \m{h}^{\texttt{D}}(\m{x}_{i,\nu },\m{y}^{+})$
via \eqref{eqn:hyper:estim}. Note here that the described local correction of \fedout only applies to the direct gradient component (the indirect component would require global Hessian information). An alterantive approach leading to \lfedblo is discussed in Section \ref{light blo}.


\noindent\textbf{\fedout applied to special nested problems.} Algorithm~\ref{alg:fedout} naturally allows the use of other optimizers for minimax \& compositional optimization. For example, in the minimax problem~\eqref{fedminmax:prob}, the bilevel gradient components are $\nabla^{\texttt{D}} f_i(\m{x},\m{y}^*(\m{x}))=\nabla_{\m{x}}f_i\left(\m{x},\m{y}^*(\m{x})\right)$ and $\nabla^{\texttt{I}} f_i(\m{x},\m{y}^*(\m{x}))=0$ for all $i \in\mc{S}$. Hence, the hypergradient estimate \eqref{eqn:hyper:estim} reduces to 
\begin{equation}\label{eqn:hyper:estim:minmax}
 \m{h}_{i,\nu } = \m{h}^{\texttt{D}}(\m{x},\m{y}^{+})-\m{h}_i^{\texttt{D}}(\m{x},\m{y}^{+})+\m{h}_i^{\texttt{D}}(\m{x}_{i,\nu },\m{y}^{+}). 
\end{equation}
For the compositional problem~\eqref{fedcompos:prob2}, Hessian becomes the identity matrix, the direct gradient is the zero vector, and  $\nabla_{\m{xy}}g(\m{x},\m{y})=- (1/m)\sum_{i=1}^m\nabla \m{r}_i(\m{x})^{\top}$. Hence, $\m{h}_i= \ell_{g,1}\nabla  \m{r}_i(\m{x})^{\top} \m{p}_{0}$ for all $i \in\mc{S}$. 


More details on these special cases are provided in Appendices~\ref{sec:app:minmax}~and~\ref{sec:app:compos}.   

\noindent\textbf{Indirect gradient estimation \& \fedhess.} Here, we aim to address one of the key challenges in nested FL: inverse Hessian gradient product.  Note from \eqref{eq.inverse-estimator.b} that the proposed stochastic approximation of the indirect gradient involves in a nonlinear way the \emph{global} Hessian, which is \emph{not} available at the client. To get around this, we use a client sampling strategy and recursive reformulation of \eqref{eq.inverse-estimator.b} so that $\bar{\nabla}^{\texttt{I}}f_i(\m{x},\m{y})$ can be estimated in an efficient federated manner. In particular, given $N \in \mb{N}$, we select $N' \in \{0 \ldots, N-1\}$ and $\mc{S}_0, \ldots, \mc{S}_{ N'} \in \mc{S}$ UAR. For all $i \in \mc{S}$, we then define
\begin{subequations}
\begin{equation}\label{eq:grad:est} 
\begin{aligned}
&\m{h}^{\texttt{I}}_i(\m{x},\m{y}) =-\nabla^2_{\m{xy}}g_i(\m{x},\m{y};\zeta_{i}) \m{p}_{N'},
\end{aligned}    
\end{equation}
where $\m{p}_{N'}=|\mc{S}_0|^{-1}\wh{\m{H}}_\m{y} \sum_{i\in\mc{S}_0} \nabla_\m{y} f_i(\m{x}, \m{y};\xi_{i,0})$ and $\wh{\m{H}}_\m{y}$ is the approximate inverse Hessian:
\begin{align}
\frac{ N}{\ell_{g,1}} \prod_{n=1}^{N'} \Big (\m{I} - \frac{1}{ \ell_{g,1}  |\mc{S}_n|} \sum_{i=1}^{|\mc{S}_n|} \nabla_{\m{y}}^2 g_i(\m{x},\m{y}; \zeta_{i,n}) \Big).
\label{eq:hessinv:est}
\end{align}
\end{subequations}
The subroutine \fedhess provides a recursive strategy to compute $\m{p}_{N'}$  and \fedout multiplies $\m{p}_{N'}$ 
with the Jacobian to drive an indirect gradient estimate. Importantly, {these approximations require only matrix-vector products and vector communications.}

\begin{lemma}\label{lem:neum:bias}
Under Assumptions~\ref{assu:f} and \ref{assu:bound:var}, the approximate inverse Hessian $\wh{\m{H}}_\m{y}$ defined in \eqref{eq:hessinv:est} satisfies the following for any $\x$ and $\y$:
\begin{align}
\nonumber 
\left\|\left[\nabla^2_{\m{y}} g (\m{x}, \m{y})\right]^{-1} - \mb{E}_{\mc{W}}[\wh{\m{H}}_{\m{y}}]\right\| &\le \frac{1}{\mu_g} \left(\frac{\kappa_g-1}{\kappa_g} \right)^N,\\
\mb{E}_{\mc{W}}\left[\left\|\left[\nabla^{2}_{\m{y}} g (\m{x}, \m{y})\right]^{-1}-\wh{\m{H}}_{\m{y}}\right\|\right] &\le \frac{2}{\mu_g}. \label{Hessian_Invs_ineq}
\end{align}
Here, $\mc{W}:=\left\{\mc{S}_n, \xi_{i},\zeta_{i}, \xi_{i,0}, \zeta_{i,n} \mid i \in \mc{S}_n, ~~~0\leq n\leq  N' \right\}$. Further, for all $~i \in \mc{S}$, $\m{h}_i^{\texttt{I}}(\m{x},\m{y})$ defined in \eqref{eq:grad:est} satisfies
\begin{align}
\left\| \mb{E}_{\mc{W}} \left[\m{h}_i^{\texttt{I}}(\m{x},\m{y})\right]- \bar{\nabla}^{\texttt{I}} f_i(\m{x},\m{y}) \right\|&\leq b, 
\end{align}
where $b:= \kappa_g\ell_{f,1}\big( (\kappa_g-1)/\kappa_g \big)^{N}$.
\end{lemma}

\subsection{Inner Optimizer: \fedinn}\label{sec:fedinn}

In FL, each client performs multiple local training steps in isolation on its own data (using for example SGD) before communicating with the server. Due to such local steps, \fedavg  suffers from a \textit{client-drift} effect under objective heterogeneity; that is, the local iterates of each client drift-off towards the minimum of their own local function. In turn, this can  lead to convergence to a point different from the  global optimum $\y^*(\x)$ of the inner problem; e.g., see \citep{mitra2021linear}. This behavior is particularly undesirable in a nested optimization setting since it directly affects the outer optimization; see, e.g. \citep[Section~7]{liu2021investigating}.  

In light of this observation, we build on the recently proposed \textsc{FedLin}~\citep{mitra2021linear} which improves \textsc{FedSVRG}~\citep{konevcny2016federated} to solve the inner problem; see Algorithm~\ref{alg:fedinn}. For each $i\in \mc{S}$, let  $\m{q}_i(\m{x}, \m{y})$ denote an unbiased estimate of the gradient $\nabla_{\m{y}} g_i(\m{x},\m{y})$. In each round, starting from a common global model $\m{y}$, each client $i$ performs $\tau_i$ local SVRG-type training steps in parallel: 
$\m{y}_{i,\nu +1} = \m{y}_{i,\nu }-\beta_i\m{q}_{i,\nu },$  
where $\m{q}_{i,\nu }:= \m{q}_i(\m{x},\m{y}_{i,\nu })-\m{q}_i(\m{x},\m{y})+\m{q}(\m{x},\m{y})$, $\beta_i \in (0,\beta]$ is the local inner stepsize, and $\m{q}(\m{x},\m{y}) := |\mc{S}|^{-1}\sum_{i\in\mathcal{S}} \m{q}_i(\m{x},\m{y})$. We note that for the optimization problems \eqref{fedblo:prob}, \eqref{fedminmax:prob}, and \eqref{fedcompos:prob2}, $\m{q}_i(\m{x},\m{y}_{i,\nu })$ is equal to $\nabla_\m{y} g_i(\m{x}, \m{y}_{i,\nu };\zeta_{i,\nu })$, $-\nabla_\m{y} f_i(\m{x}, \m{y}_{i,\nu };\xi_{i,\nu })$, and $\m{y}_{i,\nu }-\m{r}_i(\m{x};\zeta_{i,\nu })$, respectively; see  Appendices~\ref{sec:app:bilevel}--\ref{sec:app:compos}.
\begin{algorithm}[t]
\caption{$ \m{y}^+~=~\pmb{\fedinn}~(\m{x},\m{y},\beta$) 
}
\label{alg:fedinn}
\begin{algorithmic}[1]
\State $\funcin_i(\cdot)\gets$ \colorbox{cyan!30}{$\nabla_\m{y} g_i(\m{x},\cdot)$~\text{(bilevel)}},\colorbox{green!30}{$-\nabla_\m{y} f_i(\m{x},\cdot)$~\text{(minimax)}}
 \State $\m{y}_{i,0}=\m{y}$ and $\beta_i \in (0,\beta]$
  \For{$i \in \mc{S}$ \textbf{in parallel}} 
 \State $\m{q}_i=\funcin_i(\m{y};  \zeta_{i})$
 \EndFor
 \State $\m{q}=|\mathcal{S}|^{-1}\sum_{i\in\mathcal{S}}\m{q}_i$ 
 \For {$i \in \mc{S}$ \textbf{in parallel}} 
\For {$\nu=0,\ldots,\tau_i-1$} 
\State $\m{q}_{i,\nu }=\funcin_i(\m{y}_{i,\nu };\zeta_{i,\nu })-\funcin_i(\m{y};\zeta_{i,\nu })+\m{q}$
\State $\m{y}_{i,\nu +1}= \m{y}_{i,\nu }-\beta_i\m{q}_{i,\nu }$
\EndFor
\EndFor
\State $\m{y}^{+}=|\mathcal{S}|^{-1}\sum_{i\in\mathcal{S}}\m{y}_{i,\tau_i}$ 
\end{algorithmic}
\end{algorithm}

\subsection{Light-\fedblo: Communication Efficiency via Local Hypergradients}\label{light blo}
 Each \fednest epoch $k$ requires $2T+N+3$ communication rounds as follows: $2T$ rounds for SVRG of \fedinn, 
$N$ iterations for inverse Hessian approximation within \fedhess and $3$ additional aggregations. Note that, these are vector communications and we fully avoid Hessian communication. 
In Appendix~\ref{sec:simp:algs}, we also propose simplified variants of \fedout and \fedhess, {which are tailored to homogeneous or  high-dimensional FL settings}. These algorithms can then either use \textit{local} Jacobian / inverse Hessian or their approximation, and can use either SVRG or SGD. 

\noindent\textbf{Light-\fedblo:} Specifically, we propose \lfedblo where each client runs IHGP locally.  This reduces the number of rounds to $T+1$, saving $T+N+2$ rounds (see experiments in Section \ref{sec:numerics} for performance comparison and Appendix~\ref{sec:simp:algs} for further discussion.)
\section{Convergence Analysis for \fedblo}\label{sec:theory}
In this section, we present convergence results for \fedblo. All proofs are relegated to Appendices~\ref{sec:app:bilevel}--\ref{sec:app:compos}.
\begin{theorem}\label{thm:fednest}
Suppose Assumptions~\ref{assu:f} and \ref{assu:bound:var} hold.  Further, assume  $\alpha_i^k=\alpha_k/\tau_i$ and $\beta_i^k=\beta_k/\tau_i$ for all $ i \in \mathcal{S}$, where 
\begin{align}\label{eqn:param:choice}
\beta_k &= \frac{\bar{\beta} \alpha_k}{T},~~~\alpha_k=\min\left\{\bar{\alpha}_1, \bar{\alpha}_2, \bar{\alpha}_3, \frac{\bar{\alpha}}{\sqrt{K}}\right\}
\end{align}
for some positive constants $\bar\alpha_1,\bar\alpha_2,\bar\alpha_3,\bar{\alpha}$, and $\bar{\beta}$ independent of $K$. Then, for any $T\geq 1$, the iterates $\{(\m{x}^k,\m{y}^k)\}_{k\geq0}$ generated by \fednest satisfy
\begin{equation*}
\begin{aligned}
 \frac{1}{K} \sum_{k=1}^K\mb{E}\left[\left\|\nabla f(\m{x}^k)\right\|^2\right] =~&{\cal O}\Big(\frac{\bar{\alpha}\max(\sigma_{g,1}^2, \sigma_{g,2}^2,\sigma_f^2)}{\sqrt{K}} \\
& + \frac{1}{\min(\bar{\alpha}_1, \bar{\alpha}_2, \bar{\alpha}_3)K} + b^2 \Big),
\end{aligned}
\end{equation*}
where  $b=\kappa_g\ell_{f,1}\big( (\kappa_g-1)/\kappa_g \big)^{N}$ and $N$ is the input parameter to \fedhess.
\end{theorem}
%
%
\begin{corollary}[\textbf{Bilevel}]\label{thm:fednest:bilevel}
Under the same conditions as in Theorem \ref{thm:fednest},  if
{\small $N={\cal O}(\kappa_g \log K)$ and $T=\mc{O}(\kappa_g^{4})$}, then 
\begin{equation*}
 \frac{1}{K} \sum_{k=1}^K\mb{E}\left[\left\|\nabla f(\m{x}^k)\right\|^2\right]= \mc{O}\left(\frac{\kappa_g^4}{K} + \frac{\kappa_g^{2.5}}{\sqrt{K}}\right).
\end{equation*}
For $\epsilon$-accurate stationary point, we need $K=\mc{O}(\kappa^5_g \epsilon^{-2})$.
\end{corollary}
Above, we choose $N\propto \kappa_g \log K$ to guarantee $b^2\lesssim 1/\sqrt{K}$. In contrast, we use $T\gtrsim\kappa_g^4$ inner SVRG epochs. From Section~\ref{light blo}, this would imply the communication cost is dominated by SVRG epochs $N$ and $\order{\kappa_g^4}$ rounds. 

From Corollary~\ref{thm:fednest:bilevel}, we remark that \fednest matches the guarantees of non-federated alternating SGD methods, such as ALSET~\cite{chen2021closing} and  {BSA} \citep{ghadimi2018approximation}, despite federated setting, i.e. communication challenge, heterogeneity in the client objectives, and device heterogeneity.
\begin{figure*}
\vspace{-10pt}
    \centering
    \begin{subfigure}{0.31\textwidth}
    \centering
        \begin{tikzpicture}
        \node at (0,0) {\includegraphics[scale=0.33]{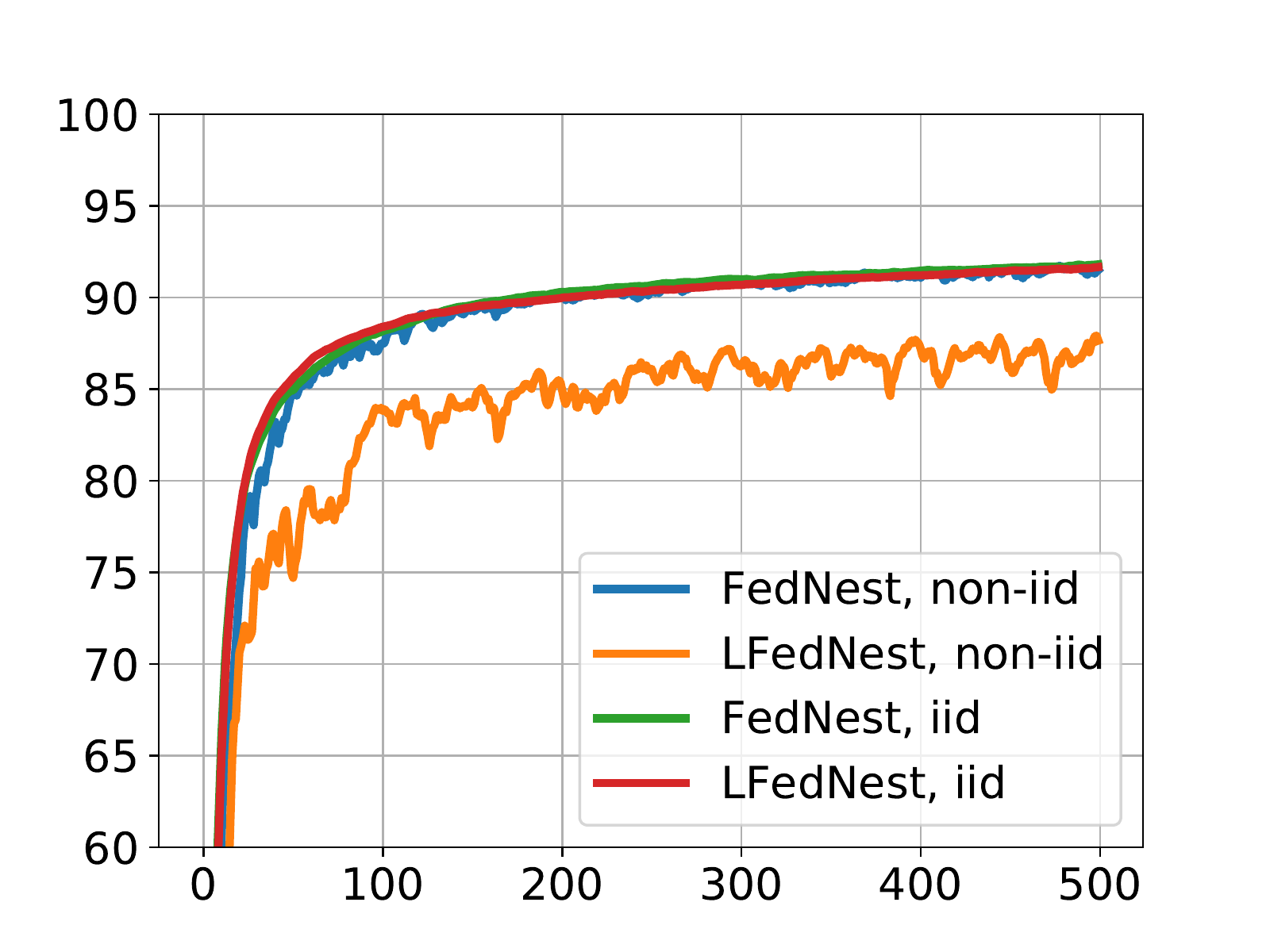}};
        \node at (0,-2.1) [scale=0.9]{\textsc{FedNest} epochs};
        \node at (-2.6,0) [scale=0.9, rotate=90]{Test accuracy};
        \end{tikzpicture}\vspace{-5pt}\caption{Comparison between \textsc{FedNest} and \textsc{LFedNest}.}
        \label{fig:hr_global_local}
    \end{subfigure}\hspace{5pt}\begin{subfigure}{0.31\textwidth}
    \centering
        \begin{tikzpicture}
        \node at (0,0) {\includegraphics[scale=0.33]{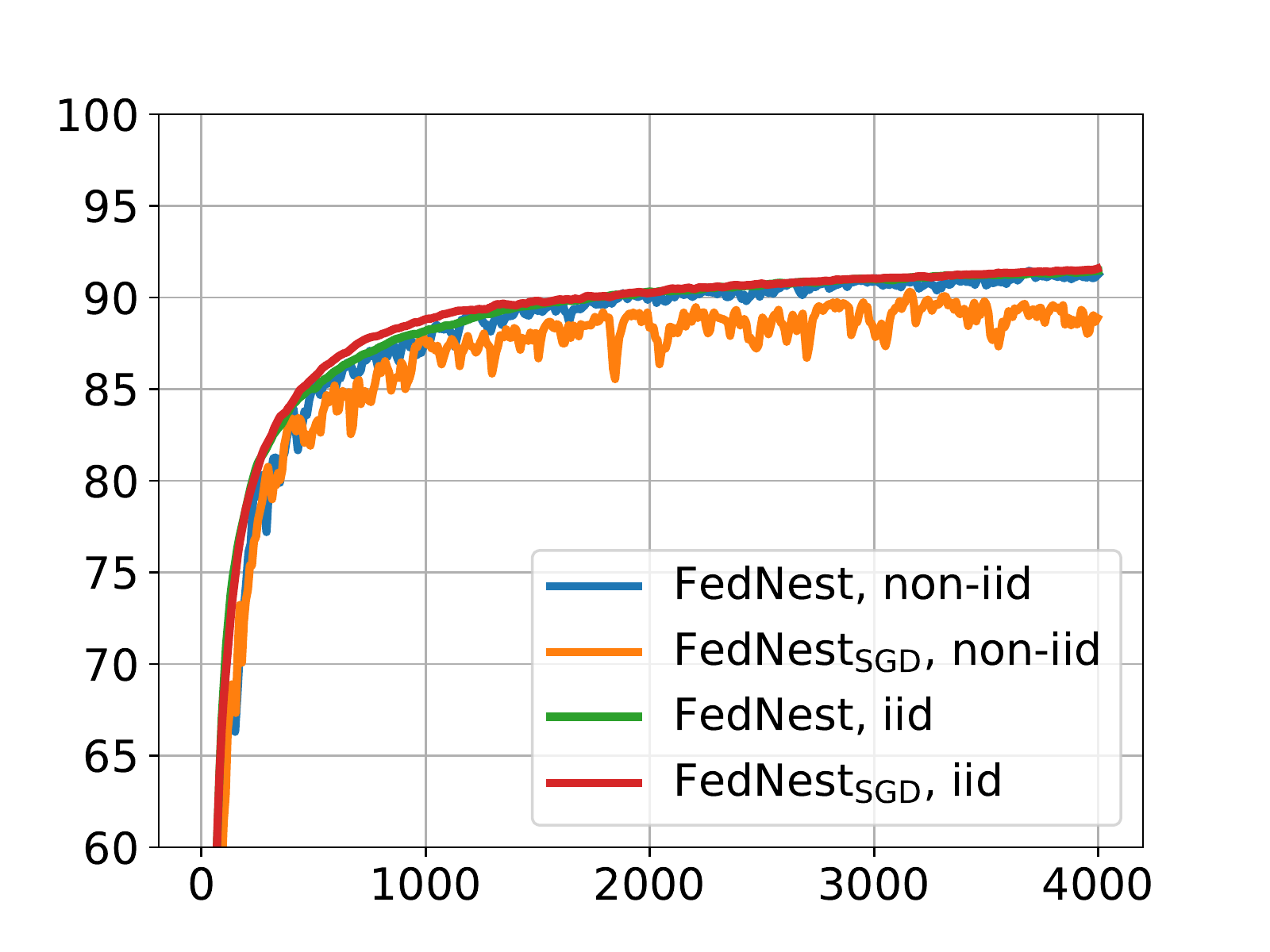}};
        \node at (0,-2.1) [scale=0.9]{Communication rounds};
        
        \end{tikzpicture}\vspace{-5pt}\caption{SVRG in \fedinn provides better convergence and stability.}\label{fig:hr_svrg_sgd}
    \end{subfigure}\hspace{5pt}\begin{subfigure}{0.31\textwidth}
    \centering
        \begin{tikzpicture}
        \node at (0,0) {\includegraphics[scale=0.33]{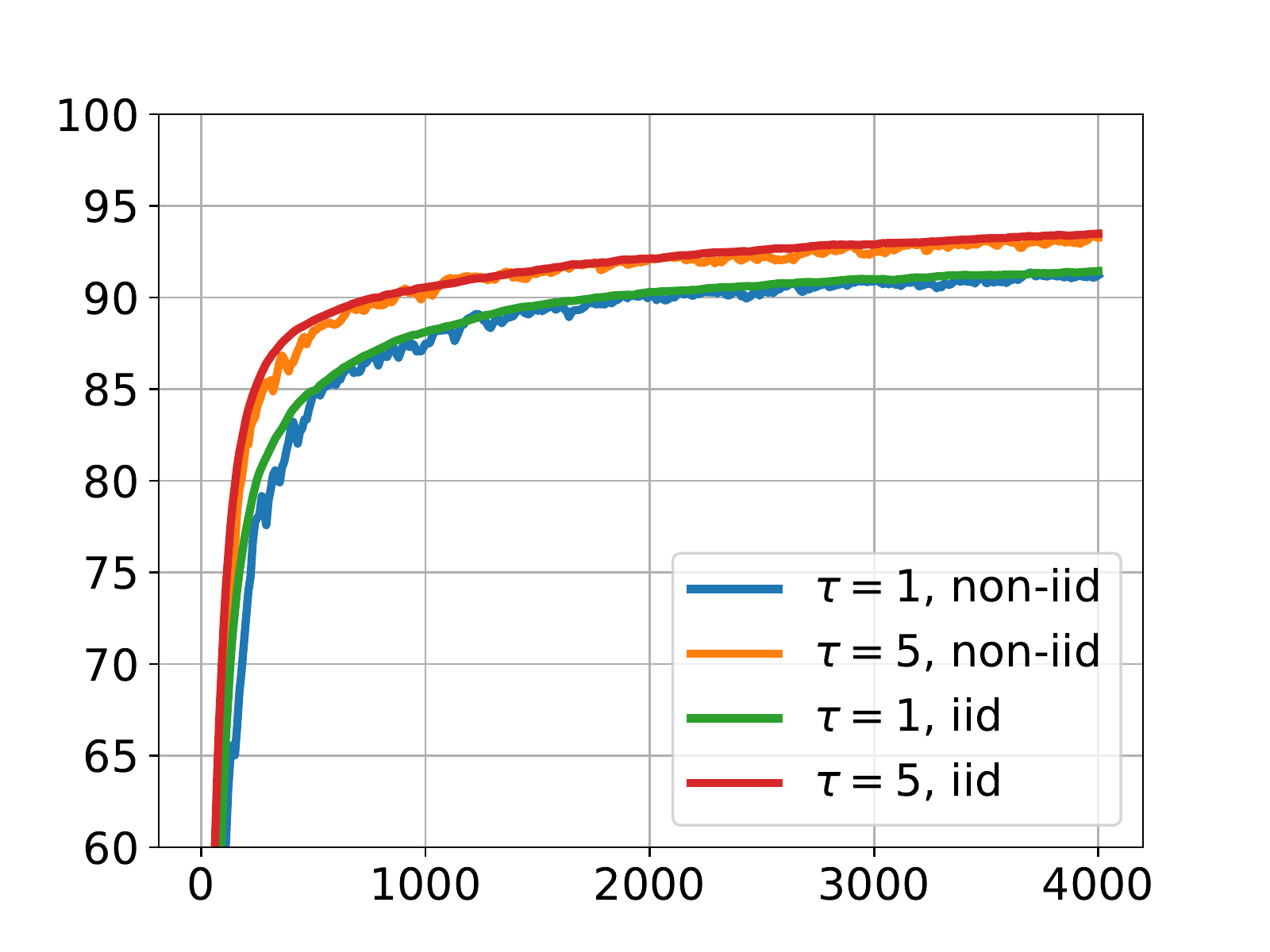}};
        \node at (0,-2.1) [scale=0.9]{Communication rounds};
        \end{tikzpicture}\vspace{-5pt}\caption{Larger $\tau$ in \fedout provides better performance.}\label{fig:hr_tau}
    \end{subfigure}\vspace{-6pt}\caption{Hyper-representation experiments on a 2-layer MLP and MNIST dataset.}\label{fig:hr_exp}
    \vspace{-8pt}
\end{figure*}
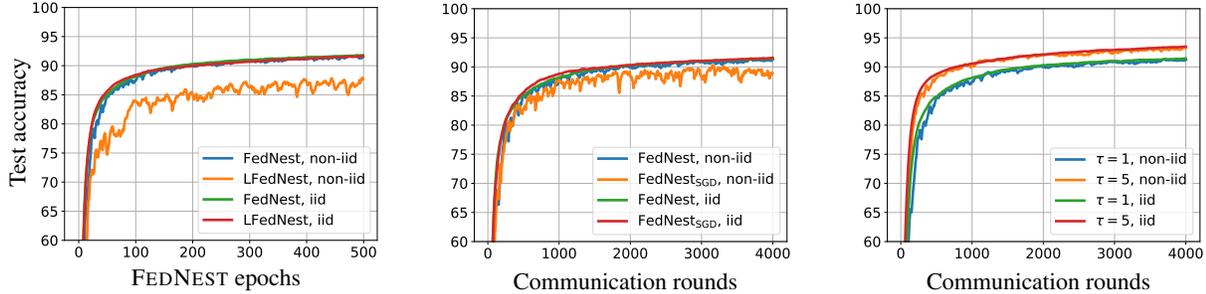
\vspace{-.3cm}
\subsection{Minimax Federated Learning}
We focus on special features of federated minimax problems and customize the general results to yield improved convergence results for this special case. Recall from \eqref{fedminmax:prob} that $g_i(\x,\y)=-f_i(\x,\y)$ which implies that $b=0$ and following Assumption \ref{assu:f},  $f_i(\x,\y)$ is $\mu_f$--strongly concave in $\y$ for all $\x$.
\begin{corollary}[\textbf{Minimax}]\label{thm:fednest:minmax}
Denote $\kappa_f=\ell_{f,1}/\mu_f$. Assume same conditions as in Theorem~\ref{thm:fednest}  and 
 $T=\mc{O}(\kappa_f)$. Then, 
\begin{equation*}
 \frac{1}{K} \sum_{k=1}^K\mb{E}\left[\left\|\nabla f(\m{x}^k)\right\|^2\right]=\mc{O}\left(\frac{\kappa^2_f}{K} + \frac{\kappa_f}{\sqrt{K}}\right).
\end{equation*}
\end{corollary}
Corrollary~\ref{thm:fednest:minmax} implies that for the minimax problem, the convergence rate of \fedblo to the stationary point of $f$ is  $\mc{O}({1/\sqrt{K}})$. Again, we note this matches the convergence rate of non-FL algorithms (see also Table~\ref{table:blo2:results}) such as SGDA \citep{lin2020gradient} and SMD~\citep{rafique2021weakly}.
\subsection{Compositional Federated Learning}
Observe that in the compositional problem~\eqref{fedcompos:prob2}, the outer function is $f_i(\m{x},\m{y};\xi)=f_i(\m{y};\xi)$ and the inner function is  $g_i(\m{x},\m{y};\zeta)=\frac{1}{2}\|\m{y}-\m{r}_i(\m{x};\zeta)\|^2$, for all $i \in \mc{S}$. Hence, $b=0$ and $\kappa_g=1$. %
\begin{corollary}[\textbf{Compositional}]\label{thm:fednest:compos}
Under the same conditions as in Theorem \ref{thm:fednest}, if we select $T=1$ in \eqref{eqn:param:choice}. Then,
\begin{equation*}
 \frac{1}{K} \sum_{k=1}^K\mb{E}\left[\left\|\nabla f(\m{x}^k)\right\|^2\right]=\mc{O}\left(\frac{1}{\sqrt{K}}\right).
\end{equation*}
\end{corollary}
Corrollary~\ref{thm:fednest:compos} implies that for the  compositional problem \eqref{fedcompos:prob2}, the convergence rate of \fedblo to the stationary point of $f$ is  $\mc{O}({1/\sqrt{K}})$. This matches the convergence rate of non-federated stochastic algorithms such as SCGD~\citep{wang2017stochastic} and NASA~\citep{ghadimi2020single} (Table~\ref{table:blo2:results}).
\subsection{Single-Level Federated Learning}
Building upon the general results for stochastic nonconvex nested problems, we establish new convergence guarantees for \emph{single-level} stochastic non-convex federated SVRG which is integrated within our \fedout.  Note that in the single-level setting,  the optimization problem~\eqref{fedblo:prob} reduces to 
\begin{align}\label{fedsingle:prob2}
\underset{\m{x} \in \mb{R}^{{d}_1}}{\min} ~~f(\m{x})=\frac{1}{m} \sum_{i=1}^{m} f_{i}\left(\m{x}\right)
\end{align}
with $f_i(\m{x}) := \mb{E}_{\xi \sim \mc{C}_i}[f_i(\m{x}; \xi)]$, where  $\xi\sim\mc{C}_i$ is sampling distribution for the $i^{\text{th}}$ client.

We make the following assumptions on \eqref{fedsingle:prob2} that are counterparts of Assumptions~\ref{assu:f} and \ref{assu:bound:var}.
\begin{assumption}[Lipschitz continuity]\label{assu:s:f} 
For all $i \in [m]$, $\nabla f_i(\m{x})$ is $L_f$-Lipschitz continuous.
\end{assumption}
\begin{assumption}[Stochastic samples]\label{assu:s:bound:var}
For all $i \in [m]$, $\nabla f_i(\m{x};\xi)$ is an unbiased estimator of $\nabla f_i(\m{x})$ and its variance is bounded, i.e., $\mb{E}_{\xi}[\|\nabla f_i(\m{x};\xi)-\nabla f_i(\m{x})\|^2] \leq \sigma_f^2$.
\end{assumption}
\begin{theorem}[\textbf{Single-Level}]\label{thm:fednest:s:level}
Suppose Assumptions~\ref{assu:s:f} and \ref{assu:s:bound:var} hold.  Further, assume  $\alpha_i^k=\alpha_k/\tau_i$ for all $ i \in \mathcal{S}$, where 
\begin{align}\label{eqn:param:s:choice}
\alpha_k&=\min\left\{\bar{\alpha}_1,\frac{\bar{\alpha}}{\sqrt{K}}\right\}
\end{align}
for some positive $\bar{\alpha}_1$ and  $\bar{\alpha}$. Then,
\begin{equation*}
 \frac{1}{K} \sum_{k=1}^K\mb{E}\left[\left\|\nabla f(\m{x}^k)\right\|^2\right]=\mc{O}\left( \frac{ \Delta_f}{ \bar{\alpha}_1 K} + \frac{  \frac{\Delta_f}{\bar{\alpha}} + \bar{\alpha}\sigma_f^2 }{\sqrt{K}}\right),
\end{equation*}
where $\Delta_{f}:= f(\m{x}^0)-\mb{E}[f(\m{x}^K)]$.
\end{theorem}
Theorem~\ref{thm:fednest:s:level} extends recent results by \cite{mitra2021linear} from the stochastic strongly convex to the stochastic nonconvex setting. The above rate is also consistent with existing single-level non-FL guarantees~\cite{ghadimi2013stochastic}.
\vspace{-.2cm}
\section{Numerical Experiments}\label{sec:numerics}
In this section, we numerically investigate the impact of several attributes of our algorithms on a hyper-representation problem \cite{franceschi2018bilevel}, a hyper-parameter optimization problem for loss function tuning \cite{li2021autobalance}, and a federated minimax optimization problem. 

\subsection{Hyper-Representation Learning}
Modern approaches in meta learning such as MAML~\cite{finn2017model} and reptile \cite{nichol2018reptile} learn representations (that are shared across all tasks) in a bilevel manner. Similarly, the hyper-representation problem optimizes a classification model in a two-phased process. The outer objective optimizes the model backbone to obtain better feature representation on validation data. The inner problem optimizes a header for downstream classification tasks on training data. In this experiment, we use a 2-layer multilayer perceptron (MLP) with 200 hidden units. The outer problem optimizes the hidden layer with 157,000 parameters, and the inner problem optimizes the output layer with 2,010 parameters. We study both i.i.d and non-i.i.d. ways of partitioning the MNIST data exactly following~\textsc{FedAvg}~\cite{mcmahan17fedavg}, and split each client's data evenly to train and validation datasets. Thus, each client has 300 train and 300 validation samples. 

Figure~\ref{fig:hr_exp} demonstrates the impact on test accuracy of  several important components of \fednest. Figure~\ref{fig:hr_global_local} compares $\fednest$ and $\lfedblo$. Both algorithms perform well on the i.i.d. setup, while on the non-i.i.d. setup, $\fednest$ achieves i.i.d. performance, significantly outperforming $\lfedblo$. These findings are in line with our discussions in Section~\ref{light blo}. $\lfedblo$ saves on communication rounds compared to $\fednest$ and performs well on homogeneous clients. However, for heterogeneous clients, the isolation of local Hessian in $\lfedblo$ (see Algorithm~\ref{alg:localfedout} in Appendix~\ref{sec:simp:algs}) degrades the test performance.  Next, Figure~\ref{fig:hr_svrg_sgd} demonstrates the importance of SVRG in $\fedinn$ algorithm for heterogeneous data (as predicted by our theoretical considerations in Section~\ref{sec:fedinn}). To further clarify the algorithm difference in Figures~\ref{fig:hr_svrg_sgd} and \ref{fig:lt_global_local}, we use $\fednest_\mathrm{SGD}$ to denote the \fednest algorithm where SGD is used in \fedinn. Finally, Figure~\ref{fig:hr_tau} elucidates the role of local epoch $\tau$ in $\fedout$: larger $\tau$ saves on communication and improves test performance by enabling faster convergence.
\begin{figure}[t]
\vspace{-10pt}
    \centering
    \begin{subfigure}{0.226\textwidth}
    \centering
        \begin{tikzpicture}
        \node at (0,0) {\includegraphics[scale=0.256]{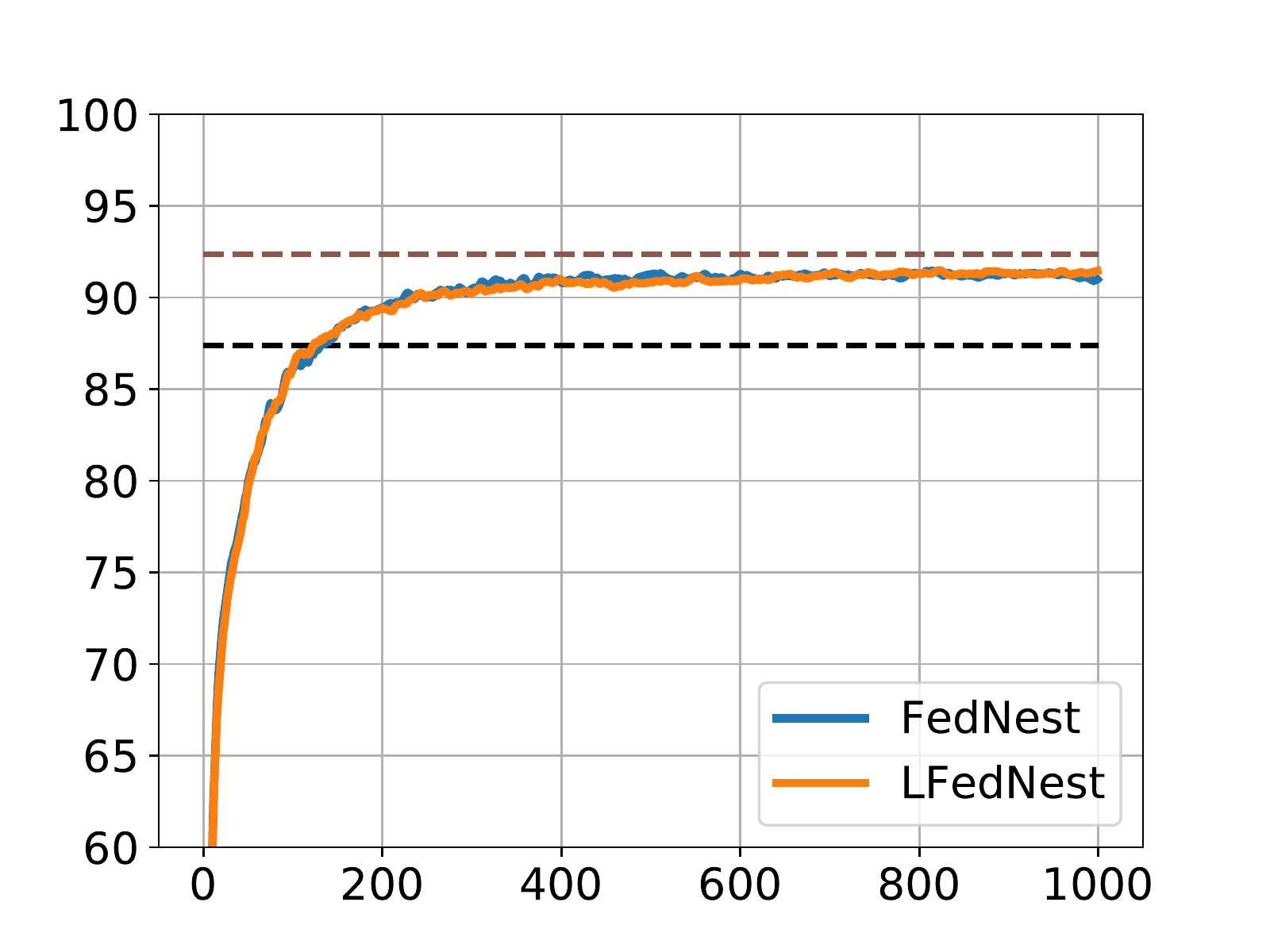}};
        \node at (0,-1.65) [scale=0.7]{\textsc{FedNest} epochs};
        \node at (-2,0) [scale=0.7, rotate=90]{Balanced Test Accuracy};
        \end{tikzpicture}\vspace{-5pt}\caption{\textsc{FedNest} achieves similar performance as non-federated bilevel loss function tuning. }\label{fig:lt_global_local}
    \end{subfigure}\hspace{5pt}\begin{subfigure}{0.225\textwidth}
    \centering
        \begin{tikzpicture}
        \node at (0,0) {\includegraphics[scale=0.256]{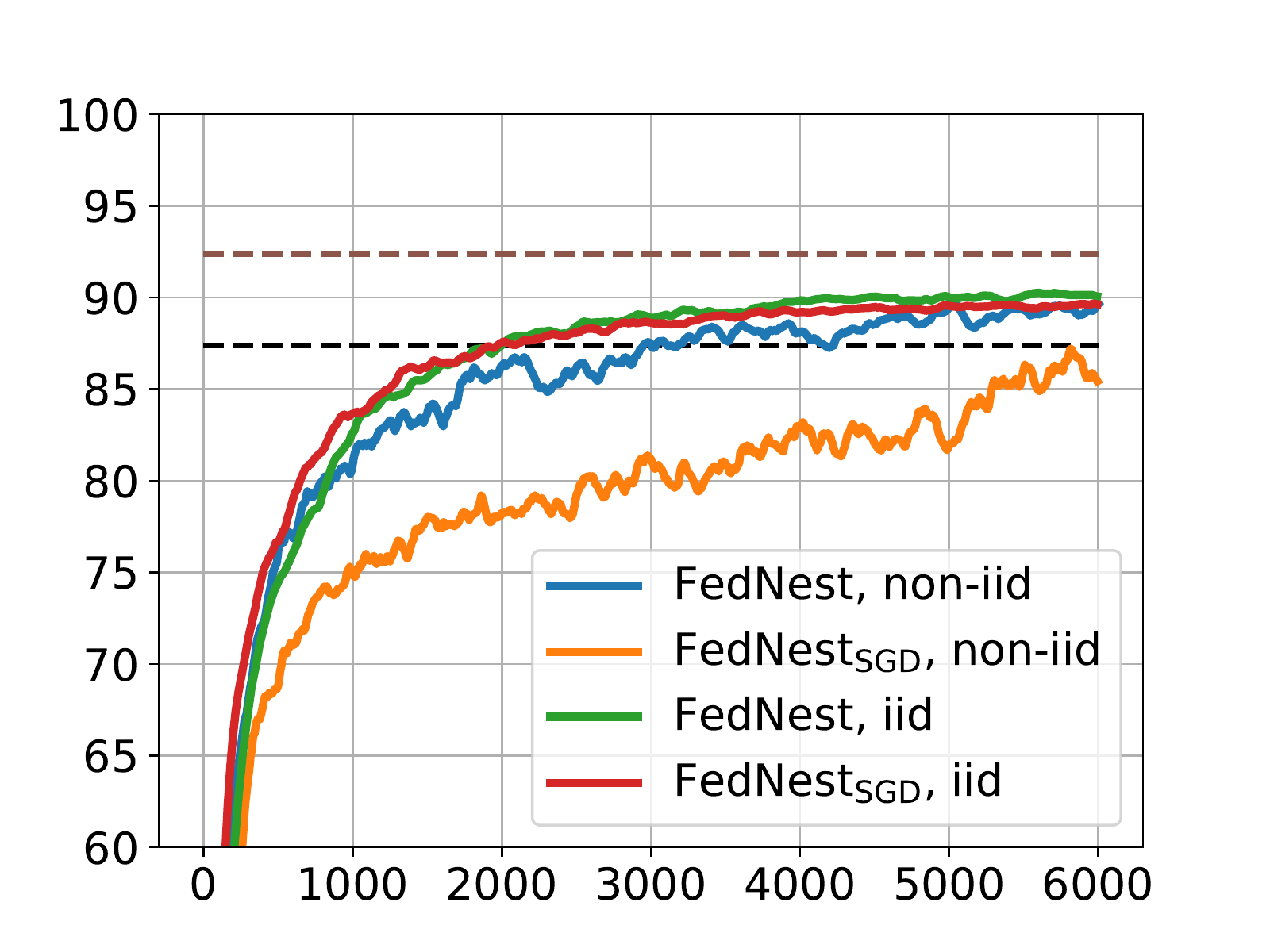}};
        \node at (0,-1.65) [scale=0.7]{Communication rounds};
        \end{tikzpicture}\vspace{-5pt}\caption{SVRG in $\fedinn$ provides better convergence and stability especially in non-iid setup.}\label{fig:lt_svrg_sgd}
    \end{subfigure}
    \vspace{-5pt}\caption{Loss function tuning on a 3-layer MLP and imbalanced MNIST dataset to maximize \emph{class-balanced test accuracy}. The \emph{brown dashed} line is the accuracy on non-federated bilevel optimization~\cite{li2021autobalance}, and the \emph{black dashed} line is the accuracy without tuning the loss function.}\label{fig:lt_exp}\vspace{-15pt}
\end{figure}
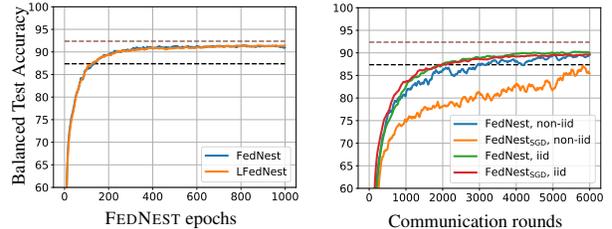

\subsection{Loss Function Tuning on Imbalanced Dataset}
We use bilevel optimization to tune a loss function for learning an imbalanced MNIST dataset. We aim to maximize the class-balanced validation accuracy (which helps minority/tail classes). Following the problem formulation in \cite{li2021autobalance}, we tune the so-called VS-loss function \cite{kini2021label} in a federated setting. In particular, we first create a long-tail imbalanced MNIST dataset by exponentially decreasing the number of examples per class (e.g.~class 0 has 6,000 samples, class 1 has 3,597 samples and finally, class 9 has only 60 samples). We partition the dataset to 100 clients following again $\fedavg$ \cite{mcmahan17fedavg} on both i.i.d.~and non-i.i.d.~setups. Different from the hyper-representation experiment, we employ 80\%-20\% train-validation on each client and use a 3-layer MLP model with 200, 100 hidden units, respectively. It is worth noting that, in this problem, the outer objective $f$ (aka validation cost) only depends on the hyperparameter $\m{x}$ through the optimal model parameters $\y^*(\x)$; thus, the direct gradient $\nabla^{\texttt{D}} f_i(\m{x},\m{y}^*(\m{x}))$ is zero for all $i \in\mc{S}$.

Figure~\ref{fig:lt_exp} displays test accuracy vs epochs/rounds for our federated bilevel algorithms. The horizontal dashed lines serve as non-FL baselines: brown depicts accuracy reached by bilevel optimization in non-FL setting, and, black depicts accuracy without any loss tuning. Compared to these, Figure~\ref{fig:lt_global_local} shows that $\fednest$ achieves near non-federated performance. In Figure~\ref{fig:lt_svrg_sgd}, we investigate the key role of SVRG in $\fedinn$ by comparing it with possible alternative implementation that uses SGD-type updates. The figure confirms our discussion in Section~\ref{sec:fedinn}: SVRG offers significant performance gains that are pronounced by client heterogeneity. 

\subsection{Federated Minimax Problem}
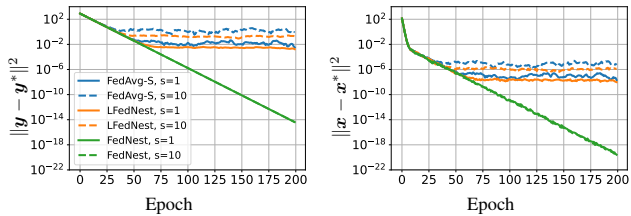
\begin{figure}
\centering
\vspace{-5pt}
\begin{subfigure}{0.225\textwidth}
\centering
        \begin{tikzpicture}
        \node at (0,0) {\includegraphics[scale=0.255]{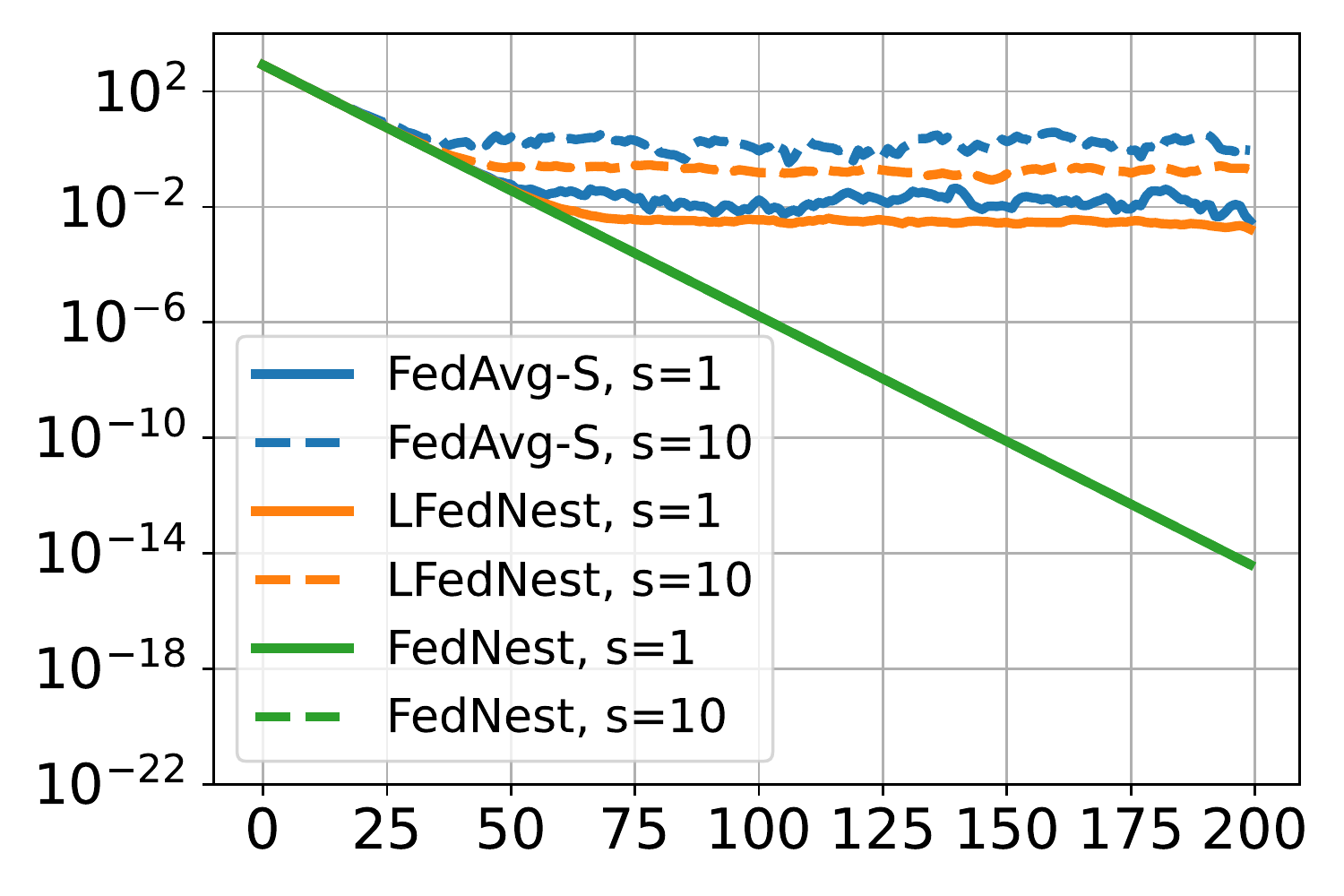}};
        \node at (0,-1.45) [scale=0.7]{Epoch};
        \node at (-2,0) [scale=0.7, rotate=90]{$\|\y-\y^*\|^2$};
        \end{tikzpicture}\vspace{-5pt}
\end{subfigure}
\hspace{7pt}
\begin{subfigure}{0.225\textwidth}
    \centering
        \begin{tikzpicture}
        \node at (0,0) {\includegraphics[scale=0.255]{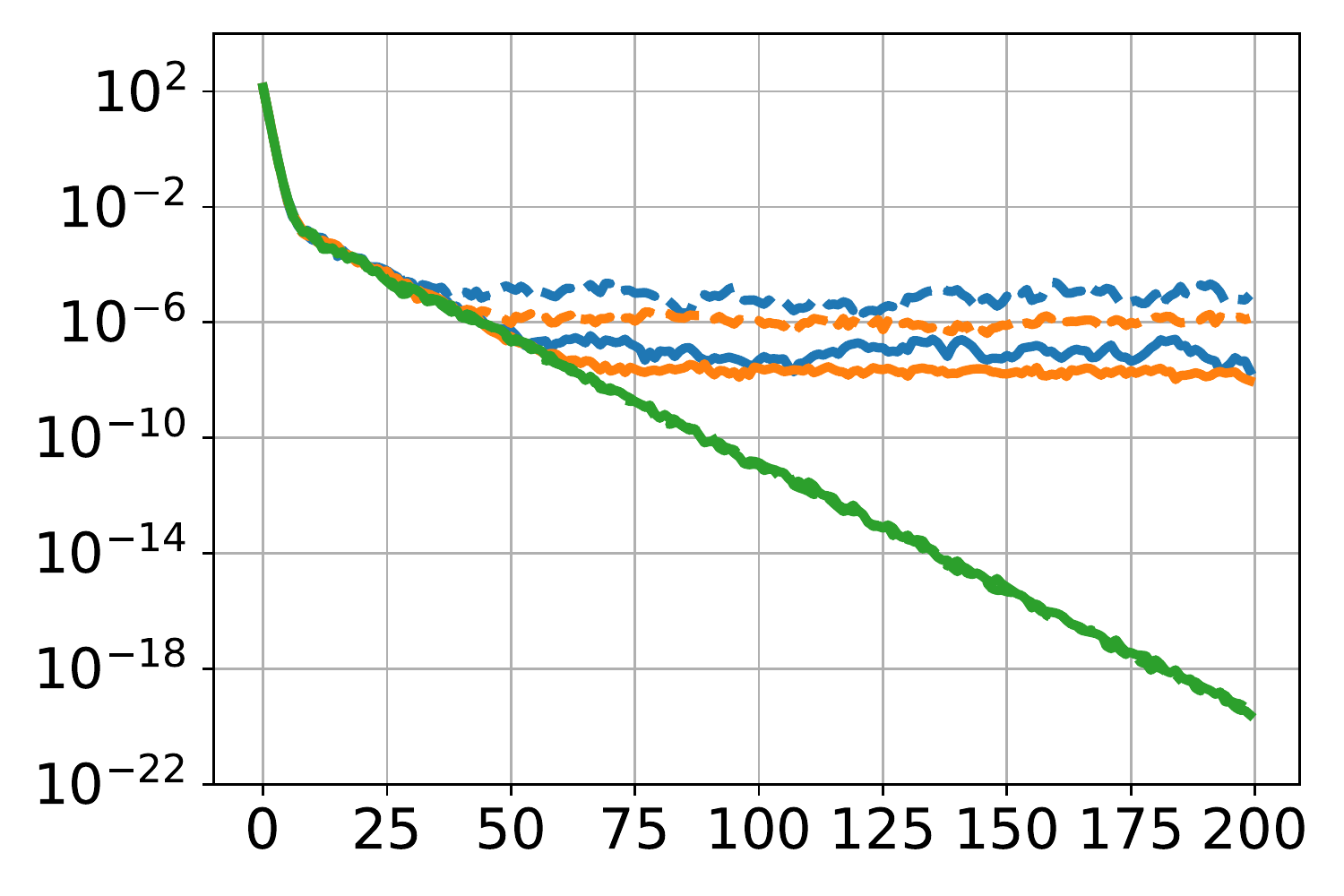}};
        \node at (0,-1.45) [scale=0.7]{Epoch};
        \node at (-2,0) [scale=0.7, rotate=90]{$\|\x-\x^*\|^2$};
        \end{tikzpicture}\vspace{-5pt}
\end{subfigure}\vspace{-5pt}
\caption{\textsc{FedNest} converges linearly despite heterogeneity. \textsc{LFedNest} slightly outperforms \textsc{FedAvg-S}.}\label{fig:mm}
\vspace{-10pt}
\end{figure}
We conduct experiments on the minimax problem \eqref{fedminmax:prob} with \[
f_i(\m{x},\m{y}):=-\left[\frac{1}{2}\|\m{y}\|^2 - \m{b}_i^\top \m{y} + \m{y}^\top \m{A}_i \m{x}\right] + \frac{\lambda}{2} \|\m{x}\|^2,
\]
to compare standard \textsc{FedAvg} saddle-point (\textsc{FedAvg}-S) method updating $(\m{x},\m{y})$ simultaneously~\cite{hou2021efficient} and our alternative approaches~(\textsc{LFedNest} and \textsc{FedNest}). This is a saddle-point formulation of $\min_{\m{x} \in \mb{R}^{d_1}} \frac{1}{2} \|\frac{1}{m} \sum_{i=1}^m \m{A}_i \m{x} - \m{b}_i\|^2$. 
We set  $\lambda=10$, $\m{b}_i=\m{b}_i' - \frac{1}{m}\sum_{i=1}^m \m{b}_i'$ and $\m{A}_i = t_i \m{I}$, where $\m{b}_i' \sim \mathcal{N}(0, s^2 I_d)$, and $t_i$ is drawn UAR over $(0,0.1)$. Figure~\ref{fig:mm} shows that \textsc{LFedNest} and \textsc{FedNest} outperform \textsc{FedAvg}-S thanks to their alternating nature. \textsc{FedNest} significantly improves the convergence of \textsc{LFedNest} due to controlling client-drift. To our knowledge, \textsc{FedNest} is the only \textit{alternating} federated SVRG for minimax problems. 

\section{Related Work}
\noindent\textbf{Federated learning.} \fedavg  was first introduced by \citet{mcmahan17fedavg}, who showed it can dramatically reduce communication costs. For identical clients, \fedavg coincides with local SGD \cite{zinkevich2010parallelized} which has been analyzed by many works \citep{stich2018local,yu2019parallel,wang2018cooperative}. Recently, many variants of \fedavg  have been proposed to tackle issues such as convergence and client drift. Examples include \textsc{FedProx}~\cite{li2020federated}, \scaffold \cite{karimireddy2020scaffold},  \textsc{FedSplit} \cite{pathak2020fedsplit}, \textsc{FedNova}~\cite{wang2020tackling}, and, the most closely relevant to us \textsc{FedLin} \cite{mitra2021linear}. A few recent studies are also devoted to the extension of \fedavg to the minimax optimization~\cite{rasouli2020fedgan,deng2020distributionally} and compositional optimization~\cite{huang2021compositional}. In contrast to these methods, \fednest makes alternating SVRG updates between the global variables $\m{x}$ and $\m{y}$, and yields sample complexity bounds and batch size choices that are on par with the non-FL guarantees~(Table~\ref{table:blo2:results}). Evaluations in the Appendix \ref{sec:joinvsalt} reveal that both alternating updates and SVRG provides a performance boost over these prior approaches.

\noindent\textbf{Bilevel optimization.} This class of problems was first introduced by \cite{bracken1973mathematical}, and since then, different types of approaches have been proposed. See~\cite{sinha2017review,liu2021investigating} for surveys. 
 Earlier works in \cite{aiyoshi1984solution,lv2007penalty} reduced the bilevel problem to a single-level optimization problem. However, the reduced problem is still difficult to solve due to for example a large number of constraints. Recently, more efficient gradient-based  algorithms have been proposed by estimating the hypergradient of $\nabla f(\m{x})$ through iterative updates~\citep{maclaurin2015gradient,franceschi2017forward,domke2012generic,pedregosa2016hyperparameter}. The asymptotic and  non-asymptotic analysis of bilevel optimization has been provided in~\citep{franceschi2018bilevel,shaban2019truncated,liu2020generic} and \citep{ghadimi2018approximation,hong2020two}, respectively. There is also a line of work focusing on minimax optimization \cite{nemirovski2004prox,daskalakis2018limit} and compositional optimization~\cite{wang2017stochastic}. Closely related to our work are \cite{lin2020gradient,rafique2021weakly,chen2021closing} and \cite{ghadimi2020single,chen2021closing} which provide non-asymptotic analysis of SGD-type methods for minimax and compositional problems with outer nonconvex objective, respectively. 

A more in-depth discussion of related work is given in Appendix~\ref{supp:sec:related}. We summarize the complexities of different methods for FL/non-FL bilevel optimization in Table~\ref{table:blo2:results}.
\section{Conclusions}\label{sec:conc}
We presented a new class of federated algorithms for solving general nested stochastic optimization spanning bilevel and \minmax~problems. \fednest runs a variant of federated SVRG on inner \& outer variables in an alternating fashion. We established provable convergence rates for \fednest under arbitrary client heterogeneity and introduced variations for min-max and compositional problems and for improved communication efficiency (\lfedblo). We showed that, to achieve an $\epsilon$-stationary point of the nested problem, \fednest requires $O(\epsilon^{-2})$ samples in total, which matches the complexity of the non-federated nested algorithms in the literature. 

\vspace{-5pt}
\section*{Acknowledgements}\vspace{-5pt}
Davoud Ataee Tarzanagh was supported by ARO YIP award W911NF1910027 and NSF CAREER award CCF-1845076. Christos Thrampoulidis was supported by NSF Grant Numbers CCF-2009030 and HDR-1934641, and an NSERC Discovery Grant. Mingchen Li and Samet Oymak were supported by the NSF CAREER award CCF-2046816, Google Research Scholar award, and ARO grant W911NF2110312.

\newpage

\bibliography{blo.bib}
\bibliographystyle{icml/icml2022.bst}

\newpage
\appendix
\onecolumn
\icmltitle{ APPENDIX \\ \fedblo: Federated~\Blo,~\Minmax, and Compositional~Optimization}
The appendix is organized as follows: Section~\ref{sec:simp:algs} introduces the \lfedblo algorithm. Section~\ref{supp:sec:related} discusses the related work. We provide all details for the proof of the  main theorems in Sections~\ref{sec:app:bilevel}, \ref{sec:app:minmax}, \ref{sec:app:compos}, and \ref{sec:app:sing} for federated bilevel, \minmax, compositional, and single-level optimization, respectively. In Section~\ref{sec:techn}, we state a few auxiliary technical lemmas. Finally, in Section~\ref{sec:app:experim}, we provide the detailed parameters of our numerical experiments (Section~\ref{sec:numerics}) and then introduce further experiments.
%
\section{\lfedblo}\label{sec:simp:algs}
Implementing \fedinn and \fedout naively by using the global direct and indirect gradients and sending the local information to the server that would then calculate the global gradients leads to a communication and space complexity of which can be  prohibitive for large-sized $d_1$ and $d_2$. One can consider possible local variants of \fedinn and \fedout tailore to such scenarios. Each of the possible algorithms (See Table~\ref{tabl:supp:methods}) can then either use the global gradient or only the local gradient, either use a SVRG or SGD. 
\begin{algorithm}[h]
\caption{$\m{x}^{+} ~=~\pmb{\lfedout}~(\m{x}, \m{y}, \alpha)$ for stochastic \colorbox{cyan!30}{bilevel}, \colorbox{green!30}{minimax}, and \colorbox{magenta!30}{compositional} problems
}
\begin{small}
\begin{algorithmic}[1]
 \State $\m{x}_{i,0}=\m{x}$ and $\alpha_i \in (0,\alpha]$ for each $i \in \mc{S}$.
 \State Choose $N \in \mb{N}$ (the number of terms of Neumann series).
 \For {$i \in \mc{S}$ \textbf{in parallel}} 
\For {$\nu=0,\ldots,\tau_i-1$} 
 \State
\colorbox{cyan!30}{Select $ N'\in \{0, \dots, N-1\}$ UAR.}
 \State
\colorbox{cyan!30}{$ \m{h}_{i,\nu }=\nabla_\m{x} f_i(\m{x}_{i,\nu },\m{y}; \xi_{i,\nu })-  \frac{N}{\ell_{g,1}} \nabla^2_{\m{xy}}g_i(\m{x}_{i,\nu },\m{y};{\zeta}_{i,\nu })  \prod\limits_{n=1}^{N'} \big(\m{I}-\frac{1}{\ell_{g,1}}\nabla^2_y g_i(\m{x}_{i,\nu }, \m{y};{\zeta}_{i,n})\big) \nabla_\m{x} f_i(\m{y}_{i,\nu },\m{y},\xi_{i,\nu })$}
 \State\colorbox{green!30}{$\m{h}_{i,\nu }=\nabla_{\m{x}} f_i(\m{x}_{i,\nu }, \m{y}; \xi_{i,\nu })$}
 \State \colorbox{magenta!30}{$\m{h}_{i,\nu }=\nabla \m{r}_i(\m{x}_{i,\nu };  \zeta_{i,\nu })^\top \nabla f_i(\m{y}_{i,\nu };  \xi_{i,\nu })$}
 \State  $ \m{x}_{i,\nu +1}= \m{x}_{i,\nu }-\alpha_i\m{h}_{i,\nu }$
  \EndFor
\EndFor
\State $\m{x}^{+}=|\mc{S}|^{-1}\sum_{i\in \mc{S}} \m{x}_{i,\tau_i}$
\end{algorithmic}
\end{small}
\label{alg:localfedout}
\end{algorithm}
\begin{algorithm}[h]
\caption{$ \m{y}^+~=~\pmb{\lfedinn}~(\m{x},\m{y},\beta$) 
for stochastic \colorbox{cyan!30}{bilevel}, \colorbox{green!30}{minimax}, and \colorbox{magenta!30}{compositional} problems}
\begin{small}
\begin{algorithmic}[1]
 \State $\m{y}_{i,0}=\m{y}$ and $\beta_i \in (0,\beta]$ for each $i \in \mc{S}$.
 \For {$i \in \mc{S}$ \textbf{in parallel}} 
\For {$\nu=0,\ldots,\tau_i-1$} 
\State \colorbox{cyan!30}{$ \m{q}_{i,\nu }=\nabla_\m{y} g_i(\m{x}, \m{y}_{i,\nu };\zeta_{i,\nu })$} \colorbox{green!30}{$ \m{q}_{i,\nu }=-\nabla_\m{y} f_i(\m{x}, \m{y}_{i,\nu };\xi_{i,\nu })$} \colorbox{magenta!30}{$ \m{q}_{i,\nu }=\m{y}_{i,\nu }-\m{r}_i(\m{x};\zeta_{i,\nu })$}
\State $\m{y}_{i,\nu +1}= \m{y}_{i,\nu }-\beta_i\m{q}_{i,\nu }$
\EndFor
\EndFor
\State $\m{y}^{+}=|\mathcal{S}|^{-1}\sum_{i\in\mathcal{S}}\m{y}_{i,\tau_i}$
\end{algorithmic}
\end{small}
\label{alg:localfedinn}
\end{algorithm}
\begin{table*}[h]
\centering
\scalebox{.9}{
\begin{tabular}{lccccccc}
\toprule 
 & \multicolumn{2}{c}{definition} & & \multicolumn{4}{c}{properties} \\ \cmidrule{2-3}\cmidrule{5-8} 
                            & outer & inner  & & global & global & global &  \#~communication \\
                       & optimizer  & optimizer & & outer gradient & IHGP& inner gradient & rounds \\
                       \midrule
\multirow{ 2}{*}{\pmb{\fedblo}} &Algorithm~\ref{alg:fedout}&Algorithm~\ref{alg:fedinn} && \multirow{ 2}{*}{yes}& \multirow{ 2}{*}{yes} & \multirow{ 2}{*}{yes}& \multirow{ 2}{*}{$2T+N+3$} \\
& (SVRG on $\m{x})$ & (SVRG on $\m{y})$ && &  & & \\
\hline
\multirow{ 2}{*}{\pmb{\lfedblo}}&Algorithm~\ref{alg:localfedout}&Algorithm~\ref{alg:localfedinn} && \multirow{ 2}{*}{no}& \multirow{ 2}{*}{no} &\multirow{ 2}{*}{ no}& \multirow{ 2}{*}{$T+1$} \\
& (SGD on $\m{x})$ & (SGD on $\m{y})$ && &  & & \\
\hline
\multirow{ 2}{*}{$\pmb{\fedblo}_\textbf{SGD}$} &Algorithm~\ref{alg:fedout}&Algorithm~\ref{alg:localfedinn} && \multirow{ 2}{*}{yes}& \multirow{ 2}{*}{yes} & \multirow{ 2}{*}{no}& \multirow{ 2}{*}{$T+N+3$} \\
& (SVRG on $\m{x})$ & (SGD on $\m{y})$ && &  & & \\

\hline
\multirow{ 2}{*}{$\pmb{\lfedblo}_\textbf{SVRG}$}&Algorithm~\ref{alg:localfedout}&Algorithm~\ref{alg:fedinn} && \multirow{ 2}{*}{no}& \multirow{ 2}{*}{no} &\multirow{ 2}{*}{ yes}& \multirow{ 2}{*}{$2T+1$} \\
& (SGD on $\m{x})$ & (SVRG on $\m{y})$ && &  & & \\

  \bottomrule
\end{tabular}
}
\caption{Definition of studied algorithms by using inner/outer optimization algorithms and server updates and resulting properties of these algorithms. $T$ and $N$ denote the number of inner iterations and terms of Neumann series, respectively.}
\label{tabl:supp:methods}
\vspace{-.3cm}
\end{table*}
\section{ Related Work}\label{supp:sec:related}
We provide an overview of the current literature on  non-federated nested (bilevel, minmimax, and compositional) optimization and federated learning.


\subsection{Bilevel Optimization} 

A broad collection of algorithms have been proposed to solve bilevel nonlinear programming problems.  \citet{aiyoshi1984solution,edmunds1991algorithms,al1992global,hansen1992new,shi2005extended,lv2007penalty,moore2010bilevel} reduce the bilevel problem to a single-level optimization problem using for example the Karush-Kuhn-Tucker (KKT) conditions or penalty function methods.  A similar idea was also explored in \citet{khodak2021federated} where the authors provide a reformulation of the hyperparameter optimization (bilevel objective) into a {single-level} objective and develop a federated online method to solve it. However, the reduced single-level problem is usually difficult to solve~\citep{sinha2017review}. 

In comparison, alternating gradient-based approaches designed for the bilevel problems are more attractive due to their simplicity and effectiveness. This type of approaches estimate the hypergradient $\nabla f(\m{x})$ for iterative updates, and are generally divided to approximate implicit differentiation (AID) and iterative differentiation (ITD) categories. ITD-based approaches~\citep{maclaurin2015gradient,franceschi2017forward,finn2017model,grazzi2020iteration} estimate the hypergradient $\nabla f(\m{x})$ in either a reverse (automatic differentiation) or forward manner. AID-based approaches~\citep{pedregosa2016hyperparameter,grazzi2020iteration,ghadimi2018approximation} estimate the hypergradient via implicit differentiation which involves solving a linear system. Our algorithms follow the latter approach.  

Theoretically, bilevel optimization has been studied via both asymptotic and non-asymptotic analysis \cite{franceschi2018bilevel,liu2020generic,li2020improved,shaban2019truncated,ghadimi2018approximation,ji2021bilevel,hong2020two}. In particular, \cite{franceschi2018bilevel} provided the asymptotic convergence of a backpropagation-based approach as one of ITD-based algorithms by assuming the inner problem is strongly convex. \cite{shaban2019truncated} gave a similar analysis for a \textit{truncated} backpropagation approach. Non-asymptotic complexity analysis for bilevel optimization has also been explored. \citet{ghadimi2018approximation} provided a finite-time convergence analysis for an AID-based algorithm under three different loss geometries, where $f(\cdot)$ is either strongly convex, convex or nonconvex, and $g(\m{x},\cdot)$ is strongly convex. \cite{ji2021bilevel} provided an improved non-asymptotic analysis for AID- and ITD-based algorithms under the nonconvex-strongly-convex geometry. \cite{Ji2021LowerBA} provided the first-known lower bounds on complexity as well as tighter upper bounds. When the objective functions can be expressed in an expected or finite-time form, \cite{ghadimi2018approximation,ji2021bilevel,hong2020two} developed stochastic bilevel algorithms and provided the non-asymptotic analysis. \citep{chen2021closing} provided a tighter analysis of SGD for stochastic bilevel problems. \cite{chen2021single,guo2021stochastic,Khanduri2021ANA,Ji2020ProvablyFA,huang2021biadam,dagreou2022framework} studied accelerated SGD, SAGA, momentum, and adaptive-type bilevel optimization methods. More results can be found in the recent review paper~\citep{liu2021investigating} and references therein. 

\subsubsection{Minimax Optimization}
Minimax optimization has a long history dating back to \cite{brown1951iterative}. Earlier works focused on the deterministic convex-concave regime \cite{nemirovski2004prox,nedic2009subgradient}. Recently, there has emerged a surge of studies of stochastic minimax problems. The alternating version of the gradient descent ascent (SGDA) has been studied by incorporating the idea of optimism \cite{daskalakis2018limit,gidel2018variational,mokhtari2020unified,yoon2021accelerated}.  \cite{rafique2021weakly,thekumparampil2019efficient,nouiehed2019solving,lin2020gradient} studied  SGDA in the nonconvex-strongly concave setting. Specifically, the ${\cal O}(\epsilon^{-2})$ sample complexity has been established in \citep{lin2020gradient} under an increasing batch size ${\cal O}(\epsilon^{-1})$. \citet{chen2021closing} provided the ${\cal O}(\epsilon^{-2})$ sample complexity under an ${\cal O}(1)$ constant batch size. In the same setting, accelerated GDA algorithms have been developed in \cite{luo2020stochastic,yan2020optimal,tran2020hybrid}.    
Going beyond the one-side concave settings, algorithms and their convergence analysis have been studied for nonconvex-nonconcave minimax problems with certain benign structure; see e.g., \cite{gidel2018variational,liu2019towards,yang2020global,diakonikolas2021efficient,barazandeh2021solving}. A comparison of our results with prior work can be found in Table \ref{table:blo2:results}.

\subsubsection{Compositional Optimization}
Stochastic compositional gradient algorithms \citep{wang2017stochastic,wang2016accelerating} can be viewed as an  alternating SGD for the special compositional problem. However, to ensure convergence, the algorithms in \citep{wang2017stochastic,wang2016accelerating} use two sequences of variables being updated in two different time scales, and thus the iteration complexity of \citep{wang2017stochastic} and \citep{wang2016accelerating} is worse than ${\cal O}(\epsilon^{-2})$ of the standard SGD. Our work is closely related to ALSET \cite{chen2021closing}, where an $\mc{O}(\epsilon^{-2})$ sample complexity has been established in a non-FL setting.
\subsection{Federated Learning}

FL involves learning a centralized model from distributed client
data. Although this centralized model benefits from all client data, it raises several types of issues such as generalization, fairness, communication efficiency, and privacy \cite{mohri2019agnostic,stich2018local,yu2019parallel,wang2018cooperative,stich2019error, basu2019qsparse,nazari2019dadam,barazandeh2021decentralized}. \fedavg~\citep{mcmahan17fedavg} can tackle some of these issues such as high communication costs. Many variants of \fedavg have been proposed to tackle other emerging issues such as convergence and \textit{client drift}. Examples include adding a regularization term in the client objectives towards the broadcast model \citep{li2020federated}, proximal splitting \cite{pathak2020fedsplit,mitra2021linear}, variance reduction \cite{karimireddy2020scaffold,mitra2021linear}, dynamic regularization \cite{acar2021federated}, and adaptive updates \cite{reddi2020adaptive}. When clients are homogeneous, \fedavg is closely related to local SGD \citep{zinkevich2010parallelized}, which has been analyzed by many works \citep{stich2018local,yu2019parallel,wang2018cooperative,stich2019error, basu2019qsparse}.

In order to analyze \fedavg in heterogeneous settings,  \citep{li2020federated, wang2019adaptive, khaled2019first, li2019convergence} derive convergence rates depending on the amount of heterogeneity. They showed that the convergence rate of \fedavg gets worse with client heterogeneity. By using control variates to reduce client drift, the \scaffold method \citep{karimireddy2020scaffold} achieves convergence rates that are independent of the amount of heterogeneity. Relatedly, \textsc{FedNova}~\cite{wang2020tackling} and  \textsc{FedLin} \cite{mitra2021linear} provided the convegence of their methods  despite arbitrary local objective and systems heterogeneity. In particular, \cite{mitra2021linear} showed that \textsc{FedLin} guarantees linear convergence to the global minimum of deterministic objective, despite arbitrary objective and systems heterogeneity.  As explained in the main body, our algorithms critically leverage these ideas after identifying the additional challenges that client drift brings to federated bilevel settings. 

\subsubsection{Federated minimax Learning}

A few recent studies are devoted to federated minimax optimization \cite{rasouli2020fedgan, reisizadeh2020robust,deng2020distributionally,hou2021efficient}. In particular, \cite{reisizadeh2020robust} consider minimax problem with inner problem satisfying PL condition and the outer one being either nonconvex or satisfying PL. However, the proposed algorithm only communicates $\m{x}$ to the server. \citet{xie2021federated} consider a general class of nonconvex-PL minimax problems in the cross-device federated learning setting. Their algorithm performs multiple local update steps on a subset of active clients in each round and leverages global gradient estimates to correct the bias in local update directions. \citet{deng2021local} studied federated optimization for a family of smooth nonconvex minimax functions. \citet{shen2021fedmm} proposed a distributed minimax optimizer called \textsc{FedMM}, designed specifically for the federated adversary domain adaptation problem. \citet{hou2021efficient} proposed a SCAFFOLD saddle point algorithm (SCAFFOLD-S) for solving strongly convex-concave minimax problems in the federated setting.  To the best of our knowledge, all the aforementioned developments require a bound on the \textit{heterogeneity} of the local functions, and do not account for the effects of systems heterogeneity which is also a key challenge in FL. In addition, our work proposes the first \textit{alternating} federated SVRG-type algorithm for minimax problems with iteration complexity that matches to the non-federated setting (see, Table~\ref{table:blo2:results}).   
\section{Proof for Federated Bilevel Optimization}\label{sec:app:bilevel}

Throughout the proof, we will use $\mathcal{F}^{k,t}_{i,\nu}$ to denote the filtration that captures all the randomness up to the $\nu$-th local step of client $i$ in inner round $t$ and outer round $k$. With a slight abuse of notation, $\mathcal{F}^{k,t}_{i,-1}$ is to be interpreted as $\mathcal{F}^{k,t}, \forall i \in \mc{S}$. For simplicity, we remove subscripts $k$ and $t$ from the definition of stepsize and model parameters. For example, $\m{x}$ and $\m{x}^+$ denote $\m{x}^k$ and $\m{x}^{k+1}$, respectively. 

We analyze the convergence of \fedblo for general setting: starting from a common global model $\m{x}_{i,0}=\m{x}$, each client $i$ performs $\tau_i$ local steps (in parallel):
\begin{align}\label{app:eqn:update1:fedout}
\m{x}_{i,\nu +1} = \m{x}_{i,\nu }-\alpha_i \left( \m{h}(\m{x},\m{y}^{+})-\m{h}_i(\m{x},\m{y}^{+})+\m{h}_i(\m{x}_{i,\nu },\m{y}^{+})\right),
\end{align}
and then the server aggregates local models via $\m{x}^+=|\mc{S}|^{-1} \sum_{i\in \mc{S}} \m{x}_{i,\tau_i}$.   Here, $\alpha_i \in (0,\alpha]$ is the local stepsize, and $\m{h}(\m{x},\m{y}) := |\mc{S}|^{-1}\sum_{i\in\mathcal{S}} \m{h}_i^{\texttt{I}} (\m{x},\m{y})+ \m{h}_i^{\texttt{D}} (\m{x},\m{y})$. The proof for the special case  \eqref{eqn:update1:fedout}--\eqref{eqn:hyper:estim} where $\m{h}_{i,\nu}$ is used as a search direction follows similarly. Note that from \eqref{app:eqn:update1:fedout}, we have  
\begin{align}\label{eqn1:lem:dec}
\nonumber 
     \m{x}^{+}&=\m{x}-\frac{1}{m}  \sum\limits_{i=1}^{m} \alpha_i \sum\limits_{\nu=0}^{\tau_i-1}  \left( \m{h}(\m{x},\m{y}^{+})-\m{h}_i(\m{x},\m{y}^{+})+\m{h}_i(\m{x}_{i,\nu },\m{y}^{+})\right) \\
    &=\m{x}-\frac{1}{m}  \sum\limits_{i=1}^{m} \alpha_i \sum\limits_{\nu=0}^{\tau_i-1} \m{h}_i(\m{x}_{i,\nu}, \m{y}^+).
 \end{align}
We further set 
 \begin{align}\label{eqn:expe:h}
 \bar{\m{h}}_i (\m{x}_{i,\nu}, \m{y}^+)&:=\mb{E}\Big[\m{h}_i(\m{x}_{i,\nu}, \m{y}^+)|{\cal F}_{i,\nu-1}\Big].    
 \end{align}
\subsection*{Proof of Lemma~\ref{lem:IFT}}
\begin{proof}
Given $\m{x}\in \mb R^{d_1}$, the optimality condition of the inner problem in \eqref{fedblo:prob} is $\nabla_\m{y}g(\m{x},\m{y})=0$. Now, since $\nabla_\m{x}\left( \nabla_\m{y} g(\m{x},\m{y})\right)=0$, we obtain
$$
0=\sum_{j=1}^m \left( \nabla^2_{\m{xy}}g_j\left(\m{x},\m{y}^*(\m{x})\right)+\nabla \m{y}^*(\m{x}) \nabla^2_{\m{y}}g_j\left(\m{x},\m{y}^*(\m{x})\right) \right),
$$
which implies
\begin{equation*}
   \nabla \m{y}^*(\m{x})=- \bigg(\sum_{i=1}^m \nabla^2_{\m{xy}}g_i\left(\m{x},\m{y}^*(\m{x})\right)\bigg) \bigg(
\sum_{i=1}^m\nabla_{\m{y}}^2 {g_j}(\m{x},\m{y}^*(\m{x}))\bigg)^{-1}.
\end{equation*}
The results follows from a simple application of the chain rule to $f$ as follows: 
$$ \nabla f  \left(\m{x},\m{y}^*(\m{x})\right)= \nabla_{\m{x}}f\left(\m{x},\m{y}^*(\m{x})\right)+\nabla \m{y}^*(\m{x})\nabla_{\m{y}}f\left(\m{x},\m{y}^*(\m{x})\right).
$$
\end{proof}
\subsection*{Proof of Lemma~\ref{lem:neum:bias}}
\begin{proof}
By independency of $N'$, $\zeta_{i,n}$, and $\mc{S}_n$, and under Assumption~\ref{assu:bound:var}, we have
\begin{align}
\nonumber 
\mb{E}_{\mc{W}}\left[\wh{\m{H}}_\m{y}\right] &= \mb{E}_{\mc{W}}\left[ \frac{ N}{\ell_{g,1}} \prod_{n=1}^{N'} \left (\m{I} - \frac{1}{ \ell_{g,1}  |\mc{S}_n|} \sum_{i=1}^{|\mc{S}_n|} \nabla_{\m{y}}^2 g_i(\m{x},\m{y}; \zeta_{i,n}) \right)\right]\\
\nonumber 
&=\mb{E}_{N'} \left[  \mb{E}_{\mc{S}_{1:N'}} \left[
\mb{E}_{\zeta} \left[ \frac{ N}{\ell_{g,1}} \prod_{n=1}^{N'} \left (\m{I} - \frac{1}{ \ell_{g,1}  |\mc{S}_n|} \sum_{i=1}^{|\mc{S}_n|} \nabla_{\m{y}}^2 g_i(\m{x},\m{y}; \zeta_{i,n}) \right)\right]\right]\right]\\
&=\tfrac{1}{\ell_{g,1}}\sum_{n=0}^{N-1} \left[\m{I} - \tfrac{1}{\ell_{g,1}}\nabla^2_{\m{y}} g(\m{x}, \m{y})\right]^{n},\label{lem_st_p1}
\end{align}
where the last equality follows from the uniform distribution of $N'$. 

Note that since $\m{I} \succeq \tfrac{1}{\ell_{g,1}} \nabla^2_{\m{y}} g_i \succeq \frac{\mu_g}{\ell_{g,1}}$ for all $i \in [m]$ due to Assumption~\ref{assu:f}, we have
\begin{align*}
\mb{E}_{\mc{W}}\left[\|\wh{\m{H}}_{\m{y}}\|\right] &\le \frac{N}{\ell_{g,1}}\mb{E}_{\mc{W}}\left[\prod_{n=1}^{N'}  \left\|\m{I} - \frac{1}{ \ell_{g,1}  |\mc{S}_n|} \sum_{i=1}^{|\mc{S}_n|} \nabla_{\m{y}}^2 g_i(\m{x},\m{y}; \zeta_{i,n})\right\| \right]  \\
&\leq \frac{N}{\ell_{g,1}} \mb{E}_{N'} \left[1- \frac{\mu_g}{\ell_{g,1}} \right]^{N'}=
\frac{1}{\ell_{g,1}}\sum_{n=0}^{N-1} \left[1- \frac{\mu_{g}}{\ell_{g,1}}\right]^n\le \frac{1}{\mu_g}.
\end{align*}
The reminder of the proof is similar to \cite{ghadimi2018approximation}.
\end{proof}

The following lemma extends \citep[Lemma 2.2]{ghadimi2018approximation} and \citep[Lemma~2]{chen2021closing} to the finite-sum problem~\eqref{fedblo:prob}. Proofs follow similarly by applying their analysis to the inner \& outer functions $(f_i, g_i)$, $\forall i \in \mc{S}$.

\begin{lemma}\label{lem:lips} 
Under Assumptions~\ref{assu:f} and \ref{assu:bound:var}, for all $\x_1$, $\x_2$:
\begin{subequations}\label{eqn:newlips}
\begin{align}
\|\nabla f(\m{x}_1)-\nabla f(\m{x}_2)\|
&	\leq  L_f\|\m{x}_1-\m{x}_2\|, \label{eqn:newlips:b}\\
	\|\m{y}^*(\m{x}_1)-\m{y}^*(\m{x}_2)\|&\leq  L_{\m{y}}\|\m{x}_1-\m{x}_2\|, \label{eqn:newlips:c}\\
\|\nabla \m{y}^*(\m{x}_1)-\nabla \m{y}^*(\m{x}_2)\|&\leq  L_{\m{yx}}\|\m{x}_1-\m{x}_2\|. \label{eqn:newlips:d}
\end{align}
Also, for all $i \in\mc{S}$, $\nu \in \{0,\ldots, \tau_i-1\}$,  $\x_1$, $\x_2$, and $\m{y}$, we have:
\begin{align}
\|\bar{\nabla} f_i(\m{x}_1,\m{y}) - \bar{\nabla} f_i (\m{x}_1,\m{y}^*(\m{x}_1)) \|
&\leq  M_f\|\m{y}^*(\m{x}_1)-\m{y}\|, \label{eqn:newlips:a}\\
\|\bar{\nabla} f_i(\m{x}_2,\m{y}) - \bar{\nabla} f_i (\m{x}_1,\m{y}) \|
&\leq  M_f\|\m{x}_2-\m{x}_1\|, \label{eqn:newlips:aa}\\
\mb{E}\left[\|\bar{\m{h}}_i (\m{x}_{i,\nu}, \m{y})- \m{h}_i(\m{x}_{i,\nu}, \m{y})\|^2\right] &\leq \tilde\sigma_f^2, \label{eqn:newlips:e}
\\
\mb{E}\left[\|\m{h}_i(\m{x}_{i,\nu}, \m{y}^+)\|^2 |{\mc{F}}_{i,\nu-1} \right] &\leq \tilde{D}_f^2.~~~ \label{eqn:newlips:f}
\end{align}
\end{subequations}
Here, 
\begin{equation}\label{eqn:lip:condi}
\begin{aligned}
L_{\m{y}} &:=\frac{\ell_{g,1}}{\mu_g}=\mc{O}(\kappa_g),\\
L_{\m{yx}}&:=\frac{\ell_{g,2}+\ell_{g,2}L_{\m{y}}}{\mu_g} + \frac{\ell_{g,1}}{\mu_g^2} \Big(\ell_{g,2}+\ell_{g,2}L_{\m{y}}\Big)=\mc{O}(\kappa_g^3),\\
M_f &:=\ell_{f,1} + \frac{\ell_{g,1}\ell_{f,1}}{\mu_g} + \frac{\ell_{f,0}}{\mu_g}\left(\ell_{g,2}+\frac{\ell_{g,1}{\ell_{g,2}}}{\mu_g}\right)=\mc{O}(\kappa_g^2),~~~
\\
L_f&:=\ell_{f,1} + \frac{\ell_{g,1}(\ell_{f,1}+M_f)}{\mu_g} + \frac{\ell_{f,0}}{\mu_g}\left(\ell_{g,2}+\frac{\ell_{g,1}{\ell_{g,2}}}{\mu_g}\right)=\mc{O}(\kappa_g^3),\\
\tilde{\sigma}_f^2&:=\sigma_f^2 + \frac{3}{\mu_g^2}\Big((\sigma_f^2+\ell_{f,0}^2)(\sigma_{g,2}^2+2\ell_{g,1}^2)+\sigma_f^2\ell_{g,1}^2\Big),\\
\tilde{D}_f^2&:= \left(\ell_{f,0} + \frac{\ell_{g,1}}{\mu_g}\ell_{f,1} + \ell_{g,1}\ell_{f,1}\frac{1}{\mu_g}\right)^2 + \tilde\sigma_f^2=\mc{O}(\kappa_g^2),
\end{aligned}
\end{equation}
where the other constants are provided in Assumptions~\ref{assu:f} and \ref{assu:bound:var}.
\end{lemma}
%

%
\subsection{Descent of Outer Objective}

The following lemma characterizes the descent of the outer objective. 

\begin{lemma}[Descent Lemma]\label{lem:dec}
Suppose Assumptions~\ref{assu:f} and \ref{assu:bound:var} hold. Further, assume $\tau_i \geq 1$ and $\alpha_i=\alpha/\tau_i, \forall i \in \mathcal{S}$ for some positive constant $\alpha$. Then, \fedout guarantees:
\begin{equation}\label{eqn0:lem:dec}
    \begin{aligned}
\mb{E}\left[f(\m{x}^{+})\right] - \mb{E}\left[f(\m{x})\right] \leq
 &- \frac{\alpha}{2} \mb{E}\left[\| \nabla f(\m{x})\|^2\right] +  \frac{\alpha^2 L_f}{2} \tilde{\sigma}_f^2 \\
  &- \frac{\alpha}{2} \left(1- \alpha L_f\right) \mb{E}\left[\left\|  \frac{1}{m}\sum\limits_{i=1}^{m} \frac{1}{\tau_i} \sum\limits_{\nu=0}^{\tau_i-1} \bar{\m{h}}_i(\m{x}_{i,\nu}, \m{y}^+)  \right\|^2\right]\\
   & +\frac{3\alpha }{2}  \left( b^2 +  M_f^2 \mb{E}\left[\|\m{y}^+-\m{y}^*(\m{x})\|^2\right] + \frac{ M_f^2 }{m} \sum\limits_{i=1}^{m}   \frac{1}{\tau_i} \sum\limits_{\nu=0}^{\tau_i-1}  \mb{E}\left[\|\m{x}_{i,\nu} -\m{x}\|^2\right]\right). 
\end{aligned}
\end{equation}
\end{lemma}
\begin{proof}
Using \eqref{eqn1:lem:dec} and the Lipschitz property of $\nabla f$ in Lemma~\ref{lem:lips}, we have
\begin{equation}
    \begin{aligned}
\mb{E}\left[f(\m{x}^{+})\right] - \mb{E} \left[f(\m{x}) \right]
    &\leq   \mb{E} \left[\langle \m{x}^{+} - \m{x}, \nabla f(\m{x}) \rangle\right] + \frac{L_f}{2} \mb{E} \left[\Vert \m{x}^{+}-\m{x} \Vert^2 \right]\\
    & = - \mb{E}\left[ \left\langle \frac{1}{m}\sum\limits_{i=1}^{m} \alpha_i \sum\limits_{\nu=0}^{\tau_i-1} \m{h}_i(\m{x}_{i,\nu}, \m{y}^+) , \nabla f(\m{x}) \right\rangle \right]\\
    &+ \frac{L_f}{2}  \mb{E}\left[ \left\| \frac{1}{m}\sum\limits_{i=1}^{m} \alpha_i \sum\limits_{\nu=0}^{\tau_i-1} \m{h}_i(\m{x}_{i,\nu}, \m{y}^+)\right\|^2\right].
    \end{aligned}
\label{eqn2:lem:dec}
\end{equation}

In the following, we bound each term on the right hand side (RHS) of~\eqref{eqn2:lem:dec}. For the first term, we have
\begin{equation}
\begin{aligned}
 - \mb{E}\left[ \left\langle \frac{1}{m}\sum\limits_{i=1}^{m} \alpha_i \sum\limits_{\nu=0}^{\tau_i-1}  \m{h}_i(\m{x}_{i,\nu}, \m{y}^+) , \nabla f(\m{x}) \right\rangle \right] =& - \mb{E}\left[ \frac{1}{m}\sum\limits_{i=1}^{m} \alpha_i \sum\limits_{\nu=0}^{\tau_i-1} \mb{E}\left[ \left\langle  \m{h}_i(\m{x}_{i,\nu}, \m{y}^+) , \nabla f(\m{x}) \right\rangle \mid {\cal F}_{i,\nu-1}  \right] \right] \\
 =& - \mb{E}\left[ \left\langle \frac{1}{m}\sum\limits_{i=1}^{m} \alpha_i \sum\limits_{\nu=0}^{\tau_i-1} \bar{\m{h}}_i(\m{x}_{i,\nu}, \m{y}^+) , \nabla f(\m{x}) \right\rangle  \right] \\
 =& - \frac{\alpha}{2} \mb{E}\left[\left\|  \frac{1}{m}\sum\limits_{i=1}^{m} \frac{1}{\tau_i} \sum\limits_{\nu=0}^{\tau_i-1} \bar{\m{h}}_i(\m{x}_{i,\nu}, \m{y}^+)  \right\|^2\right]- \frac{\alpha}{2} \mb{E}\left[  \left\| \nabla f(\m{x})\right\|^2\right]\\
&+ \frac{\alpha}{2} \mb{E}\left[ \left\| \frac{1}{m}\sum\limits_{i=1}^{m} \frac{1}{\tau_i} \sum\limits_{\nu=0}^{\tau_i-1} \bar{\m{h}}_i(\m{x}_{i,\nu}, \m{y}^+) -\nabla f(\m{x})\right\|^2\right],
\end{aligned}
\label{eqn5:lem:dec}
\end{equation} 
where the first equality follows from the law of total expectation; the second equality uses the fact that $\bar{\m{h}}_i (\m{x}_{i,\nu}, \m{y}^+)=\mb{E}\left[\m{h}_i(\m{x}_{i,\nu}, \m{y}^+)|{\cal F}_{i,\nu-1}\right]$; and the last equality is obtained from our assumption $\alpha_i=\alpha/\tau_i, \forall i \in \mathcal{S}$.

Next, we bound the last term in~\eqref{eqn5:lem:dec}. Note that
\begin{equation*}
 \begin{aligned}
\quad \left\|\frac{1}{m}\sum\limits_{i=1}^{m} \frac{1}{\tau_i}  \sum\limits_{\nu=0}^{\tau_i-1} \bar{\m{h}}_i(\m{x}_{i,\nu},\m{y}^+) -\nabla f(\m{x}) \right\|^2&=\Big\|\frac{1}{m}\sum\limits_{i=1}^{m} \frac{1}{\tau_i} \sum\limits_{\nu=0}^{\tau_i-1} ( \bar{\m{h}}_i(\m{x}_{i,\nu}, \m{y}^+)-  \bar{\nabla} f_i(\m{x}, \m{y}^+)) \\
&+ \frac{1}{m} \sum\limits_{i=1}^{m} \frac{1}{\tau_i}  \sum\limits_{\nu=0}^{\tau_i-1} \bar{\nabla} f_i(\m{x}, \m{y}^+) -\nabla f(\m{x}) \Big\|^2\\
  &\leq 3 \left\|\frac{1}{m}\sum\limits_{i=1}^{m} \frac{1}{\tau_i}  \sum\limits_{\nu=0}^{\tau_i-1} (\bar{\m{h}}_i(\m{x}_{i,\nu},\m{y}^+)-\bar{\nabla}f _i(\m{x}_{i,\nu},\m{y}^+))\right\|^2  \\
&+3\left\|\frac{1}{m}\sum\limits_{i=1}^{m} \frac{1}{\tau_i}  \sum\limits_{\nu=0}^{\tau_i-1}(\bar{\nabla}f _i(\m{x}_{i,\nu},\m{y}^+) - \bar{\nabla}f _i(\m{x},\m{y}^+))\right\|^2\\
     &+3\left\| \bar{\nabla} f(\m{x}, \m{y}^+) - \nabla f(\m{x}) \right\|^2,
    \end{aligned}
\end{equation*}   
where the inequality uses Lemma~\ref{lem:Jens}.

Hence,
\begin{equation}
 \begin{aligned}
& \quad  \mb{E}\left[\|\frac{1}{m}\sum\limits_{i=1}^{m} \frac{1}{\tau_i}  \sum\limits_{\nu=0}^{\tau_i-1} \bar{\m{h}}_i(\m{x}_{i,\nu},\m{y}^+) -\nabla f(\m{x}) \|^2 \right]\\
& \leq 3 b^2 + \frac{3 M_f^2 }{m} \sum\limits_{i=1}^{m}   \frac{1}{\tau_i} \sum\limits_{\nu=0}^{\tau_i-1}  \mb{E}\left[\|\m{x}_{i,\nu} -\m{x}\|^2\right] +  3 M_f^2 \mb{E}\left[\|\m{y}^+-\m{y}^*(\m{x})\|^2\right], 
    \end{aligned}
\label{eqn6+:lem:dec}
\end{equation}   
where the inequality uses Lemmas  \ref{lem:neum:bias} and \ref{lem:lips}. 

Substituting \eqref{eqn6+:lem:dec} into \eqref{eqn5:lem:dec} yields 
\begin{equation}
\begin{aligned}
&- \mb{E}\left[ \left\langle \frac{1}{m}\sum\limits_{i=1}^{m} \alpha_i \sum\limits_{\nu=0}^{\tau_i-1}  \m{h}_i(\m{x}_{i,\nu}, \m{y}^+) , \nabla f(\m{x}) \right\rangle \right] \\
&\leq - \frac{\alpha}{2} \mb{E}\left[\|  \frac{1}{m}\sum\limits_{i=1}^{m} \frac{1}{\tau_i} \sum\limits_{\nu=0}^{\tau_i-1} \bar{\m{h}}_i(\m{x}_{i,\nu}, \m{y}^+)  \|^2\right]- \frac{\alpha}{2} \mb{E}\left[\| \nabla f(\m{x})\|^2\right]\\
 & +\frac{3\alpha }{2}  \Big( b^2 +  M_f^2 \mb{E}\left[\|\m{y}^+-\m{y}^*(\m{x})\|^2\right] + \frac{ M_f^2 }{m} \sum\limits_{i=1}^{m}   \frac{1}{\tau_i} \sum\limits_{\nu=0}^{\tau_i-1}  \mb{E}\left[\|\m{x}_{i,\nu} -\m{x}\|^2\right]\Big).
\end{aligned}
\label{eqn6-:lem:dec}
\end{equation}   

Next, we bound the second term on the RHS of~\eqref{eqn2:lem:dec}. Observe that 
\begin{equation}
 \begin{aligned}
&  \mb{E}\left[ \left\| \frac{1}{m}\sum\limits_{i=1}^{m} \alpha_i \sum\limits_{\nu=0}^{\tau_i-1} \m{h}_i(\m{x}_{i,\nu}, \m{y}^+)\right\|^2 \right]\\
&= \alpha^2 \mb{E}\left[  \left\|\frac{1}{m}\sum\limits_{i=1}^{m} \frac{1}{\tau_i} \sum\limits_{\nu=0}^{\tau_i-1} \left( \m{h}_i(\m{x}_{i,\nu}, \m{y}^+)-  \bar{\m{h}}_i(\m{x}_{i,\nu}, \m{y}^+) + \bar{\m{h}}_i(\m{x}_{i,\nu}, \m{y}^+) \right) \right\| \right]\\
 &\leq \alpha^2 \mb{E}\left[\norm[\big]{\frac{1}{m}\sum\limits_{i=1}^{m} \frac{1}{\tau_i}\sum\limits_{\nu=0}^{\tau_i-1}  \bar{\m{h}}_i(\m{x}_{i,\nu}, \m{y}^+)}^2\right]+  \alpha^2 \tilde{\sigma}_f^2,
    \end{aligned}
\label{eqn6:lem:dec}
\end{equation}
where the inequality follows from Lemmas~\ref{lem:rand:zer} and \ref{lem:lips}. 

Plugging \eqref{eqn6:lem:dec} and \eqref{eqn6-:lem:dec} into \eqref{eqn2:lem:dec} completes the proof.
\end{proof}

\subsection{Error of \fedinn}

The following lemma establishes the progress of \fedinn. It should be mentioned that the assumption on $\beta_i, \forall i \in \mathcal{S}$ is identical to the one listed in~\citep[Theorem 4]{mitra2021linear}.
\begin{lemma}[Error of \fedinn]\label{thm:fedin} 
Suppose Assumptions~\ref{assu:f} and \ref{assu:bound:var} hold.  Further, assume 
\begin{equation*}
\tau_i \geq 1,~~~~\alpha_i=\frac{\alpha}{\tau_i},~~~~\beta_i=\frac{\beta}{\tau_i},~~~~\forall i \in \mathcal{S},    
\end{equation*}
where  $0<\beta <  \min\big(1/(6\ell_{g,1}),1\big)$ and $\alpha$ is some positive constant. Then, \fedinn guarantees: 
\begin{subequations}\label{eqn:err:fedin1}
\begin{align}
  \mb{E}\left[\left\|\m{y}^{+}\!\!-\m{y}^\star(\m{x})\right\|^2\right]& \leq  \left(1-  \frac{{\beta} \mu_g}{2} \right)^T \mb{E}\left[\left\| \m{y}- \m{y}^\star(\m{x}) \right\|^2\right] +25 T {\beta}^2 \sigma^2_{g,1},~~~\textnormal{and} \label{eqn:err:fedina1}\\
 \nonumber  
  \mb{E}\left[\left\|\m{y}^{+}-\m{y}^\star(\m{x}^{+})\right\|^2\right] &\leq a_1(\alpha)  \mb{E}\left[ \left\| \frac{1}{m}\sum\limits_{i=1}^{m} \frac{1}{\tau_i} \sum\limits_{\nu=0}^{\tau_i-1} \bar{\m{h}}_i(\m{x}_{i,\nu},\m{y}^+)\right\|^2 \right]\\
 &+ a_2(\alpha) \mb{E}\left[\left\|\m{y}^+\!\!-\m{y}^\star(\m{x})\right\|^2\right] +a_3(\alpha) \tilde{\sigma}_f^2.
 \label{eqn:err:fedinb1}
\end{align}
\end{subequations}
Here,
\begin{equation}\label{eqn:a1-3}
\begin{aligned}
a_1(\alpha) &:=  L_{\m{y}}^2 \alpha^2+ \frac{L_{\m{y}}\alpha}{4M_f} +\frac{L_{\m{yx}} \alpha^2}{2\eta},\\
a_2(\alpha)  &:=  1+ 4M_f L_{\m{y}} \alpha + \frac{\eta L_{\m{yx}}\tilde{D}_f^2 \alpha^2}{2}, \\
a_3(\alpha) &:=  \alpha^2  L_{\m{y}}^2  +\frac{L_{\m{yx}} \alpha^2}{2\eta},
\end{aligned}
\end{equation}
for any $\eta>0$.
\end{lemma}
\begin{proof}
Note that 
\begin{align}\label{eqn-1:dec:fedin}
\nonumber 
 \mb{E}\left[\|\m{y}^{+}-\m{y}^*(\m{x}^{+})\|^2\right] =& \mb{E}\left[\|\m{y}^{+}-\m{y}^*(\m{x})\|^2\right] + \mb{E}\left[\|\m{y}^*(\m{x}^{+})-\m{y}^*(\m{x})\|^2\right]\\
+&2\mb{E}\left[\langle \m{y}^{+}-\m{y}^*(\m{x}), \m{y}^*(\m{x})-\m{y}^*(\m{x}^{+})\rangle\right].
\end{align}
Next, we upper bound each term on the RHS of~\eqref{eqn-1:dec:fedin}.
\\
\textbf{Bounding the first term in \eqref{eqn-1:dec:fedin}}:\\
From \citep[Theorem 4]{mitra2021linear}, for all $t \in \{0,\ldots,T-1\}$, we obtain
\begin{equation*}
    \mb{E}[{\Vert\m{y}^{t+1}- \m{y}^\star(\m{x})\Vert}^2] \leq \left(1-  \frac{{\beta} \mu_g}{2} \right) \mb{E}[{\Vert \m{y}^{t}- \m{y}^\star(\m{x})\Vert}^2] +25 {\beta}^2 \sigma^2_{g,1},
\end{equation*}
which together with our setting $\m{y}^+= \m{y}^{T}$ implies
\begin{equation}\label{eqn0:dec:fedin}
\mb{E}[\|\m{y}^{+}-\m{y}^*(\m{x})\|^2] \leq \left(1-  \frac{{\beta} \mu_g}{2} \right)^T \mb{E}[{\Vert \m{y}- \m{y}(\m{x}^*) \Vert}^2] +25 T {\beta}^2 \sigma^2_{g,1}.
\end{equation}
\textbf{Bounding the second term in \eqref{eqn-1:dec:fedin}}:\\
By similar steps as in \eqref{eqn6:lem:dec}, we have
\begin{equation}
\begin{aligned}\label{eqn1:dec:fedin}
 \mb{E}\left[\|\m{y}^*(\m{x}^{+})-\m{y}^*(\m{x})\|^2\right] &\leq  L_{\m{y}}^2 \mb{E}\left[\left\| \frac{1}{m}\sum\limits_{i=1}^{m} \alpha_i \sum\limits_{\nu=0}^{\tau_i-1}  \m{h}_i(\m{x}, \m{y}^+)\right\|^2\right] \\
&\leq   L_{\m{y}}^2 \mb{E}\left[\norm[\bigg]{\frac{1}{m}\sum\limits_{i=1}^{m}\alpha_i \sum\limits_{\nu=0}^{\tau_i-1}  \bar{\m{h}}_i(\m{x}_{i,\nu}, \m{y}^+)}^2\right]+  \alpha^2  L_{\m{y}}^2 \tilde{\sigma}_f^2,
\end{aligned}
\end{equation}
where the inequalities are obtained from Lemmas~\ref{lem:lips} and \ref{lem:rand:zer}.

\textbf{Bounding the third term in \eqref{eqn-1:dec:fedin}}:\\
Observe that 
\begin{equation}
\label{eqn2:dec:fedin}
    \begin{aligned}
\mb{E}\left[\langle \m{y}^{+}-\m{y}^*(\m{x}), \m{y}^*(\m{x})-\m{y}^*(\m{x}^{+})\rangle \right] &= -\mb{E}\left[\langle\m{y}^+\!\!-\m{y}^\star(\m{x}), \nabla \m{y}^\star(\m{x})(\m{x}^+\!\!-\m{x})\rangle \right]
\\
&\quad -\mb{E}\left[\langle\m{y}^+\!\!-\m{y}^\star(\m{x}), \m{y}^\star(\m{x}^+)-\m{y}^\star(\m{x})-\nabla \m{y}^\star(\m{x})(\m{x}^+\!\!-\m{x})\rangle \right].
\end{aligned}
\end{equation}
For the first term on the R.H.S. of the above equality, we have
\begin{equation}
\label{eqn3:dec:fedin}
    \begin{aligned}
 -\mb{E}[\langle\m{y}^+\!\!-\m{y}^\star(\m{x}), \nabla \m{y}^\star(\m{x})(\m{x}^+\!\!-\m{x})\rangle]  
=&-\mb{E}\left[\langle{\m{y}^+\!\!-\m{y}^\star(\m{x}), \frac{1}{m} \nabla \m{y}^\star(\m{x}) \sum\limits_{i=1}^{m} \alpha_i \sum\limits_{\nu=0}^{\tau_i-1} \bar{\m{h}}_i(\m{x}_{i,\nu},\m{y}^+) \rangle}\right]\\
 \leq & \mb{E}\left[\left\|\m{y}^+\!\!-\m{y}^\star(\m{x})\right\|  \left\| \frac{1}{m} \nabla \m{y}^\star(\m{x})\sum\limits_{i=1}^{m} \alpha_i \sum\limits_{\nu=0}^{\tau_i-1} \bar{\m{h}}_i(\m{x}_{i,\nu},\m{y}^+)\right\|\right]\\
\leq & L_{\m{y}}\mb{E}\left[\left\|\m{y}^+\!\!-\m{y}^\star(\m{x})\right\| \left\| \frac{1}{m}\sum\limits_{i=1}^{m} \alpha_i \sum\limits_{\nu=0}^{\tau_i-1} \bar{\m{h}}_i(\m{x}_{i,\nu},\m{y}^+)\right\|\right]\\
\leq &  2 \gamma\mb{E}\left[\left\|\m{y}^+\!\!-\m{y}^\star(\m{x})\right\|^2\right] + \frac{L_{\m{y}}^2 \alpha^2}{8\gamma } \mb{E}\left[\left\| \frac{1}{m}\sum\limits_{i=1}^{m} \frac{1}{\tau_i} \sum\limits_{\nu=0}^{\tau_i-1} \bar{\m{h}}_i(\m{x}_{i,\nu},\m{y}^+)\right\|^2\right],
\end{aligned}
\end{equation}
where the first equality uses the fact that $\bar{\m{h}}_i (\m{x}_{i,\nu}, \m{y}^+)=\mb{E}\left[\m{h}_i(\m{x}_{i,\nu}, \m{y}^+)|{\cal F}_{i,\nu-1}\right]$;  the second inequality follows from Lemma~\ref{lem:lips}; 
and the last inequality is obtained from the Young's inequality such that $ab\leq 2\gamma a^2+\frac{b^2}{8\gamma}$.

Further, using Lemma \ref{lem:lips}, we have 
\begin{equation}
\label{eqn4-:dec:fedin}
\begin{aligned}
-&\mb{E}[\langle\m{y}^+\!\!-\m{y}^\star(\m{x}),  \m{y}^\star(\m{x}^+)-\m{y}^\star(\m{x})-\nabla \m{y}^\star(\m{x})(\m{x}^+\!\!-\m{x})\rangle] \\
& \leq  \mb{E}\left[\left\|\m{y}^+\!\!-\m{y}^\star(\m{x})\| \|\m{y}^\star(\m{x}^+)-\m{y}^\star(\m{x})-\nabla \m{y}^\star(\m{x})(\m{x}^+\!\!-\m{x})\right\|\right] \\
&\leq \frac{L_{\m{yx}}}{2}\mb{E}\left[\left\|\m{y}^+\!\!-\m{y}^\star(\m{x})\right\|\left\|\m{x}^+\!\!-\m{x}\right\|^2\right],
\end{aligned}
\end{equation}
where the inequality follows from  Lemma~\ref{lem:lips}.

Note that $\mc{F}_{0}=\mc{F}_{i,0}$ for all $i \in \mc{S}$. This together with \eqref{eqn1:lem:dec} implies that
\begin{equation}
\begin{aligned}\label{eqn:cond}
&~ \mb{E}\left[ \left\|\m{y}^+\!\!-\m{y}^\star(\m{x})\right\|^2 \left\|\m{x}^+\!\!-\m{x}\right\|^2 \right]\\
 &\leq  \frac{1}{m}  \sum\limits_{i=1}^{m} \frac{\alpha^2}{\tau_i} \sum\limits_{\nu=0}^{\tau_i-1}  \mb{E}\left[\left\|\m{y}^+\!\!-\m{y}^\star(\m{x})\right\|^2   \mb{E}\left[ \left\| \m{h}_i(\m{x}_{i,\nu}, \m{y}^+) \right\|^2 \mid\mc{F}_{i,\tau_i-1} \right] \right]\\
 & \leq \alpha^2 \tilde{D}_f^2 \mb{E}\left[\left\|\m{y}^+\!\!-\m{y}^\star(\m{x})\right\|^2 \right], 
\end{aligned}
\end{equation}
where the last inequality uses Lemma~\ref{lem:lips}.

Note also that for any $\eta>0$, we have $1 \leq \frac{\eta}{2}+\frac{1}{2\eta}$. Combining this inequality with \eqref{eqn4-:dec:fedin} and using \eqref{eqn:cond} give
\begin{equation}
\label{eqn4:dec:fedin}
\begin{aligned}
-&\mb{E}[\langle\m{y}^+\!\!-\m{y}^\star(\m{x}),  \m{y}^\star(\m{x}^+)-\m{y}^\star(\m{x})-\nabla \m{y}^\star(\m{x})(\m{x}^+\!\!-\m{x})\rangle]\\
& \leq \frac{L_{\m{yx}}}{2}\mb{E}\left[\left\|\m{y}^+\!\!-\m{y}^\star(\m{x})\right\|\left\|\m{x}^+\!\!-\m{x}\right\|^2\right] \\
& \leq \frac{\eta L_{\m{yx}}}{4}\mb{E}\left[\left\|\m{y}^+\!\!-\m{y}^\star(\m{x})\right\|\left\|\m{x}^+\!\!-\m{x}\right\|^2\right] + \frac{ L_{\m{yx}}}{4\eta}\mb{E}\left[\left\|\m{x}^+\!\!-\m{x}\right\|^2\right] \\
&\leq \frac{\eta L_{\m{yx}} \tilde{D}_f^2\alpha^2 }{4}\mb{E}\left[\left\|\m{y}^+\!\!-\m{y}^\star(\m{x})\right\|^2 \right] + \frac{L_{\m{yx}} \alpha^2}{4\eta}  \mb{E}\left[\left\|\frac{1}{m}\sum\limits_{i=1}^{m} \frac{1}{\tau_i} \sum\limits_{\nu=0}^{\tau_i-1} \bar{\m{h}}_i(\m{x}_{i,\nu},\m{y}^+)\right\|^2\right] +  \frac{L_{\m{yx}} \alpha^2}{4\eta}  \tilde{\sigma}_f^2,
\end{aligned}
\end{equation}
where  the last inequality uses \eqref{eqn:cond} and Lemma~\ref{lem:neum:bias}.

Let  $\gamma=M_f L_{\m{y}} \alpha$. Plugging \eqref{eqn4:dec:fedin} and \eqref{eqn3:dec:fedin} into \eqref{eqn2:dec:fedin}, we have
\begin{equation}
\label{eqn5:dec:fedin}
    \begin{aligned}
  \mb{E}[\langle \m{y}^{+}-\m{y}^*(\m{x}), \m{y}^*(\m{x})-\m{y}^*(\m{x}^{+})\rangle] 
    &\leq\left( 2\gamma + \frac{\eta L_{\m{yx}}\tilde{D}_f^2}{4}\alpha^2\right)\mb{E}\left[\left\|\m{y}^+\!\!-\m{y}^\star(\m{x})\right\|^2\right] \\
    &+ \left(\frac{L_{\m{y}}^2 \alpha^2}{8\gamma} +\frac{L_{\m{yx}} \alpha^2}{4\eta}\right)\mb{E}\left[\left\|\frac{1}{m}\sum\limits_{i=1}^{m} \frac{1}{\tau_i} \sum\limits_{\nu=0}^{\tau_i-1} \bar{\m{h}}_i(\m{x}_{i,\nu},\m{y}^+)\right\|^2\right] +  \frac{L_{\m{yx}} \alpha^2}{4\eta}  \tilde{\sigma}_f^2\\
    &= \left( 2 M_f L_{\m{y}} \alpha + \frac{\eta L_{\m{yx}}\tilde{D}_f^2}{4}\alpha^2\right)\mb{E}\left[\left\|\m{y}^+\!\!-\m{y}^\star(\m{x})\right\|^2\right]\\
     &+ \left(\frac{L_{\m{y}} \alpha}{8 M_f} +\frac{L_{\m{yx}} \alpha^2}{4\eta}\right)\mb{E}\left[\left\|\frac{1}{m}\sum\limits_{i=1}^{m} \frac{1}{\tau_i} \sum\limits_{\nu=0}^{\tau_i-1} \bar{\m{h}}_i(\m{x}_{i,\nu},\m{y}^+)\right\|^2\right]+  \frac{L_{\m{yx}} \alpha^2}{4\eta}  \tilde{\sigma}_f^2.
\end{aligned}
\end{equation}
Substituting \eqref{eqn5:dec:fedin}, \eqref{eqn1:dec:fedin}, and \eqref{eqn0:dec:fedin} into \eqref{eqn-1:dec:fedin} completes the proof.
\end{proof}

\subsection{Drifting Errors of \fedout }

The following lemma provides a bound on the \textit{drift} of each $\m{x}_{i,\nu}$ from $\m{x}$ for stochastic nonconvex bilevel problems. 

\begin{lemma}[Drifting Error of \fedout]\label{thm4:lemma2}
Suppose Assumptions~\ref{assu:f} and \ref{assu:bound:var} hold. Further, assume $\tau_i \geq 1$ and $ \alpha_i\leq 1/(5M_f\tau_i), \forall i \in \mathcal{S}$. Then, for each $i \in \mathcal{S}$ and $\forall \nu \in\{0,\ldots, \tau_i-1\}$,  \fedout guarantees:
\begin{align}\label{eqn1:lemm:drift}
\nonumber
 \mb{E}\left[\left\|\m{x}_{i,\nu}-\m{x}\right\|^2\right] 
 &\leq 36 \tau_i^2 \alpha_i^2 \left(M_f^2 \mb{E}\left[\|\m{y}^+-\m{y}^\star(\m{x})\|^2\right]+  \mb{E}\left[\|\nabla f(\m{x})\|^2\right]+ 3 b^2\right) \\
 &+ 27 \tau_i \alpha_i^2  \tilde{\sigma}_f^2.
\end{align}
\end{lemma}
\begin{proof}
The result trivially holds for $\tau_i=1$. Let $\tau_i>1$ and define
\begin{equation}\label{eqn1+:lemm:drift}
\begin{aligned}
\m{v}_{i,\nu}&:= \bar{\m{h}}_i(\m{x}_{i,\nu}, \m{y}^+)- \bar{\nabla} f_i(\m{x}_{i,\nu}, \m{y}^+)- \bar{\m{h}}_i(\m{x},\m{y}^+) \\
& +\bar{\nabla} f_i(\m{x}, \m{y}^+) +\bar{\m{h}}(\m{x},\m{y}^+) -\bar{\nabla} f(\m{x}, \m{y}^+),\\
\m{w}_{i,\nu}&:= \m{h}_i(\m{x}_{i,\nu}, \m{y}^+)-\bar{\m{h}}_i(\m{x}_{i,\nu},\m{y}^+)+\bar{\m{h}}_i(\m{x},\m{y}^+)\\
&-\m{h}_i(\m{x},\m{y}^+)+\m{h}(\m{x},\m{y}^+)-\bar{\m{h}} (\m{x},\m{y}^+),\\
\m{z}_{i,\nu}&:=\bar{\nabla} f_i(\m{x}_{i,\nu}, \m{y}^+) - \bar{\nabla} f_i(\m{x}, \m{y}^+)+  \bar{\nabla} f(\m{x}, \m{y}^+) - \nabla f(\m{x}) + \nabla  f(\m{x}).
\end{aligned}
\end{equation}
One will notice that 
$$
\m{v}_{i,\nu}+\m{w}_{i,\nu}+\m{z}_{i,\nu} = \m{h}_i(\m{x}_{i,\nu}, \m{y}^+)-\m{h}_i(\m{x},\m{y}^+)+\m{h}(\m{x},\m{y}^+).
$$
Hence, from Algorithm~\ref{alg:fedout}, for each $i\in\mathcal{S}$, and $\forall  \nu \in\{0,\ldots,\tau_i-1\}$, we have   
\begin{equation}
\begin{aligned}
        \m{x}_{i,\nu+1}-\m{x}&=\m{x}_{i,\nu}-\m{x}-\alpha_i (\m{v}_{i,\nu}+\m{w}_{i,\nu}+\m{z}_{i,\nu}),
    \end{aligned}
\end{equation}
which implies that 
\begin{equation}\label{eqn2:lemm:drift}
\begin{aligned}
\mb{E}\left[\|\m{x}_{i,\nu+1}-\m{x}\|^2\right]
&=\mb{E}\left[\|\m{x}_{i,\nu}-\m{x}-\alpha_i (\m{v}_{i,\nu}+\m{z}_{i,\nu})\|^2\right]+\alpha_i^2\mb{E}\left[\|{\m{w}_{i,\nu}}\|^2\right]\\&-2\mb{E}\left[\mb{E}\left[\langle \m{x}_{i,\nu}-\m{x}-\alpha_i(\m{v}_{i,\nu}+\m{z}_{i,\nu}),\alpha_i \m{w}_{i,\nu} \rangle\mid\mathcal{F}_{i,\nu-1}\right]\right] \\ 
&= \mb{E}\left[\|\m{x}_{i,\nu}-\m{x}-\alpha_i(\m{v}_{i,\nu}+\m{z}_{i,\nu})\|^2\right]+\alpha_i^2\mb{E}\left[\|\m{w}_{i,\nu}\|^2\right].
\end{aligned}
\end{equation}
Here, the last equality uses Lemma~\ref{lem:rand:zer} since  $\mb{E}[\m{w}_{i,\nu}|\mathcal{F}_{i,\nu-1}]=0$, {by definition}.

From Lemmas~\ref{lem:Jens}, \ref{lem:neum:bias}, and \ref{lem:lips},  for $\m{v}_{i,\nu}$, $\m{w}_{i,\nu}$, and $\m{z}_{i,\nu}$ defined in \eqref{eqn1+:lemm:drift}, we have
\begin{subequations}\label{eqn3:v-z:drift}
\begin{equation}
\label{eqn3a:lemm:drift}
\begin{aligned}
\mb{E}\Big[\|\m{v}_{i,\nu}\|^2\Big]  &\leq 3  \mb{E}\Big[ \left\|\bar{\nabla} f_i(\m{x}_{i,\nu}, \m{y}^+)-\bar{\m{h}}_i(\m{x}_{i,\nu},\m{y}^+) \right\|^2 \\
&+\left\|\bar{\m{h}}_i(\m{x},\m{y}^+)-\bar{\nabla} f_i(\m{x},\m{y}^+)\right\|^2+\left\|\bar{\nabla} f(\m{x},\m{y}^+)-\bar{\m{h}} (\m{x},\m{y}^+)\right\|^2 \Big]\\
&\leq  9 b^2,
\end{aligned}
\end{equation}
%
\begin{equation}
\label{eqn3b:lemm:drift}
\begin{aligned}
\mb{E}\left[\|\m{w}_{i,\nu}\|^2\right] &\leq 3   \mb{E}\Big[ \|\m{h}_i(\m{x}_{i,\nu},\m{y}^+)-\bar{\m{h}}_i(\m{x}_{i,\nu},\m{y}^+)\|^2\\
&+\|\bar{\m{h}}_i(\m{x},\m{y}^+)-\m{h}_i(\m{x},\m{y}^+)\|^2+\|\m{h}(\m{x},\m{y}^+)-\bar{\m{h}} (\m{x},\m{y}^+)\|^2\Big]\\
        & \leq 9\tilde{\sigma}_f^2,
    \end{aligned}
\end{equation}
and
\begin{equation}
\label{eqn3c:lemm:drift}
\begin{aligned}
\mb{E}\left[\|\m{z}_{i,\nu}\|^2\right] &\leq 3   \mb{E}\Big[ \|
\bar{\nabla} f_i(\m{x}_{i,\nu}, \m{y}^+) - \bar{\nabla} f_i(\m{x}, \m{y}^+)\|^2\\
&+\|  \bar{\nabla} f(\m{x}, \m{y}^+) - \nabla f(\m{x})\|^2+ \|\nabla  f(\m{x})\|^2\Big]\\
& \leq 3\left(M_f^2 \mb{E}\left[\|\m{x}_{i,\nu}-\m{x}\|^2\right]+ M_f^2 \mb{E}\left[\|\m{y}^+-\m{y}^\star(\m{x})\|^2\right]+  \mb{E}\left[\|\nabla f(\m{x})\|^2\right]\right).
\end{aligned}
\end{equation}
\end{subequations}

Now, the first term in the RHS of \eqref{eqn2:lemm:drift} can be bounded as follows:
\begin{equation}
\label{eqn4:lemm:drift}
\begin{aligned}
\mb{E}\left[\|\m{x}_{i,\nu}-\m{x}-\alpha_i(\m{v}_{i,\nu}+\m{z}_{i,\nu})\|^2\right] &\leq  \left(1+\frac{1}{2\tau_i-1}\right)\mb{E}\left[\|\m{x}_{i,\nu}-\m{x}\|^2\right]+ 2\tau_i \mb{E}\left[\|\alpha_i (\m{v}_{i,\nu}+\m{z}_{i,\nu})\|^2\right]\\
&\leq  \left(1+\frac{1}{2\tau_i-1}\right)\mb{E}\left[\|\m{x}_{i,\nu}-\m{x}\|^2\right]+ 4\tau_i \alpha_i^2  \left(\mb{E}\left[\| \m{z}_{i,\nu}\|^2\right] + \mb{E}\left[\| \m{v}_{i,\nu}\|^2\right] \right)\\
 &\leq  \left(1+\frac{1}{2\tau_i-1}\right)\mb{E}\left[\|\m{x}_{i,\nu}-\m{x}\|^2\right]+ 4\tau_i \alpha_i^2  \left(\mb{E}\left[\| \m{z}_{i,\nu}\|^2\right] + 9  b^2\right)\\
       &\leq  \left(1+\frac{1}{2\tau_i-1}+12\tau_i \alpha_i^2 M_f^2\right)\mb{E}\left[\|\m{x}_{i,\nu}-\m{x}\|^2\right]\\
       &+ 12\tau_i \alpha_i^2 \left( M_f^2 \mb{E}\left[\|\m{y}^+-\m{y}^\star(\m{x})\|^2\right]+  \mb{E}\left[\|\nabla f(\m{x})\|^2\right]+ 3 b^2\right).
    \end{aligned}
\end{equation}
Here, the first inequality follows from Lemma~\ref{lem:trig}; the second inequality uses Lemma~\ref{lem:Jens}; and the third and last inequalities follow from \eqref{eqn3a:lemm:drift} and \eqref{eqn3c:lemm:drift}.


Substituting \eqref{eqn4:lemm:drift} into \eqref{eqn2:lemm:drift} gives 
\begin{equation}
\label{eqn5:lemm:drift}
\begin{aligned}
        \mb{E}\left[\left\|\m{x}_{i,\nu+1}-\m{x}\right\|^2\right]&\leq \left(1+\frac{1}{2\tau_i-1}+12\tau_i \alpha_i^2 M_f^2\right)\mb{E}\left[\|\m{x}_{i,\nu}-\m{x}\|^2\right]\\
       & + 12\tau_i \alpha_i^2 \left( M_f^2 \mb{E}\left[\|\m{y}^+-\m{y}^\star(\m{x})\|^2\right]+  \mb{E}\left[\|\nabla f(\m{x})\|^2\right]+ 3  b^2\right) + 9 \alpha_i^2 \tilde{\sigma}_f^2 \\
        &\leq \left(1+\frac{1}{\tau_i-1}\right)\mb{E}\left[\|\m{x}_{i,\nu}-\m{x}\|^2\right] \\
        & + 12\tau_i \alpha_i^2 \left( M_f^2 \mb{E}\left[\|\m{y}^+-\m{y}^\star(\m{x})\|^2\right]+  \mb{E}\left[\|\nabla f(\m{x})\|^2\right]+ 3  b^2\right) + 9 \alpha_i^2 \tilde{\sigma}_f^2. \\
    \end{aligned} 
\end{equation}
Here, the first inequality uses \eqref{eqn3b:lemm:drift} and the last inequality follows by noting $\alpha_i\leq 1/(5M_f\tau_i)$. 

For all $\tau_i>1$, we have
\begin{equation}\label{eqn:sum:geom}
\begin{split}
\sum\limits_{j=0}^{\nu -1} \left(1+\frac{1}{\tau_i-1}\right)^j &= \frac{ \left(1+\frac{1}{\tau_i-1}\right)^{\nu }-1}{\left(1+\frac{1}{\tau_i-1}\right)-1} \\
&\leq  \tau_i {\left(1+\frac{1}{\tau_i} \right)}^{\nu } \leq \tau_i {\left(1+\frac{1}{\tau_i} \right)}^{\tau_i} \leq \exp{(1)}\tau_i < 3 \tau_i. 
\end{split}
\end{equation}
Now, iterating equation \eqref{eqn5:lemm:drift} and using $\m{x}_{i,0}=\m{x}, \forall i\in\mathcal{S}$, we obtain
\begin{equation*}
\begin{aligned}
        \mb{E}\left[\|\m{x}_{i,\nu}-\m{x}\|^2\right]&\leq  \left( 12\tau_i \alpha_i^2 \left( M_f^2 \mb{E}\left[\|\m{y}^+-\m{y}^\star(\m{x})\|^2\right]+  \mb{E}\left[\|\nabla f(\m{x})\|^2\right]+ 3  b^2\right) + 9 \alpha_i^2 \tilde{\sigma}_f^2 \right) \sum\limits_{j=0}^{\nu -1} \left(1+\frac{1}{\tau_i-1}\right)^j\\
        &\leq 3 \tau_i \left( 12\tau_i \alpha_i^2 \left( M_f^2 \mb{E}\left[\|\m{y}^+-\m{y}^\star(\m{x})\|^2\right]+  \mb{E}\left[\|\nabla f(\m{x})\|^2\right]+ 3  b^2\right) + 9 \alpha_i^2 \tilde{\sigma}_f^2 \right),
    \end{aligned}
\end{equation*}
where the second inequality uses \eqref{eqn:sum:geom}. This completes the proof.
\end{proof}

\begin{remark}
Lemma~\ref{thm4:lemma2} shows that the bound on the client-drift scales linearly with $\tau_i$ and the inner error $\|\m{y}^+-\m{y}^\star(\m{x})\|^2$ in general nested FL. We aim to control such a drift by selecting $\alpha_i=\mc{O}(1/\tau_i)$ for all $ i \in \mc{S}$ and using the inner error bound provided in Lemma~\ref{thm:fedin}. 
\end{remark}

Next, we provide the proof of our main result which can be adapted to general nested problems (bilevel, min-max, compositional). 
\subsection{Proof of Theorem ~\ref{thm:fednest}}
\begin{proof}
We define the following Lyapunov function 
\begin{equation}\label{eqn:lyapfunc}
 \mb{W}^k:= f(\m{x}^k)  + \frac{M_f}{L_{\m{y}}}\|\m{y}^k - \m{y}^\star(\m{x}^k)\|^2.   
\end{equation}
Motivated by \citep{chen2021closing}, we bound the difference between two Lyapunov functions. That is,
\begin{align}\label{eqn:lyapfunc:diff}
  \mb{W}^{k+1}  -  \mb{W}^k   =  &  f(\m{x}^{k+1}) -  f(\m{x}^k)+\frac{M_f}{L_{\m{y}}}\left(\|\m{y}^{k+1} - \m{y}^*(\m{x}^{k+1})\|^2 - \|\m{y}^k - \m{y}^\star(\m{x}^k)\|^2\right).
\end{align}
The first two terms on the RHS of \eqref{eqn:lyapfunc:diff} quantifies the descent of outer objective $f$ and the reminding terms measure the descent of the inner errors.  

From our assumption, we have  $\alpha_i^k=\alpha_k/\tau_i$, $\beta_i^k=\beta_k/\tau_i$,  $\forall i \in \mathcal{S}$. Substituting these stepsizes into the bounds provided in  Lemmas \ref{lem:dec} and \ref{thm:fedin}, and using \eqref{eqn:lyapfunc:diff}, we get 
\begin{subequations}\label{eqn:dec:lyap2}
\begin{align}
\nonumber 
\qquad \mb{E}[\mb{W}^{k+1}] - \mb{E}[\mb{W}^k]  &\leq  \left(\frac{L_f \alpha_k^2 }{2} + \frac{M_f}{L_{\m{y}}} a_3(\alpha_k)\right)\tilde{\sigma}_f^2 + \frac{3\alpha_k}{2} b^2,
\\
 &  -\frac{\alpha_k}{2}\mb{E}\left[\|\nabla f(\m{x}^k)\|^2\right] + \frac{3 M_f^2 \alpha_k }{2m}\sum \limits_{i=1}^m \frac{1}{\tau_i}\sum \limits_{\nu=0}^{\tau_i-1}\mb{E}\left[\|\m{x}_{i,\nu}^k-\m{x}^k\|^2\right]
\label{eqn:dec:lyap2i1}
\\
  & -  \left(\frac{\alpha_k}{2}-\frac{L_f\alpha_k^2}{2}-\frac{M_f}{L_\m{y}} a_1(\alpha_k)\right)  \mb{E}\left[ \left\| \frac{1}{m}\sum\limits_{i=1}^{m} \frac{1}{\tau_i} \sum\limits_{\nu=0}^{\tau_i-1} \bar{\m{h}}_i(\m{x}_{i,\nu}^k,\m{y}^{k+1})\right\|^2 \right] 
\label{eqn:dec:lyap2i2}   
   \\
   &+ \frac{M_f}{L_{\m{y}}} \left( \frac{3 M_f L_{\m{y}}\alpha_k }{2} +a_2(\alpha_k)\right) \mb{E}\left[\|\m{y}^{k+1}-\m{y}^*(\m{x}^k)\|^2\right]-  \frac{M_f}{L_{\m{y}}} \mb{E}\left[\|\m{y}^k-\m{y}^*(\m{x}^k)\|^2\right],
\label{eqn:dec:lyap2i3}  
\end{align}
\end{subequations}
where $a_1(\alpha)-a_3(\alpha)$ are defined in \eqref{eqn:a1-3}.

Let
\begin{equation}\label{eqn:stepsize0}
\alpha_i^k=\frac{\alpha_k}{\tau_i},~~\forall i \in \mathcal{S}~ ~~\textnormal{where}~~\alpha_k \leq \frac{1}{ 216M_f^2 +5M_f}.
\end{equation}
 The above choice of $\alpha_k$ satisfies the condition of Lemma~\ref{thm4:lemma2} and we have $54 M_f^2 \alpha_k^3 \leq \alpha_k^2/4$. Hence, from  Lemma~\ref{thm4:lemma2}, we get
\begin{equation}\label{eqn:step1}
\begin{aligned}
\eqref{eqn:dec:lyap2i1}
 \leq &- \frac{\alpha_k}{2}\mb{E}\left[\|\nabla f(\m{x}^k)\|^2\right] + \frac{81}{2} \alpha_k^3 M_f^2\tilde{\sigma}_f^2\\
&+ 54 M_f^2 \alpha_k^3 \left(M_f^2 \mb{E}\left[\|\m{y}^{k+1}-\m{y}^\star(\m{x}^k)\|^2\right]+  \mb{E}\left[\|\nabla f(\m{x}^k)\|^2\right]+ 3 b^2\right)\\
\leq &- \frac{\alpha_k}{4}\mb{E}\left[\|\nabla f(\m{x}^k)\|^2\right] + \frac{\alpha_k^2}{4}\tilde{\sigma}_f^2+ \frac{\alpha_k^2}{4} \left(M_f^2 \mb{E}\left[\|\m{y}^{k+1}-\m{y}^\star(\m{x}^k)\|^2\right]+ 3 b^2\right)\\
\leq &- \frac{\alpha_k}{4}\mb{E}\left[\|\nabla f(\m{x}^k)\|^2\right] + \frac{\alpha_k^2}{4}\tilde{\sigma}_f^2+ \frac{3\alpha_k^2}{4} b^2\\
&+  \frac{25M_f }{L_{\m{y}}} \left(\frac{M_fL_{\m{y}}}{4} \alpha_k^2 \right) T\beta_k^2\sigma_{g,1}^2 +\frac{M_f  }{L_{\m{y}}} \left( \frac{M_fL_{\m{y}}}{4} \alpha_k^2\right)\left(1-  \frac{{\beta}_k \mu_g}{2} \right)^T, 
\end{aligned}
\end{equation}
where the first inequlaity uses \eqref{eqn:stepsize0} and the last inequality follows from \eqref{eqn:err:fedina1}.

To guarantee the descent of $\mb{W}^k$, the following constraints need to be satisfied
\begin{equation}\label{eqn:stepsize1}
\begin{aligned}
\eqref{eqn:dec:lyap2i2} &\leq  0, \\
&\Longrightarrow~~ \frac{\alpha_k}{2}-\frac{L_f\alpha_k^2}{2} -\frac{M_f}{L_\m{y}} \left( L_{\m{y}}^2 \alpha^2_k+ \frac{L_{\m{y}}\alpha_k}{4M_f} +\frac{L_{\m{yx}} \alpha^2_k}{2\eta}\right) \geq 0,\\
&\Longrightarrow~~\alpha_k\leq \frac{1}{2L_f+4M_fL_{\m{y}}+\frac{2M_fL_{\m{yx}}}{L_{\m{y}}\eta}}, 
\end{aligned}
\end{equation}
where the second line uses \eqref{eqn:a1-3}.


Further, substituting \eqref{eqn:err:fedin1} in \eqref{eqn:dec:lyap2i3} gives
\begin{equation}\label{eqn:step2}
\begin{aligned}
 \eqref{eqn:dec:lyap2i3} & \leq  \frac{25M_f }{L_{\m{y}}} \left( \frac{3 M_f L_{\m{y}}\alpha_k }{2}  +a_2(\alpha_k)\right) T\beta_k^2\sigma_{g,1}^2
 \\
 & +\frac{M_f  }{L_{\m{y}}}  \left(\left( \frac{3 M_f L_{\m{y}}\alpha_k }{2}  +a_2(\alpha_k)\right)\left(1-  \frac{{\beta}_k \mu_g}{2} \right)^T -1\right) \mb{E}[\|\m{y}^{k}-\m{y}^*(\m{x}^k)\|^2].
\end{aligned}
\end{equation}
Substituting  \eqref{eqn:step1}--\eqref{eqn:step2} into \eqref{eqn:dec:lyap2} gives 
\begin{subequations}\label{eqn:dec:lyap}
\begin{align}
\nonumber 
   \mb{E}[\mb{W}^{k+1}] - \mb{E}[\mb{W}^k]  \leq & -\frac{\alpha_k}{4}\mb{E}[\|\nabla f(\m{x}^k)\|^2]
   \\
\nonumber    
   &+ \left(\frac{3}{2} \alpha_k+ \frac{3}{4}\alpha_k^2\right)b^2\\
   \nonumber 
   &+ \left(\left(\frac{L_f}{2}+\frac{1}{4}\right) \alpha_k^2  +  \frac{M_f}{L_{\m{y}}} a_3(\alpha_k) \right)\tilde{\sigma}_f^2 \\
   \nonumber 
   &+\frac{25M_f}{L_{\m{y}}}  \left( \frac{M_fL_{\m{y}}}{4} \alpha_k^2+  \frac{3 M_f L_{\m{y}}\alpha_k }{2}  +a_2(\alpha_k)\right) T\beta_k^2\sigma_{g,1}^2
\\
   &+  \frac{M_f  }{L_{\m{y}}}  \left( \left( \frac{M_fL_{\m{y}}}{4} \alpha_k^2+ \frac{3 M_f L_{\m{y}}\alpha_k }{2}  +a_2(\alpha_k)\right)\left(1-  \frac{{\beta}_k\mu_g}{2} \right)^T -1\right) \mb{E}[\|\m{y}^{k}-\m{y}^*(\m{x}^k)\|^2]. \label{eqn:dec:lyap4} 
\end{align}
\end{subequations}
Let $\beta_k <  \min\big(1/(6\ell_{g,1}),1\big)$. Then, we have  $\beta_k\mu_g/2 < 1 $. This together with \eqref{eqn:a1-3} implies that for any $\alpha_k >0$ 
\begin{equation}\label{eqn:stepsize2}
\begin{aligned}
      \eqref{eqn:dec:lyap4}& \leq  0,\\
     &\Longrightarrow~~ \left(1+  \frac{M_fL_{\m{y}}}{4} \alpha_k^2+ \frac{11M_f L_{\m{y}} \alpha_k}{2} + \frac{\eta L_{\m{yx}}\tilde{D}_f^2 \alpha^2_k}{2} \right)\left(1-  \frac{{\beta}_k\mu_g}{2} \right)^T -1 \leq  0,\\
      &\Longrightarrow \exp\left(\frac{M_fL_{\m{y}}}{4} \alpha_k^2+ \frac{11M_f L_{\m{y}} \alpha_k}{2} + \frac{\eta L_{\m{yx}}\tilde{D}_f^2 \alpha^2_k}{2}\right) \exp\left(- \frac{T{\beta}_k\mu_g}{2} \right) -1 \leq 0, \\
      &\Longrightarrow~~\beta_k\geq \frac{11 M_fL_{\m{y}} + \eta L_{\m{yx}}\tilde{D}_f^2 \alpha_k + \frac{M_f L_{\m{y}}\alpha_k}{2}}{\mu_g} \cdot \frac{\alpha_k}{T}.
\end{aligned}
\end{equation}

 From \eqref{eqn:stepsize0}, \eqref{eqn:stepsize1} and \eqref{eqn:stepsize2}, we select
 \begin{equation}\label{eq.step-cond-3}
    \alpha_k=\min\{\bar\alpha_1, \bar\alpha_2, \bar\alpha_3,\frac{\bar\alpha}{\sqrt{K}}\},~~~~~~~\beta_k=\frac{\bar{\beta}\alpha_k}{T},
\end{equation}
 where 
\begin{equation}\label{eqn2:dec:lyap}
\begin{aligned}
\bar{\beta}&:=\frac{1}{\mu_g} \left(11 M_fL_{\m{y}} + \eta L_{\m{yx}}\tilde{D}_f^2\bar\alpha_1 + \frac{M_fL_{\m{y}} \bar\alpha_1}{2}\right),\\
\bar\alpha_1&:=\frac{1}{2L_f+4M_fL_{\m{y}}+\frac{2M_fL_{\m{yx}}}{L_{\m{y}}\eta}},
~~~\bar\alpha_2:= \frac{T}{8\ell_{g,1} \bar{\beta}}, 
~~~\bar\alpha_3:= \frac{1}{216M_f^2+5M_f}, \\
\end{aligned}
\end{equation}

With the above choice of stepsizes, \eqref{eqn:dec:lyap} can be simplified as
\begin{equation}\label{eqn:dec3:lyap}
\begin{aligned}
   \mb{E}[\mb{W}^{k+1}] - \mb{E}[\mb{W}^k] & \leq  -\frac{\alpha_k}{4}\mb{E}[\|\nabla f(\m{x}^k)\|^2] \\
   &+ \left(\frac{3}{2} \alpha_k+ \frac{3}{4}\alpha_k^2\right)b^2\\
   &+ \left(\frac{L_f+\frac{1}{2}}{2} \alpha_k^2  +  \frac{M_f}{L_{\m{y}}} a_3(\alpha_k) \right)\tilde{\sigma}_f^2 \\
   &+\frac{25M_f}{L_{\m{y}}}  \left( \frac{M_fL_{\m{y}}}{4} \alpha_k^2+ \frac{3 M_f L_{\m{y}} }{2}  \alpha_k +a_2(\alpha_k)\right) T\beta_k^2\sigma_{g,1}^2\\
   & \leq  - \frac{\alpha_k}{4} \mb{E}[\|\nabla f(\m{x}^k)\|^2]  + c_1\alpha_k^2\sigma_{g,1}^2 +    \left(\frac{3}{2} \alpha_k+ \frac{3}{4}\alpha_k^2\right)b^2+ c_2\alpha_k^2\tilde\sigma_f^2,
\end{aligned}
\end{equation}
where the constants $c_1$ and $c_2$ are defined as
\begin{equation}\label{eqn44:dec:lyap}
\begin{aligned}
    &c_1=\frac{25 M_f}{L_{\m{y}}}\left( 1+ \frac{11 M_f L_{\m{y}}}{2} \bar{\alpha}_1  +  \left(\frac{M_fL_{\m{y}} + 2\eta L_{\m{yx}}\tilde{D}_f^2 }{4} \right)\bar{\alpha}_1^2 \right)\bar{\beta}^2\frac{1}{T},\\
    &c_2=\frac{L_f +\frac{1}{2}}{2} + M_fL_{\m{y}} +\frac{L_{\m{yx}}M_f}{4\eta L_{\m{y}}}.
\end{aligned}
\end{equation}
Then telescoping gives
\begin{equation}\label{eqn:dec:final:bilevel}
\begin{aligned}
    \frac{1}{K}\sum_{k=0}^{K-1}\mb{E}[\|\nabla f(\m{x}^k)\|^2]
    &\leq \frac{4}{\sum_{k=0}^{K-1}\alpha_k} \left(\Delta_{\mb{W}} +\sum_{k=0}^{K-1} \frac{3}{2}\left(\alpha_k+ \frac{\alpha_k^2}{2}\right) b^2 + c_1\alpha_k^2\sigma_{g,1}^2 + c_2 \alpha_k^2\tilde{\sigma}_{f}^2\right)\\
    &\leq \frac{4\Delta_{\mb{W}} }{\min\{\bar\alpha_1, \bar\alpha_2,\bar\alpha_3\}K} + \frac{4\Delta_{\mb{W}} }{\bar\alpha\sqrt{K}} +6\left(1+ \frac{\bar\alpha}{2\sqrt{K}}\right) b^2 +  \frac{4c_1\bar\alpha}{\sqrt{K}}\sigma_{g,1}^2 + \frac{4c_2\bar\alpha}{\sqrt{K}}\tilde\sigma_{f}^2 \\
    &= {\cal O}\left(\frac{1}{\min\{\bar{\alpha}_1, \bar{\alpha}_2, \bar{\alpha}_3\}K}+ \frac{\bar{\alpha}\max(\sigma_{g,1}^2, \sigma_{g,2}^2,\sigma_f^2)}{\sqrt{K}} + b^2 \right),
\end{aligned}
\end{equation}
where $\Delta_{\mb{W}} := \mb{W}^0-\mb{E}[\mb{W}^K]$.
\end{proof}

\subsection{Proof of Corollary~\ref{thm:fednest:bilevel}}
\begin{proof}
Let $\eta=\frac{M_f}{L_{\m{y}}}=\mc{O}(\kappa_g)$ in \eqref{eqn44:dec:lyap}. It follows from \eqref{eqn:lip:condi}, \eqref{eqn2:dec:lyap}, and \eqref{eqn44:dec:lyap} that  
\begin{equation}
\bar\alpha_1=\mc{O}(\kappa^{-3}_g), ~~\bar\alpha_2=\mc{O}(T\kappa^{-3}_g),~~\bar\alpha_3=\mc{O}(\kappa^{-4}_g),~~ c_1=\mc{O}(\kappa^9_g/T), ~~c_2=\mc{O}(\kappa^3_g).
\end{equation}
Further,  $N=\mc{O}(\kappa_g\log K)$  gives $b=\frac{1}{K^{1/4}}$. Now, if we select $\bar\alpha=\mc{O}(\kappa^{-2.5}_g)$ and $T=\mc{O}(\kappa^4_g)$, Eq.~\eqref{eqn:dec:final:bilevel} gives
\begin{align*}
    \frac{1}{K}\sum_{k=0}^{K-1}\mb{E}[\|\nabla f(\m{x}^k)\|^2] = \mc{O}\left(\frac{\kappa^4_g}{K} + \frac{\kappa^{2.5}_g}{\sqrt{K}}\right).
\end{align*}
To achieve $\varepsilon$-optimal solution, we need $K=\mc{O}(\kappa^5_g\varepsilon^{-2})$, and the samples in  $\xi$ and $\zeta$ are $\mc{O}(\kappa^5_g\varepsilon^{-2})$ and $\mc{O}(\kappa^9_g\varepsilon^{-2})$, respectively.
\end{proof}

\section {Proof for Federated Minimax Optimization}\label{sec:app:minmax}
Note that the  minimax optimization problem \eqref{fedminmax:prob} has the following bilevel form 
\begin{subequations}\label{fedminmax:prob2}
\begin{align}
\begin{array}{ll}
\underset{\m{x} \in \mb{R}^{{d}_1}}{\min} &
\begin{array}{c}
f(\m{x})=\frac{1}{m} \sum_{i=1}^{m} f_{i}\left(\m{x},\m{y}^*(\m{x})\right) 
\end{array}\\
\text{subj.~to} & \begin{array}[t]{l} \m{y}^*(\m{\m{x}})=\underset{ \m{y}\in \mb{R}^{{d}_2}}{\textnormal{argmin}}~~ -\frac{1}{m}\sum_{i=1}^{m} f_i\left(\m{x},\m{y}\right). 
\end{array}
\end{array}
\end{align}
Here, 
\begin{align}
 f_i(\m{x},\m{y}) &= \mb{E}_{\xi \sim \mc{C}_i}[f_i(\m{x}, \m{y}; \xi)]
\end{align}
\end{subequations}
is the loss functions of the $i^{\text{th}}$ client.

In this case, the hypergradient of  \eqref{fedminmax:prob2} is  
\begin{equation}
	\nabla f_i(\m{x})=\nabla_\m{x} f_i\big(\m{x}, \m{y}^*(\m{x})\big)+\nabla_\m{x} \m{y}^*(\m{x})^{\top}\nabla_\m{y} f_i\big(\m{x}, \m{y}^*(\m{x})\big)=\nabla_\m{x} f_i\big(\m{x}, \m{y}^*(\m{x})\big), 
\end{equation}
where the second equality follows from the optimality condition of the inner problem, i.e., $\nabla_\m{y}f(\m{x},\m{y}^*(\m{x}))=0$.

For each $i \in \mc{S}$, we can approximate $\nabla f_i(\m{x})$ on a vector $\m{y}$ in place of $\m{y}^*(\m{x})$, denoted as $\overline{\nabla}f_i(\m{x},\m{y}):=\nabla_\m{x} f_i\big(\m{x}, \m{y}\big)$. We also note that in the minimax case $\m{h}_i$ is an unbiased estimator of $\bar{\nabla} f_i(\m{x},\m{y})$. Thus, $b=0$. Therefore, we can apply \fedblo  using 
\begin{equation}\label{eqn:grad:min-max}
\begin{aligned}
\m{q}_{i,\nu}&=-\nabla_{\m{y}} f_i(\m{x}, \m{y}_{i,\nu};\xi_{i,\nu})+\nabla_{\m{y}} f_i(\m{x}, \m{y};\xi_{i,\nu})-\frac{1}{m} \sum_{i=1}^m \nabla_{\m{y}} f_i(\m{x}, \m{y};\xi_{i}),\\
\m{h}_{i,\nu}&=\nabla_{\m{x}} f_i(\m{x}_{i,\nu}, \m{y}^+;\xi_{i,\nu})-\nabla_{\m{x}} f_i(\m{x}, \m{y}^+;\xi_{i,\nu})+
\frac{1}{m} \sum_{i=1}^m \nabla_{\m{x}} f_i(\m{x}, \m{y}^+;\xi_{i}).
\end{aligned}
\end{equation}

\subsection{Supporting Lemmas}

Let $\m{z}=(\m{x},\m{y}) \in \mb{R}^{d_1+d_2}$. We make the following assumptions that are counterparts of Assumptions~\ref{assu:f} and  \ref{assu:bound:var}.
\vspace{0.1cm}
%
\begin{assumption}\label{assu:f:minmax} 
For all $i \in [m]$:
\begin{enumerate}[label={\textnormal{\textbf{(\ref{assu:f:minmax}\arabic*})}}]
\item 
$f_i(\m{z}), \nabla f_i(\m{z}), \nabla^2 f_i (\m{z})$ are respectively $\ell_{f,0}$, $\ell_{f,1}, \ell_{f,2}$-Lipschitz continuous; and 
\item 
$f_i(\m{x},\m{y})$ is $\mu_{f}$-strongly convex in $\m{y}$ for any fixed $\m{x}\in \mb{R}^{d_1}$.
\end{enumerate}
\end{assumption}
We use $\kappa_f=\ell_{f,1}/\mu_f$ to denote the condition number of the inner objective with respect to $\y$.
%
\begin{assumption}\label{assu:bound:var:minmax}
For all $i \in [m]$:
\begin{enumerate}[label={\textnormal{\textbf{(\ref{assu:bound:var:minmax}\arabic*})}}]
\item $\nabla f_i(\m{z};\xi)$ is unbiased estimators of $\nabla f_i(\m{z})$; and
\item Its variance is bounded, i.e., $\mb{E}_{\xi}[\|\nabla f_i(\m{z};\xi)-\nabla f_i(\m{z})\|^2] \leq \sigma_f^2$,  for some $\sigma_f^2$.
\end{enumerate}
\end{assumption}
In the following, we re-derive Lemma~\ref{lem:lips} for the finite-sum minimax problem \eqref{fedminmax:prob2}. 
\begin{lemma}\label{lem:lips:minmax} 
Under Assumptions~\ref{assu:f:minmax} and \ref{assu:bound:var:minmax}, we have $\bar{\m{h}}_i (\m{x}, \m{y})=\bar{\nabla} f_i(\m{x},\m{y})$ for all $i \in \mc{S}$ and  \eqref{eqn:newlips:b}--\eqref{eqn:newlips:f} hold with 
\begin{equation}\label{eqn:lip:condi:minimax}
\begin{split}
    &L_{yx}=\frac{\ell_{f,2}+\ell_{f,2}L_{\m{y}}}{\mu_f} + \frac{\ell_{f,1}(\ell_{f,2}+\ell_{f,2}L_{\m{y}})}{\mu_f^2}={\cal O}(\kappa^3_f),\\
    &M_f=\ell_{f,1}={\cal O}(1), ~~~ L_f=(\ell_{f,1}+\frac{\ell_{f,1}^2}{\mu_f})={\cal O}(\kappa_f), \\
     L_{\m{y}}&=\frac{\ell_{f,1}}{\mu_f}={\cal O}(\kappa_f),~~~\tilde\sigma_f^2 = \sigma_f^2,~~~ \tilde{D}_f^2=\ell_{l,0}^2 + \sigma_f^2,
\end{split}    
\end{equation}
where $\ell_{f,0}$, $\ell_{f,1}, \ell_{f,2}$, $\mu_f$, and $\sigma_f$ are given in Assumptions~\ref{assu:f:minmax} and \ref{assu:bound:var:minmax}.
\end{lemma}

\subsection{Proof of Corollary~\ref{thm:fednest:minmax}}
\begin{proof}
Let $\eta=1$. From \eqref{eq.step-cond-3} and \eqref{eqn2:dec:lyap}, we have
\begin{subequations}\label{eqn:scons:steps:minimax}
\begin{equation}
    \alpha_k=\min \left\{\bar\alpha_1, \bar\alpha_2, \bar\alpha_3,\frac{\bar\alpha}{\sqrt{K}}\right\},~~~~~~~\beta_k=\frac{\bar{\beta}\alpha_k}{T},
\end{equation}
 where 
\begin{equation}
\begin{aligned}
\bar{\beta}&=\frac{1}{\mu_g} \left(11 \ell_{f,1} L_{\m{y}} + L_{yx}\tilde{D}_f^2\bar\alpha_1 + \frac{\ell_{f,1}L_{\m{y}} \bar\alpha_1}{2}\right),\\
\bar\alpha_1&=\frac{1}{2L_f+4 \ell_{f,1} L_{\m{y}}+\frac{2 \ell_{f,1}L_{yx}}{L_{\m{y}}}},
~~~\bar\alpha_2= \frac{T}{8\ell_{g,1} \bar{\beta}}, 
~~~\bar\alpha_3= \frac{1}{216\ell_{f,1}^2+5\ell_{f,1}}. \\
\end{aligned}
\end{equation}
\end{subequations}
Using the above choice of stepsizes, \eqref{eqn:dec:final:bilevel} reduces to
\begin{equation}\label{eqn4:dec:lyap:minimax}
\begin{aligned}
    \frac{1}{K}\sum_{k=0}^{K-1}\mb{E}[\|\nabla f(\m{x}^k)\|^2]
    &\leq \frac{4\Delta_{\mb{W}}}{K\min\{\bar\alpha_1, \bar\alpha_2, \bar\alpha_3\}} + \frac{4\Delta_{\mb{W}}}{\bar{\alpha}\sqrt{K}} + \frac{4(c_1+c_2)\bar{\alpha}}{\sqrt{K}}\sigma_{f}^2,
\end{aligned}
\end{equation}
 where $\Delta_{\mb{W}}= \mb{W}^0-\mb{E}[\mb{W}^K]$, 
\begin{equation}\label{eqn:cons:cs:minimax}
    \begin{split}
     & c_1=\frac{25 \ell_{f,1}}{L_{\m{y}}}\left( 1+ \frac{11 \ell_{f,1} L_{\m{y}}}{2} \bar{\alpha}_1  +  \left(\frac{\ell_{f,1} L_{\m{y}} + 2 L_{\m{yx}}\tilde{D}_f^2 }{4} \right)\bar{\alpha}_1^2 \right)\bar{\beta}^2\frac{1}{T},
    \\
    &c_2=\frac{L_f +\frac{1}{2}}{2} + \ell_{f,1} L_{\m{y}} +\frac{L_{\m{yx}}\ell_{f,1}}{4 L_{\m{y}}}.       
    \end{split}
\end{equation}
Let $\bar{\alpha}=\mc{O}(\kappa^{-1}_f)$. Since by our assumption, $T=\mc{O}(\kappa_f)$, it follows from \eqref{eqn:lip:condi:minimax} and \eqref{eqn:cons:steps:minimax} that 
\begin{equation}\label{eqn:o:conditions}
    \bar\alpha_1={\cal O}(\kappa^{-2}_f),~~\bar\alpha_2={\cal O}(\kappa^{-1}_f),~~ \bar\alpha_3=\mc{O}(1),~~c_1= {\cal O}(\kappa^2_f), ~~c_2={\cal O}(\kappa^2_f).
\end{equation}
Substituting \eqref{eqn:o:conditions} in \eqref{eqn4:dec:lyap:minimax} and \eqref{eqn:cons:cs:minimax} gives
\begin{align}
    \frac{1}{K}\sum_{k=0}^{K-1}\mb{E}[\|\nabla f(\m{x}^k)\|^2]={\cal O}\left(\frac{\kappa^2_f}{K} + \frac{\kappa_f}{\sqrt{K}}\right).
\end{align}
To achieve $\varepsilon$-accuracy, we need $K={\cal O}(\kappa^2_f\varepsilon^{-2})$.
\end{proof}

\section{Proof for Federated Compositional Optimization}\label{sec:app:compos}
Note that in the stochastic compositional problem \eqref{fedcompos:prob2}, the inner function $f_i(\m{x},\m{y};\xi)=f_i(\m{y};\xi)$ for all $i \in \mc{S}$, and the outer function is  $g_i(\m{x},\m{y};\zeta)=\frac{1}{2}\|\m{y}-\m{r}_i(\m{x};\zeta)\|^2$, for all $i \in \mc{S}$.  
In this case, we have 
\begin{align}
\nabla_{\m{y}} g_i\left(\m{x},\m{y}\right) = \m{y}_i-\m{r}_i(\m{x};\zeta),~~\nabla_{\m{yy}}g(\m{x},\m{y};\zeta)=\mathbf{I}_{d_2\times d_2},~~\textnormal{and}~~ \nabla_{\m{xy}}g(\m{x},\m{y};\zeta)=- \frac{1}{m}\sum_{i=1}^m\nabla \m{r}_i(\m{x};\zeta)^{\top}.    
\end{align}
Hence, the hypergradient of \eqref{fedcompos:prob2} has the following form
\begin{align}
\nonumber 
	\nabla f_i(\m{x})&= \nabla_{\m{x}}f_i\left(\m{y}^*(\m{x})\right)\\
\nonumber
&-\nabla^2_{\m{x}\m{y}} g(\m{x},\m{y}^*(\m{x}))[\nabla^2_{\m{y}}g(\m{x},\m{y}^*(\m{x}))]^{-1} \nabla_\m{y} f_i \left(\m{y}^*(\m{x})\right)\\
	&=(\frac{1}{m}\sum_{i=1}^m\nabla \m{r}_i(\m{x}))^{\top}\nabla_\m{y} f_i(\m{y}^*(\m{x})).
\end{align}
We can obtain an approximate gradient $\nabla f_i(\m{x})$ by replacing  $\m{y}^\star(\m{x})$ with $\m{y}$; that is  $\bar{\nabla} f_i(\m{x},\m{y}) = (\frac{1}{m}\sum_{i=1}^m\nabla \m{r}_i(\m{x}))^{\top}\nabla_\m{y} f_i(\m{y})$. It should be mentioned that in the compositional case  $b=0$. Thus, we can apply \fedblo and \lfedblo using the above gradient approximations.
%
\subsection{Supporting Lemmas}

Let $\m{z}=(\m{x},\m{y}) \in \mb{R}^{d_1+d_2}$. We make the following assumptions that are counterparts of Assumptions~\ref{assu:f} and  \ref{assu:bound:var}.
\vspace{0.1cm}
%
\begin{assumption}\label{assu:f:compos} 
For all $i \in [m]$, $f_i(\m{z}), \nabla f_i(\m{z}), \m{r}_i(\m{z}), \nabla \m{r}_i (\m{z})$ are respectively $\ell_{f,0}$, $\ell_{f,1}, \ell_{\m{r},0}, \ell_{\m{r},1} $-Lipschitz continuous.
\end{assumption}
\begin{assumption}\label{assu:bound:var:compos}
For all $i \in [m]$:
\begin{enumerate}[label={\textnormal{\textbf{(\ref{assu:bound:var:compos}\arabic*})}}]
\item $\nabla f_i(\m{z};\xi)$, $\m{r}_i(\m{x};\zeta)$,  $\nabla \m{r}_i (\m{x};\zeta)$ are unbiased estimators of $\nabla f_i(\m{z})$,  $ \m{r}_i(\m{x})$,  and $ \nabla \m{r}_i (\m{x})$.
\item Their variances are bounded, i.e., $\mb{E}_{\xi}[\|\nabla f_i(\m{z};\xi)-\nabla f_i(\m{z})\|^2] \leq \sigma_f^2$,   $\mb{E}_{\zeta}[\|\m{r}_i(\m{z};\zeta) -\m{r}_i(\m{z})\|^2] \leq \sigma_{\m{r},0}^2$, and $\mb{E}_{\zeta}[\|\nabla \m{r}_i(\m{z};\zeta)-\nabla \m{r}_i(\m{z})\|^2] \leq \sigma_{\m{r},1}^2$ for some $\sigma_f^2, \sigma_{\m{r},0}^2$, and $\sigma_{\m{r},1}^2$.
\end{enumerate}
\end{assumption}

The following lemma is the counterpart of Lemma~\ref{lem:lips}. The proof is similar to \citep[Lemma~7]{chen2021closing}.
\begin{lemma}\label{lem:lips:compos} 
Under Assumptions~\ref{assu:f:compos} and \ref{assu:bound:var:compos}, we have $\bar{\m{h}}_i (\m{x}, \m{y})=\bar{\nabla} f_i(\m{x},\m{y})$ for all $i \in \mc{S}$, and \eqref{eqn:newlips:b}--\eqref{eqn:newlips:f} hold with 
\begin{equation}
\begin{split}
&M_f = \ell_{\m{r},0}\ell_{f,1}, ~~~ L_{\m{y}} =\ell_{\m{r},0}, ~~~ L_{f} =\ell_{\m{r},0}^2\ell_{f,1} + \ell_{f,0}\ell_{\m{r},1}, ~~L_{\m{yx}} =\ell_{\m{r},1},\\
&\tilde\sigma_f^2=\ell_{\m{r},0}^2\sigma_f^2+(\ell_{f,0}^2 + \sigma_f^2)\sigma_{\m{r},1}^2, ~~~\tilde{D}_f^2=(\ell_{f,0}^2+\sigma_f^2)(\ell_{\m{r},0}^2+\sigma_{\m{r},1}^2).        
\end{split}
\end{equation}

\end{lemma}

\subsection{Proof of Corrollary~\ref{thm:fednest:compos}}
\begin{proof}
By our assumption $T=1$. Let $ \bar{\alpha}=1$ and  $\eta=1/L_{\m{yx}}$.  From \eqref{eq.step-cond-3} and \eqref{eqn2:dec:lyap}, we obtain
\begin{subequations}\label{eqn:cons:steps:minimax}
\begin{equation}
    \alpha_k=\min \left\{\bar\alpha_1, \bar\alpha_2, \bar\alpha_3,\frac{1}{\sqrt{K}}\right\},~~~~~~~\beta_k=\bar{\beta}\alpha_k,
\end{equation}
 where 
\begin{equation}
\begin{aligned}
\bar{\beta}&=\frac{1}{\mu_g} \left(11  \ell_{f,1} \ell_{\m{r},0}^2 +\tilde{D}_f^2\bar\alpha_1 + \frac{\ell_{f,1}  \ell_{\m{r},0}^2 \bar\alpha_1}{2}\right),\\
\bar\alpha_1&=\frac{1}{2 \ell_{f,0}\ell_{\m{r},1} +6  \ell_{f,1} \ell_{\m{r},0}^2+2\ell_{f,1}\ell_{\m{r},1}^2},
~~~\bar\alpha_2= \frac{1}{8\ell_{\m{r},0} \bar{\beta}}, 
~~~\bar\alpha_3= \frac{1}{216(\ell_{\m{r},0}\ell_{f,1})^2+5\ell_{\m{r},0}\ell_{f,1}}. \\
\end{aligned}
\end{equation}
\end{subequations}
Then, using \eqref{eqn:dec:final:bilevel}, we obtain
\begin{equation}\label{eqn:final:compos}
\begin{aligned}
\frac{1}{K}\sum_{k=0}^{K-1}\mb{E}\left[\|\nabla f(\m{x}^k)\|^2\right]
    &=
     {\cal O}\left(\frac{1}{\sqrt{K}}\right).
\end{aligned}
\end{equation}
This completes the proof.
\end{proof}


\section{Proof for Federated Single-Level Optimization}\label{sec:app:sing}
Next, we re-derive Lemmas~\ref{lem:dec} and \ref{thm4:lemma2} for single-level nonconvex FL under Assumptions~\ref{assu:s:f} and \ref{assu:s:bound:var}. 

\begin{lemma}[Counterpart of Lemma~\ref{lem:dec}]\label{lem:s:dec}
Suppose Assumptions~\ref{assu:s:f} and \ref{assu:s:bound:var} hold. Further, assume $\tau_i \geq 1$ and $\alpha_i=\alpha/\tau_i$, for all $ i \in \mathcal{S}$ and some positive constant $\alpha$. Then, \fedout guarantees:
\begin{equation}
\begin{aligned}
\mb{E}\left[f(\m{x}^{+})\right] - \mb{E}\left[f(\m{x})\right]& \leq - \frac{\alpha}{2} \left(1- \alpha L_f\right) \mb{E}\left[\left\|\frac{1}{m}\sum\limits_{i=1}^{m} \frac{1}{\tau_i}\sum\limits_{\nu=0}^{\tau_i-1}  \nabla f_i(\m{x}_{i,\nu})\right\|^2\right]\\
   & - \frac{\alpha}{2} \mb{E}\left[\| \nabla f(\m{x})\|^2\right] + \frac{ \alpha L_f^2 }{ 2 m} \sum\limits_{i=1}^{m}   \frac{1}{\tau_i} \sum\limits_{\nu=0}^{\tau_i-1}  \mb{E}\left[\|\m{x}_{i,\nu} -\m{x}\|^2\right] +  \frac{\alpha^2 L_f}{2} \sigma_f^2.
\end{aligned}
\label{eqn0:lem:s:dec}
\end{equation}
\end{lemma}
\begin{proof}
By applying  Algorithm~\ref{alg:fedout} to the single-level optimization problem \eqref{fedsingle:prob2}, we get 
\begin{align*}
\m{x}_{i,0}=\m{x}~~\forall i\in\mathcal{S},~~~~\m{x}^{+}&=\m{x}- \frac{1}{m}\sum\limits_{i=1}^{m} \alpha_i \sum\limits_{\nu=0}^{\tau_i-1} \m{h}_i(\m{x}_{i,\nu}) - \m{h}_i(\m{x})  +\m{h}(\m{x})\\  
&=\m{x}- \frac{1}{m}\sum\limits_{i=1}^{m} \alpha_i \sum\limits_{\nu=0}^{\tau_i-1} \m{h}_i(\m{x}_{i,\nu}),
\end{align*}    
where 
\begin{align*}
\m{h}_i(\m{x}_{i,\nu})=\nabla f_i(\m{x}_{i,\nu};\xi_{i,\nu}),~~~~\m{h}_i(\m{x})=\nabla f_i(\m{x};\xi_{i,\nu}), ~~~~\textnormal{and}~~~~\m{h}(\m{x})= 1/m\sum_{i=1}^m \nabla f_i(\m{x};\xi_{i}).
\end{align*}
This together with  Assumption~\ref{assu:s:f} implies that 
\begin{equation}\label{eqn2:s:lem:dec}
    \begin{aligned}
\mb{E}\left[f(\m{x}^{+})\right] - \mb{E} \left[f(\m{x}) \right]
    \leq &   \mb{E} \left[\langle \m{x}^{+} - \m{x}, \nabla f(\m{x}) \rangle\right] + \frac{L_f}{2} \mb{E} \left[\Vert \m{x}^{+}-\m{x} \Vert^2 \right]\\
    = &- \mb{E}\left[ \left\langle \frac{1}{m}\sum\limits_{i=1}^{m} \alpha_i \sum\limits_{\nu=0}^{\tau_i-1} \m{h}_i(\m{x}_{i,\nu}) , \nabla f(\m{x}) \right\rangle \right]\\
    & +\frac{L_f}{2}  \mb{E}\left[ \left\| \frac{1}{m}\sum\limits_{i=1}^{m} \alpha_i \sum\limits_{\nu=0}^{\tau_i-1} \m{h}_i(\m{x}_{i,\nu})\right\|^2\right]. 
    \end{aligned}
\end{equation}
For the first term on the RHS of~\eqref{eqn2:s:lem:dec}, we obtain
\begin{equation}
\begin{aligned}
 - \mb{E}\left[ \left\langle \frac{1}{m}\sum\limits_{i=1}^{m} \alpha_i \sum\limits_{\nu=0}^{\tau_i-1}  \m{h}_i(\m{x}_{i,\nu}), \nabla f(\m{x}) \right\rangle \right] =& - \mb{E}\left[\frac{1}{m}\sum\limits_{i=1}^{m} \frac{\alpha}{\tau_i} \sum\limits_{\nu=0}^{\tau_i-1} \mb{E}\left[ \left\langle \m{h}_i(\m{x}_{i,\nu}), \nabla f(\m{x}) \right\rangle \mid {\cal F}_{i,\nu-1}  \right]\right]\\
 =& - \frac{\alpha}{2} \mb{E}\left[\left\|  \frac{1}{m}\sum\limits_{i=1}^{m} \frac{1}{\tau_i} \sum\limits_{\nu=0}^{\tau_i-1} \nabla f_i(\m{x}_{i,\nu}) \right\|^2\right]- \frac{\alpha}{2} \mb{E}\left[\left\| \nabla f(\m{x})\right\|^2\right]\\
&+ \frac{\alpha}{2} \mb{E}\left[ \left\| \frac{1}{m}\sum\limits_{i=1}^{m} \frac{1}{\tau_i} \sum\limits_{\nu=0}^{\tau_i-1} \nabla f_i(\m{x}_{i,\nu})-\nabla f(\m{x})\right\|^2\right], \\
\leq& - \frac{\alpha}{2} \mb{E}\left[\left\|  \frac{1}{m}\sum\limits_{i=1}^{m} \frac{1}{\tau_i} \sum\limits_{\nu=0}^{\tau_i-1} \nabla f_i(\m{x}_{i,\nu}) \right\|^2\right]- \frac{\alpha}{2} \mb{E}\left[  \left\| \nabla f(\m{x})\right\|^2\right]\\
&+ \frac{\alpha L_f^2 }{2m} \sum\limits_{i=1}^{m}   \frac{1}{\tau_i} \sum\limits_{\nu=0}^{\tau_i-1}  \mb{E}\left[\left\|\m{x}_{i,\nu} -\m{x}\right\|^2\right].
\end{aligned}
\label{eqn5:s:lem:dec}
\end{equation} 
Here, the first equality follows from the law of total expectation; the second equality uses $\mb{E}\left[\nabla f_i(\m{x}_{i,\nu})\right]= \mb{E}\left[\mb{E}\left[\m{h}_i(\m{x}_{i,\nu})|{\cal F}_{i,\nu-1}\right]\right]$ and the fact that 2 $\m{a}^\top \m{b} = \|\m{a}\|^2 + \|\m{b}\|^2 - \|\m{a}-\m{b}\|^2$; and the last inequality is obtained from Assumption~\ref{assu:s:f}.

For the second term on the RHS of~\eqref{eqn2:s:lem:dec}, Assumption~\ref{assu:s:bound:var} together with Lemma~\ref{lem:Jens} gives 
\begin{equation}
 \begin{aligned}
\mb{E}\left[ \left\| \frac{1}{m}\sum\limits_{i=1}^{m} \alpha_i \sum\limits_{\nu=0}^{\tau_i-1} \m{h}_i(\m{x}_{i,\nu})\right\|^2 \right]&= \alpha^2 \mb{E}\left[  \left\|\frac{1}{m}\sum\limits_{i=1}^{m} \frac{1}{\tau_i} \sum\limits_{\nu=0}^{\tau_i-1} \left( \m{h}_i(\m{x}_{i,\nu})- \nabla f_i(\m{x}_{i,\nu}) + \nabla f_i(\m{x}_{i,\nu}) \right) \right\|^2 \right]\\
 &\leq \alpha^2 \mb{E}\left[\left\|\frac{1}{m}\sum\limits_{i=1}^{m} \frac{1}{\tau_i}\sum\limits_{\nu=0}^{\tau_i-1}  \nabla f_i(\m{x}_{i,\nu})\right\|^2\right]+  \alpha^2 \sigma_f^2.
    \end{aligned}
\label{eqn6:s:lem:dec}
\end{equation}
Plugging \eqref{eqn6:s:lem:dec} and \eqref{eqn5:s:lem:dec} into \eqref{eqn2:s:lem:dec} gives the desired result.
\end{proof}

The following lemma provides a bound on the \textit{drift} of each $\m{x}_{i,\nu}$ from $\m{x}$ for stochastic nonconvex singl-level problems. It should be mentioned that similar drifting bounds are provided in the literature under either strong convexity~\citep{mitra2021linear} and/or bounded dissimilarity assumptions~\citep{wang2020tackling,reddi2020adaptive,li2020federated}. 

\begin{lemma}[Counterpart of Lemma~\ref{thm4:lemma2}]\label{thm4:s:lemma2}
Suppose Assumptions~\ref{assu:s:f} and \ref{assu:s:bound:var} hold. Further, assume $\tau_i \geq 1$ and $\alpha_i=\alpha/\tau_i, \forall i \in \mathcal{S}$, where $\alpha \leq 1/(3 L_f)$. Then, for all $\nu \in \{0,\ldots, \tau_i-1\}$, \fedout gives
\begin{align}\label{eqn1:lemm:s:drift}
 \mb{E}\left[\left\|\m{x}_{i,\nu}-\m{x}\right\|^2\right]\leq 12 \tau_i^2 \alpha_i^2  \mathbb{E}\left[\left\|\nabla  f(\m{x})\right\|^2\right] +  27 \tau_i \alpha_i^2 \sigma_f^2.
\end{align}
\end{lemma}
\begin{proof}
The result trivially holds for $\tau_i=1$. Similar to what is done in
the proof of Lemma~\ref{thm4:lemma2}, let $\tau_i>1$ and define
\begin{equation}\label{eqn1+:s:lemm:drift}
\begin{aligned}
\m{v}_{i,\nu}&:=\nabla f_i(\m{x}_{i,\nu})- \nabla f_i(\m{x}) + \nabla f (\m{x}),\\
\m{w}_{i,\nu}&:= \m{h}_i(\m{x}_{i,\nu})-\nabla f_i(\m{x}_{i,\nu})+  \nabla f_i(\m{x})-\m{h}_i(\m{x})+\m{h}(\m{x})- \nabla f(\m{x}).
\end{aligned}
\end{equation}

From Algorithm~\ref{alg:fedout}, for each $i\in\mathcal{S}$, and $\forall  \nu \in\{0,\ldots,\tau_i-1\}$, we obtain
\begin{equation*}
\begin{aligned}
        \m{x}_{i,\nu+1}-\m{x}&=\m{x}_{i,\nu}-\m{x}-\alpha_i \left(\m{h}_i(\m{x}_{i,\nu})-\m{h}_i(\m{x})+\m{h}(\m{x})\right)\\
          &= \m{x}_{i,\nu}-\m{x}-\alpha_i \left(\m{v}_{i,\nu}+\m{w}_{i,\nu}\right),
    \end{aligned}
\end{equation*}
which implies that 
\begin{equation}\label{eqn2:s:lemm:drift}
\begin{aligned}
\mb{E}\left[\|\m{x}_{i,\nu+1}-\m{x}\|^2\right]
&=\mb{E}\left[\|\m{x}_{i,\nu}-\m{x}-\alpha_i \m{v}_{i,\nu}\|^2\right]+\alpha_i^2\mb{E}\left[\|{\m{w}_{i,\nu}}\|^2\right]\\&-2\mb{E}\left[\mb{E}\left[\langle \m{x}_{i,\nu}-\m{x}-\alpha_i\m{v}_{i,\nu},\alpha_i \m{w}_{i,\nu} \rangle\mid\mathcal{F}_{i,\nu-1}\right]\right] \\ 
&= \mb{E}\left[\|\m{x}_{i,\nu}-\m{x}-\alpha_i\m{v}_{i,\nu}\|^2\right]+\alpha_i^2\mb{E}\left[\|\m{w}_{i,\nu}\|^2\right].
\end{aligned}
\end{equation}
Here, the last equality uses Lemma~\ref{lem:rand:zer} since  $\mb{E}[\m{w}_{i,\nu}|\mathcal{F}_{i,\nu-1}]=0$.

From  Assumption~\ref{assu:s:bound:var} and Lemma~\ref{lem:Jens}, for $\m{w}_{i,\nu}$ defined in \eqref{eqn1+:s:lemm:drift}, we have
\begin{equation}
\label{eqn3b:s:lemm:drift}
\begin{aligned}
\mb{E}\left[\|\m{w}_{i,\nu}\|^2\right] &\leq 3   \mb{E}\left[\|\m{h}_i(\m{x}_{i,\nu})-\nabla f_i(\m{x}_{i,\nu})\|^2+\|\nabla f_i(\m{x})-\m{h}_i(\m{x})\|^2+\|\m{h}(\m{x})- \nabla  f  (\m{x})\|^2\right]\\
        & \leq 9\sigma_f^2.
    \end{aligned}
\end{equation}
Substituting  \eqref{eqn3b:s:lemm:drift} into \eqref{eqn2:s:lemm:drift}, we get
\begin{equation}
\label{eqn4:s:lemm:drift}
\begin{aligned}
\mb{E}\left[\|\m{x}_{i,\nu}-\m{x}-\alpha_i\m{v}_{i,\nu}\|^2\right] &\leq  \left(1+\frac{1}{2\tau_i-1}\right)\mb{E}\left[\|\m{x}_{i,\nu}-\m{x}\|^2\right]+ 2\tau_i \alpha_i^2 \mb{E}\left[\|\m{v}_{i,\nu}\|^2\right] +9 \alpha_i^2\sigma_f^2 \\
       &\leq  \left(1+\frac{1}{2\tau_i-1}+4 \tau_i \alpha_i^2 L_f^2\right)\mb{E}\left[\|\m{x}_{i,\nu}-\m{x}\|^2\right]
 + 4 \tau_i \alpha_i^2  \mb{E}\left[\|\nabla f(\m{x})\|^2\right] +9 \alpha_i^2\sigma_f^2\\
        &\leq  \left(1+\frac{1}{\tau_i-1}\right)\mb{E}\left[\|\m{x}_{i,\nu}-\m{x}\|^2\right]+ 4 \tau_i \alpha_i^2  \mb{E}\left[\|\nabla f(\m{x})\|^2\right] +9 \alpha_i^2\sigma_f^2.
    \end{aligned}
\end{equation}
Here, the first inequality follows from Lemma~\ref{lem:trig}; the second inequality uses Assumption~\ref{assu:s:f} and Lemma~\ref{lem:Jens};
and the last inequality follows by noting $\alpha_i=\alpha/\tau_i, \forall i \in \mathcal{S}$ and $\alpha \leq 1/(3 L_f)$.

Now, iterating equation \eqref{eqn4:s:lemm:drift} and using $\m{x}_{i,0}=\m{x}, \forall i\in\mathcal{S}$, we obtain
\begin{equation}
\label{eqn6:lemm:drift}
    \begin{aligned}
        \mb{E}\left[\|\m{x}_{i,\nu}-\m{x}\|^2\right]&\leq  \left(   4 \tau_i \alpha_i^2  \mb{E}\left[\|\nabla f(\m{x})\|^2\right] +9 \alpha_i^2\sigma_f^2 \right) \sum\limits_{j=0}^{\nu -1} \left(1+\frac{1}{\tau_i-1}\right)^j\\
        &\leq 12 \tau_i^2 \alpha_i^2  \mb{E}\left[\|\nabla f(\m{x})\|^2\right] +27\tau_i \alpha_i^2 \sigma_f^2,
    \end{aligned}
\end{equation}
where the second inequality uses \eqref{eqn:sum:geom}. This completes the proof.
\end{proof}
\subsection{Proof of Theorem~\ref{thm:fednest:s:level}}
\begin{proof}
Let $\bar{\alpha}_1:= 1/(3L_f(1 + 8 L_f))$. Note that by our assumption $\alpha_k \leq \bar{\alpha}_1$. Hence, the stepsize $\alpha_k$ satisfies the condition of Lemma~\ref{thm4:s:lemma2}, and we also have $ 6 L_f^2 \alpha_k^3 \leq \alpha_k^2/4 \leq  \alpha_k/4$. This together with Lemmas \ref{lem:s:dec} and \ref{thm4:s:lemma2} gives
\begin{equation}\label{eqn:dec:s:lyap}
\begin{aligned}
    \mb{E}\left[f(\m{x}^{k+1})\right] - \mb{E} \left[f(\m{x}^k) \right] & \leq  -\frac{\alpha_k}{2}\mb{E}[\|\nabla f(\m{x}^k)\|^2] + \frac{L_f^2 \alpha_k }{2m}\sum \limits_{i=1}^m \frac{1}{\tau_i}\sum \limits_{\nu=0}^{\tau_i-1}\mb{E}\Big[\|\m{x}_{i,\nu}^k-\m{x}^k\|^2\Big]
\\
   & - \frac{\alpha_k}{2} (1- \alpha_k L_f) \mathbb{E}\left[ \left\| \frac{1}{m}\sum\limits_{i=1}^{m} \frac{1}{\tau_i} \sum\limits_{\nu=0}^{\tau_i-1} \nabla f_i(\m{x}_{i,\nu}^k)\right\|^2 \right] + \frac{\alpha^2_k L_f}{2} \sigma_f^2 
   \\
   & \leq   -\frac{\alpha_k}{2}\mb{E}[\|\nabla f(\m{x}^k)\|^2] + \frac{L_f^2 \alpha_k }{2m}\sum \limits_{i=1}^m \frac{1}{\tau_i}\sum \limits_{\nu=0}^{\tau_i-1}\mb{E}\Big[\|\m{x}_{i,\nu}^k-\m{x}^k\|^2\Big]+  \frac{\alpha^2_k L_f}{2}  \sigma_f^2
\\
&\leq  - \frac{\alpha_k}{2}\mb{E}[\|\nabla f(\m{x}^k)\|^2] + 6 L_f^2 \alpha_k^3 \mb{E}[\|\nabla f(\m{x}^k)\|^2] + \left(\frac{27}{2}\alpha_k^3  L_f^2 +  \frac{\alpha^2_k L_f}{2} \right)\sigma_f^2  \\
& \leq   - \frac{\alpha_k}{4}\mb{E}[\|\nabla f(\m{x}^k)\|^2] +  (1+L_f)  \alpha^2_k  \sigma_f^2, 
\end{aligned}
\end{equation}
where the second and last inequalities follow from  \eqref{eqn:param:s:choice}. 

Summing \eqref{eqn:dec:s:lyap} over $k$ and using our choice of stepsize in \eqref{eqn:param:s:choice}, we obtain
\begin{equation}\label{eqn:dec:s:final:bilevel}
\begin{aligned}
    \frac{1}{K}\sum_{k=0}^{K-1}\mb{E}[\|\nabla f(\m{x}^k)\|^2] 
     &\leq \frac{4 \Delta_{f}}{K} \cdot  \min\left\{ \frac{1}{\bar{\alpha}_1},\frac{\sqrt{K}}{\bar{\alpha}}\right\} + 4 (1+L_f)  
      \sigma_{f}^2  \cdot \frac{\bar{\alpha}}{\sqrt{K}}  \\
 &\leq \left(\frac{4 \Delta_{f}}{ \bar{\alpha}_1}\right)\frac{1}{K} + \left( \frac{4\Delta_{f} }{\bar{\alpha}} + 4  (1+L_f) \bar{\alpha} \sigma_f^2 \right) \frac{1}{\sqrt{K}}, \\
 & \leq \mc{O}\left( \frac{ \Delta_f}{ \bar{\alpha}_1 K} + \frac{  \frac{\Delta_f}{\bar{\alpha}} + \bar{\alpha}\sigma_f^2 }{\sqrt{K}}\right),
\end{aligned}
\end{equation}
where $\Delta_{f}= f(\m{x}^0)-\mb{E}[f(\m{x}^K)]$.
\end{proof}

\section{Other Technical Lemmas}\label{sec:techn}
 We collect additional technical lemmas in this section.

\begin{lemma}\label{lem:Jens}
For any set of vectors $\{\m{x}_i\}_{i=1}^m$ with $\m{x}_i\in\mathbb{R}^d$, we have
 \begin{equation}
     \norm[\bigg]{ \sum\limits_{i=1}^{m} \m{x}_i}^2 \leq m \sum\limits_{i=1}^{m} {\Vert \m{x}_i \Vert}^2.
\label{eqn:Jens}
 \end{equation} 
 \end{lemma}
\begin{lemma}\label{lem:trig}
For any $\m{x},\m{y}\in\mathbb{R}^d$, the following holds for any $c >0$:
\begin{equation}
    {\Vert \m{x}+\m{y} \Vert}^2 \leq (1+c){\Vert \m{x} \Vert}^2 + \left(1+\frac{1}{c}\right){\Vert \m{y} \Vert}^2.
\label{eqn:triang}
\end{equation}
\end{lemma}
\begin{lemma}\label{lem:rand:zer}
For any set of independent, mean zero random variables $\{\m{x}_i\}_{i=1}^m$ with $\m{x}_i\in\mathbb{R}^d$, we have
 \begin{equation}
     \mb{E}\left[\norm[\bigg]{ \sum\limits_{i=1}^{m} \m{x}_i}^2\right] = \sum\limits_{i=1}^{m} \mb{E}\left[{\Vert \m{x}_i \Vert}^2\right].
\label{eqn:randJens}
 \end{equation} 
\end{lemma}




\section{Additional Experimental Results}\label{sec:app:experim}
In this section, we first provide the detailed parameters in Section~\ref{sec:numerics} and then discuss more experiments. In Section~\ref{sec:numerics}, our federated algorithm implementation is based on \cite{shaoxiong}, both hyper-representation and loss function tuning use batch size $64$ and Neumann series parameter $N=5$. We conduct 5 SGD/SVRG epoch of local updates in \fedinn and $\tau=1$ in \fedout. In \fednest, we use $T=1$, have 100 clients in total, and 10 clients are selected in each \fednest epoch.  
\begin{figure}[h]
    \begin{subfigure}{0.5\textwidth}
    \centering
        \begin{tikzpicture}
        \node at (0,0) [scale=0.5]{\includegraphics[]{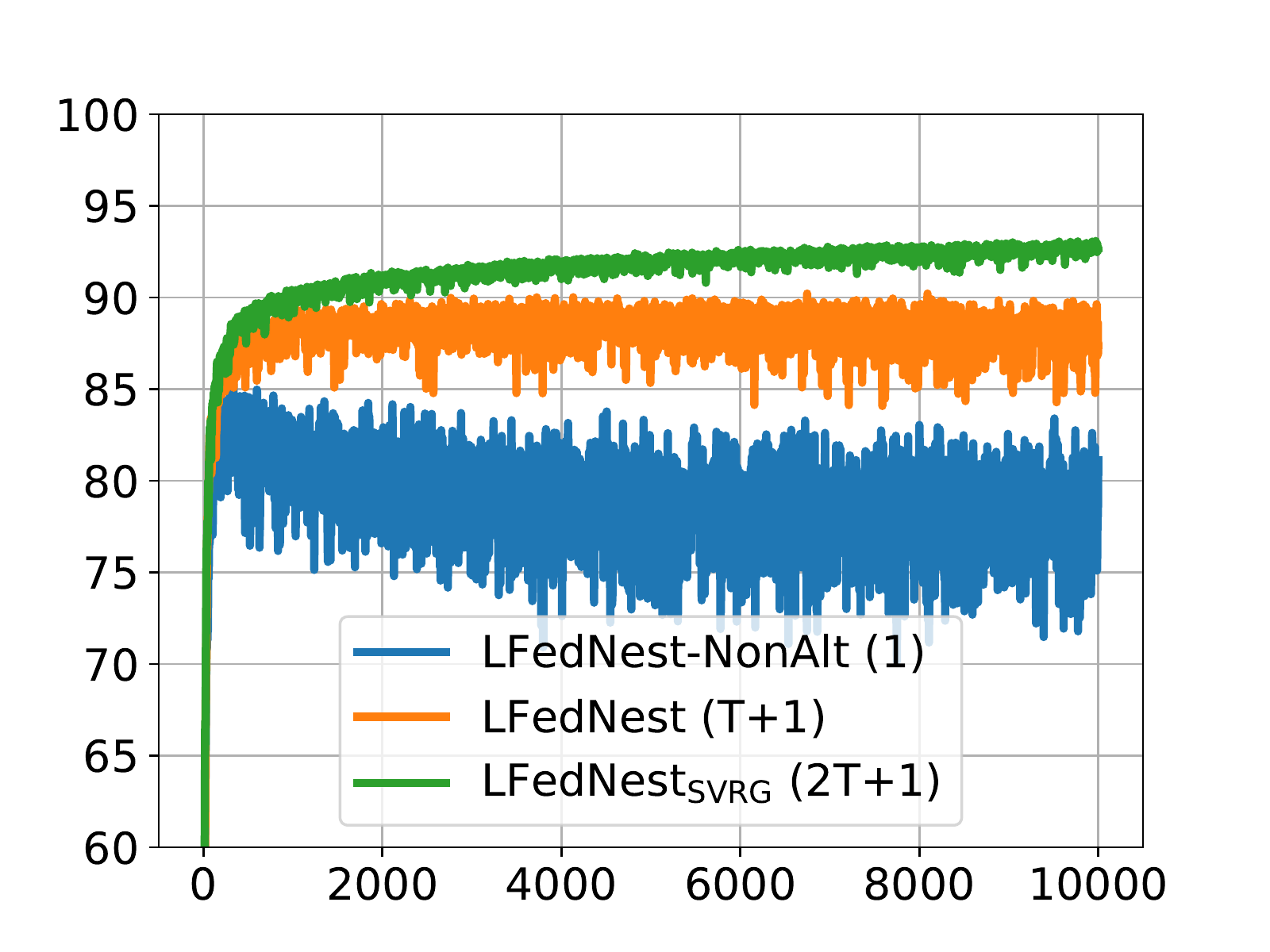}};
        \node at (-3.8,0) [rotate=90] {Test accuracy};
        \node at (0,-3.2) [] {Epoch};
        \end{tikzpicture}\caption{The test accuracy w.r.t to the algorithm epochs.}\label{fig:app:noniid_comm_a}
    \end{subfigure}
    \begin{subfigure}{0.5\textwidth}
    \centering
        \begin{tikzpicture}
        \node at (0,0) [scale=0.5]{\includegraphics[]{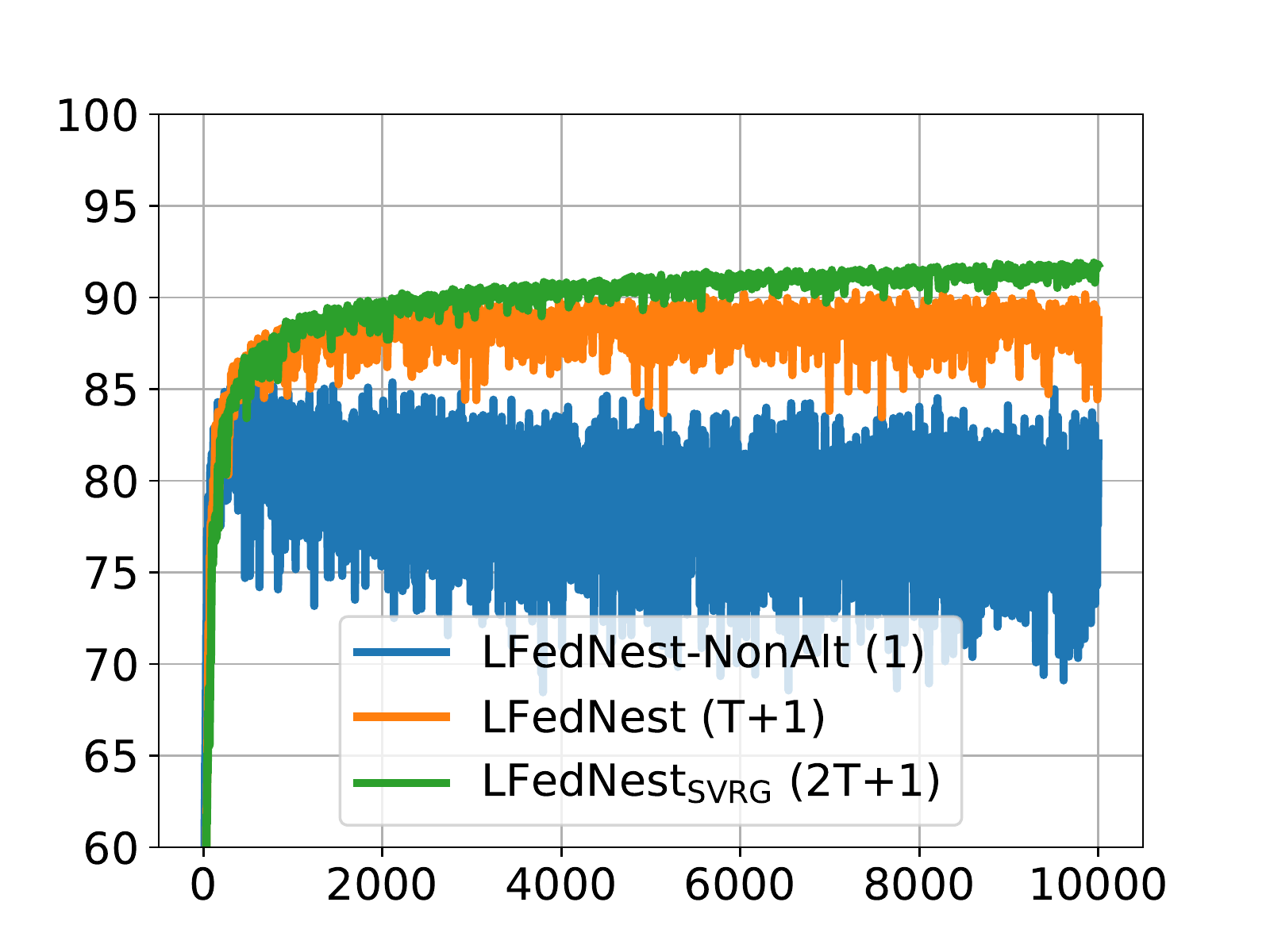}};
        \node at (0,-3.2) [] {\# of communications};
        \end{tikzpicture}\caption{The test accuracy w.r.t to the number of communications.}\label{fig:app:noniid_comm_b}
    \end{subfigure}\caption{Hyper representation experiment comparing \lfedblo, $\lfedblo_{\textnormal{\tiny{SVRG}}}$ and \textsc{LFedNest-NonAlt} on \textbf{non-i.i.d} dataset. The number in parentheses corresponds to communication rounds shown in Table~\ref{tabl:supp:methods}.}
    \label{fig:app:noniid_comm}
\end{figure}
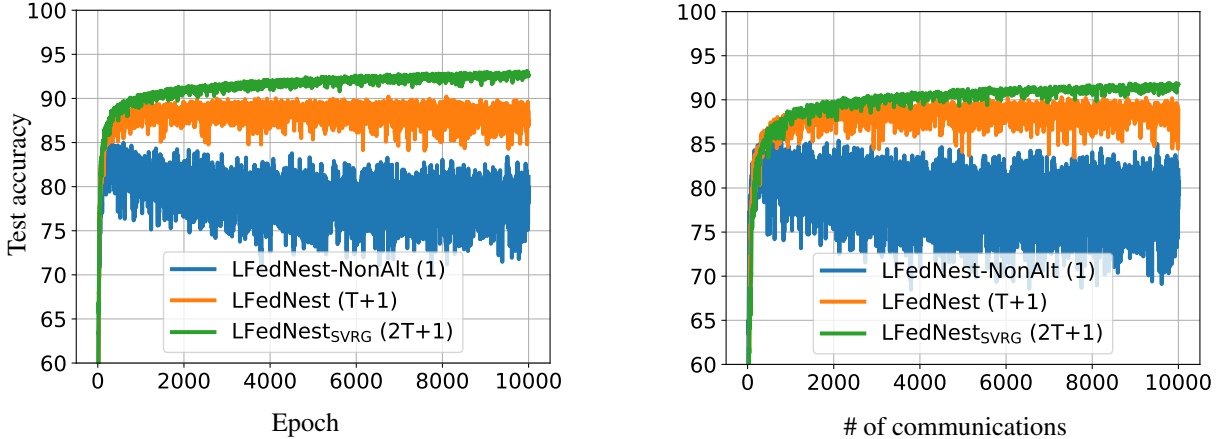
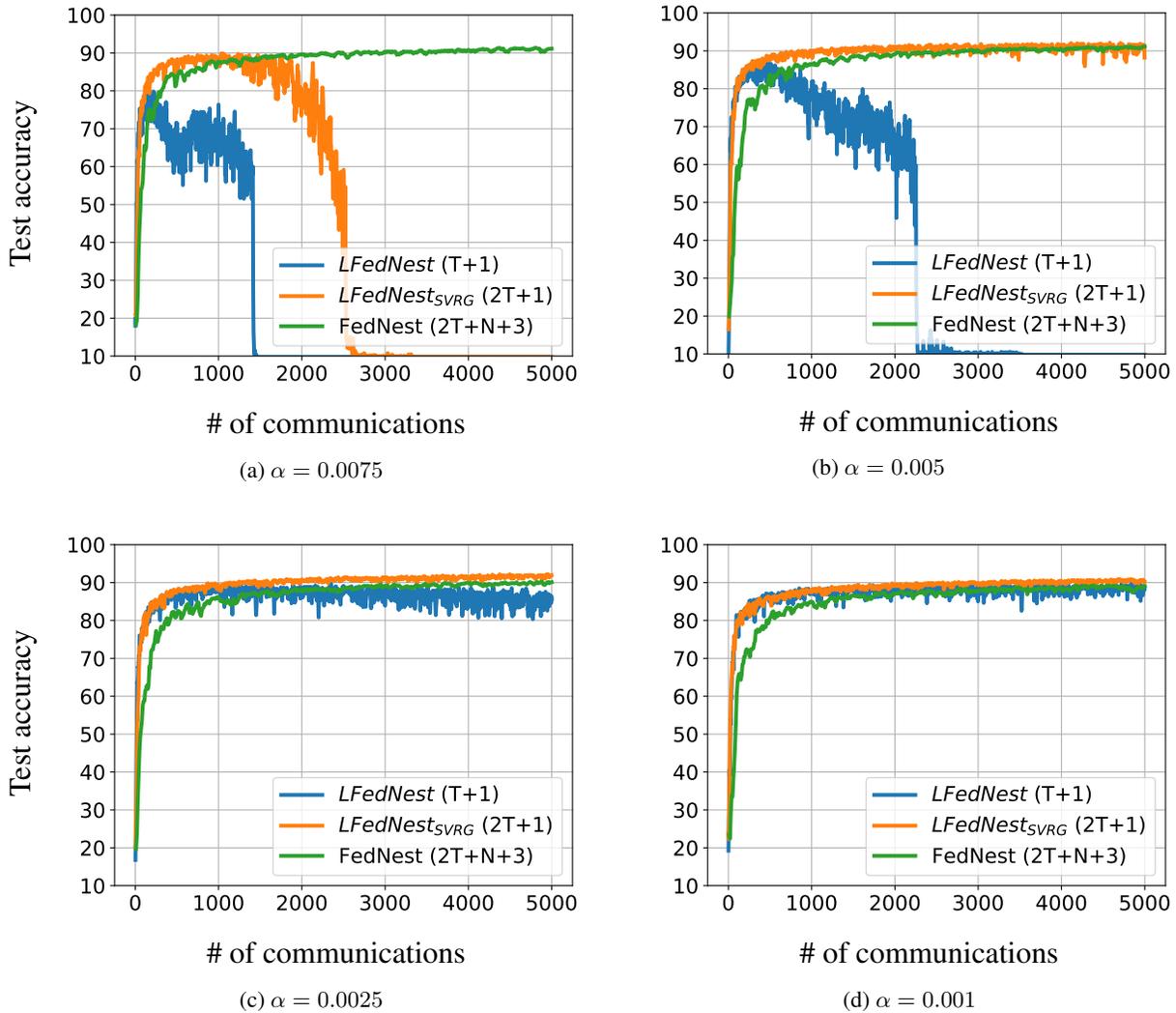
\begin{figure}[h]
    \begin{subfigure}{0.50\textwidth}
    \centering
        \begin{tikzpicture}
        \node at (0,0) {\includegraphics[scale=0.5]{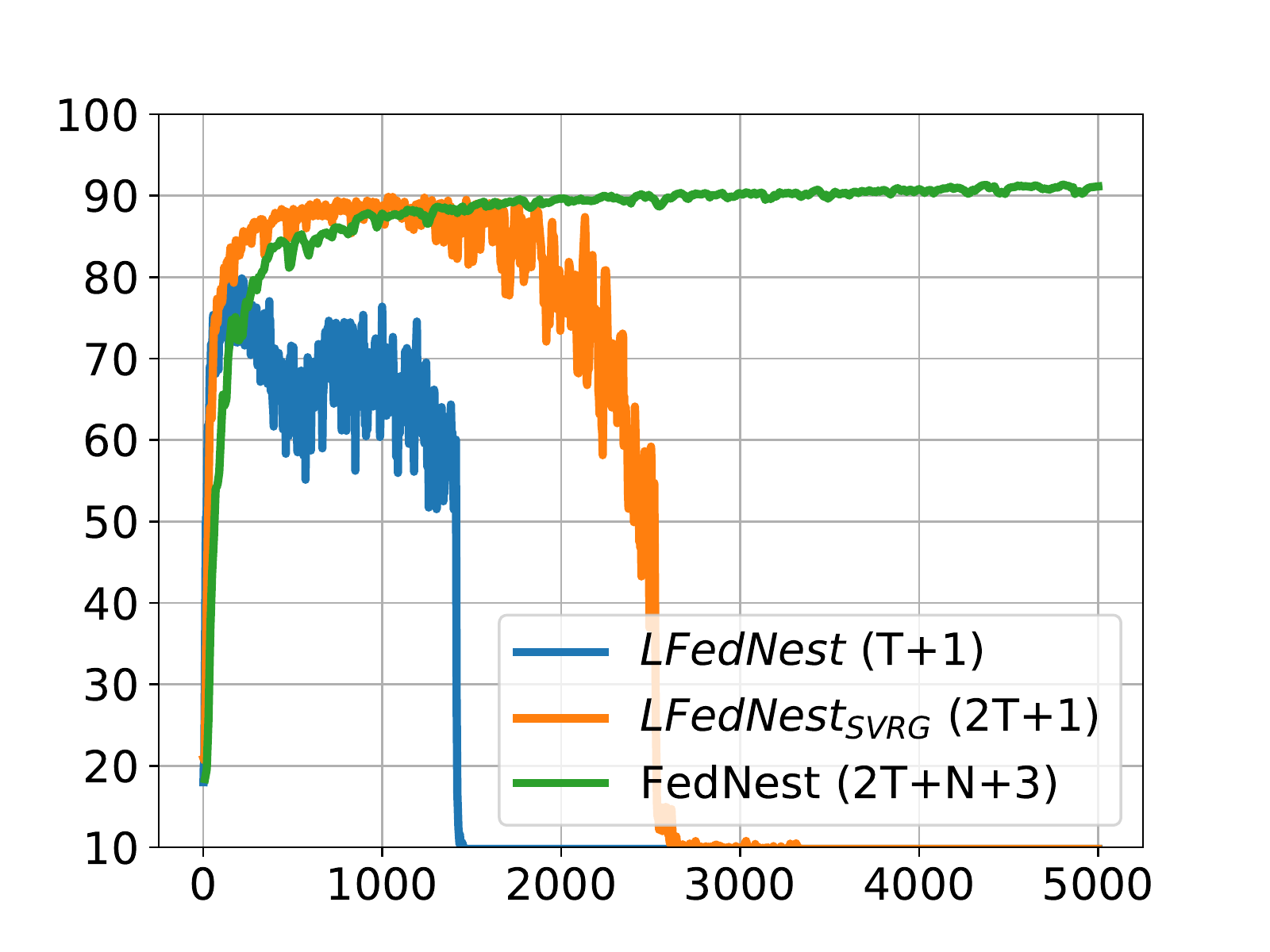}};
        \node at (-4.3,0) [rotate=90, scale=1.2] {Test accuracy};
        \node at (0,-3.3) [scale=1.2] {\# of communications};
        \end{tikzpicture}\caption{$\alpha=0.0075$}
    \end{subfigure}
    \begin{subfigure}{0.4\textwidth}
    \centering
        \begin{tikzpicture}
        \node at (0,0) [scale=0.5]{\includegraphics[]{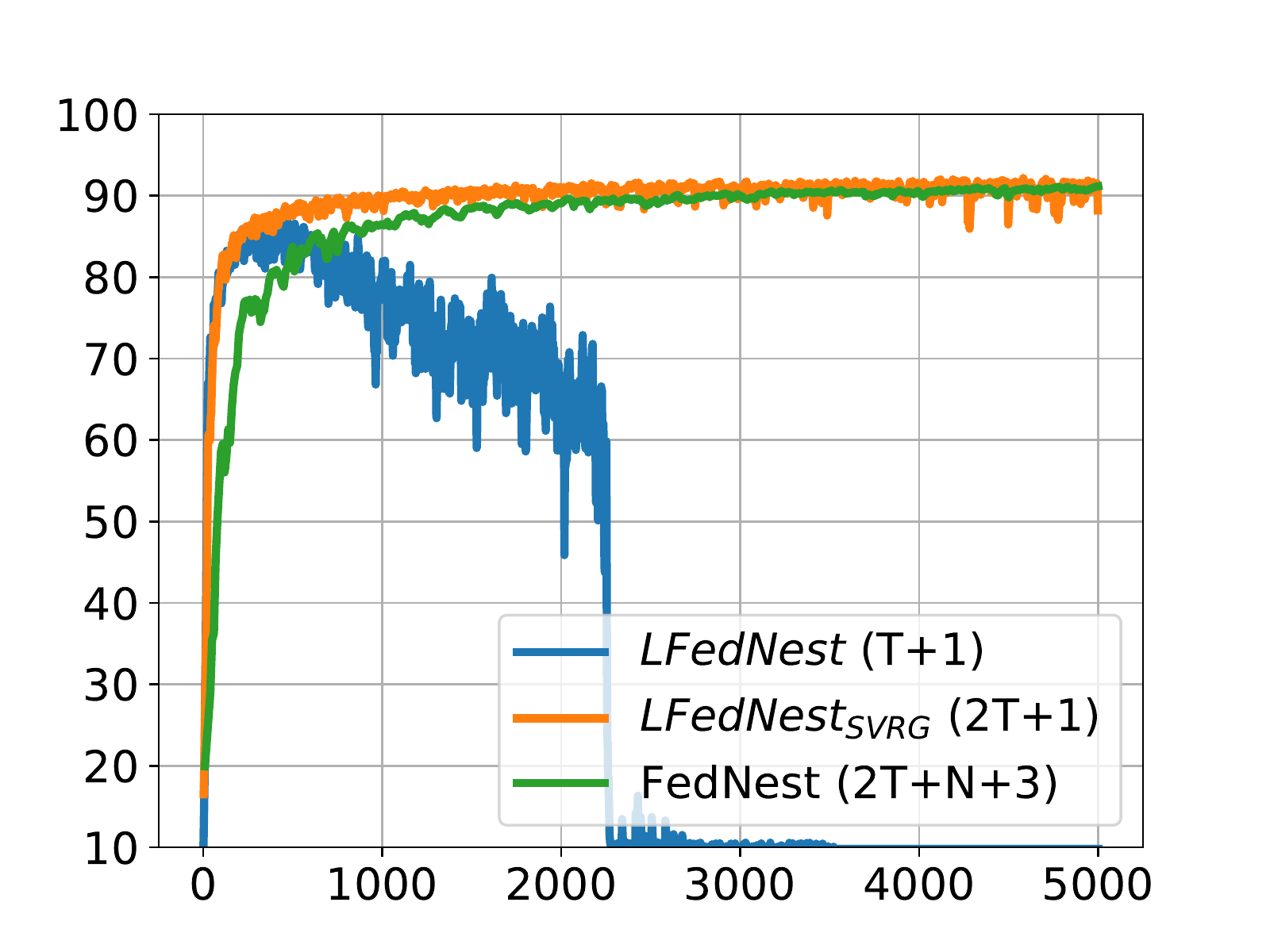}};
        \node at (0,-3.3) [scale=1.2] {\# of communications};
        \end{tikzpicture}\caption{$\alpha=0.005$}
    \end{subfigure}
    
    \begin{subfigure}{0.5\textwidth}
    \centering
        \begin{tikzpicture}
        \node at (0,0) [scale=0.5]{\includegraphics[]{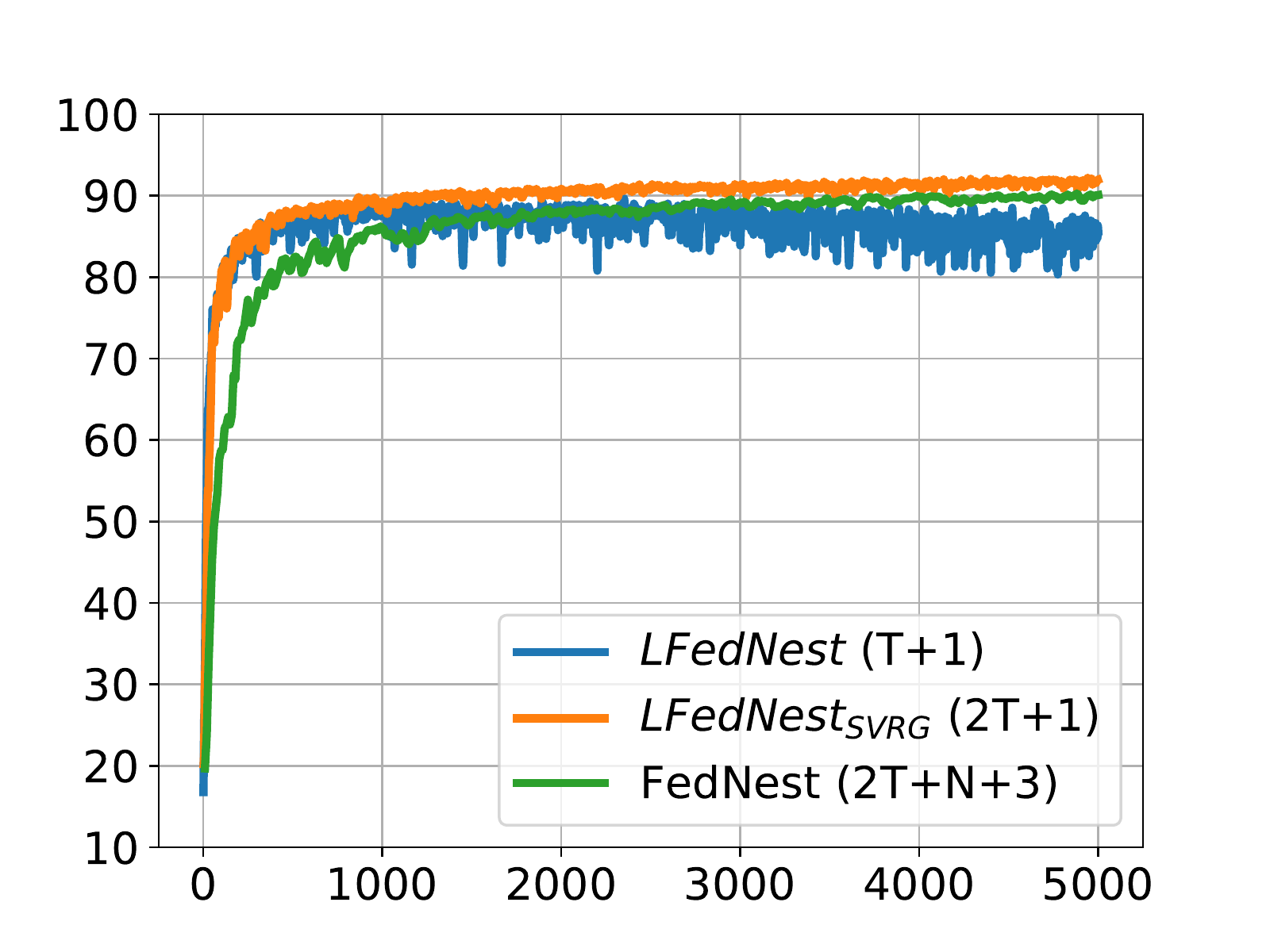}};
        \node at (-4.3,0) [rotate=90, scale=1.2] {Test accuracy};
        \node at (0,-3.3) [scale=1.2] {\# of communications};
        \end{tikzpicture}\caption{$\alpha=0.0025$}
    \end{subfigure}
    \begin{subfigure}{0.45\textwidth}
    \centering
        \begin{tikzpicture}
        \node at (0,0) [scale=0.5]{\includegraphics[]{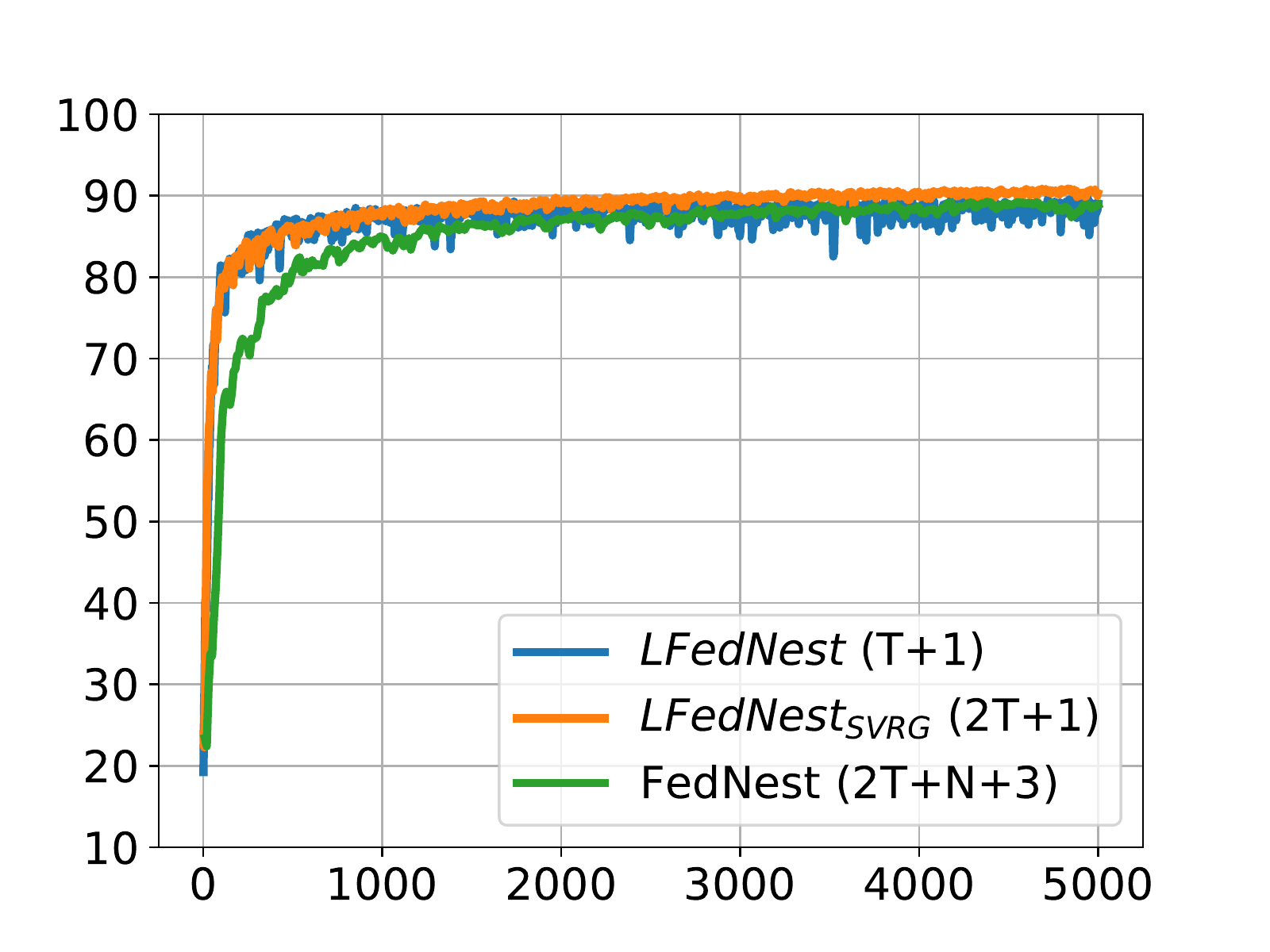}};
        \node at (0,-3.3) [scale=1.2] {\# of communications};
        \end{tikzpicture}\caption{$\alpha=0.001$}
    \end{subfigure}\caption{Learning rate analysis on \textbf{non-i.i.d.} data with respect to \textbf{\# of communications}.}\label{fig:app:lr_noniid}
\end{figure}
\subsection{The effect of the alternating between inner and outer global variables}\label{sec:joinvsalt} 
In our addition experiments, we investigate the effect of the alternating between inner and outer {global} variables $\m{x}$ and $\m{y}$. We use \textsc{LFedNest-NonAlt} to denote the training where each client updates their local $\y_{i}$ and then update local $\x_{i}$ w.r.t. local $\y_{i}$ for all $i \in\mc{S}$. Hence, the nested optimization is performed locally (within the clients) and the joint variable $[\x_{i},\y_{i}]$ is communicated with the server. One can notice that only one communication is conducted when server update global $\x$ and $\y$ by aggregating all $\x_{i}$ and $\y_{i}$. 

As illustrated in Figure~\ref{fig:app:noniid_comm}, the test accuracy of \textsc{LFedNest-NonAlt} remains around $80\%$, but both standard $\lfedblo$ and $\lfedblo_{\textnormal{\tiny{SVRG}}}$  achieves better performance. Here, the number of inner iterations is set to $T=1$. The performance boost reveals the necessity of both averaging and SVRG in \fedinn, where the extra communication makes clients more consistent.

\subsection{The effect of the learning rate and the global inverse Hessian} 
Figure~\ref{fig:app:lr_noniid} shows that on non-i.i.d. dataset, both SVRG and \fedout have the effect of stabilizing the training. Here, we set  $T=1$ and $N=5$. As we observe in (a)-(d), where the learning rate decreases, the algorithms with more communications are easier to achieve convergence. We note that \lfedblo successfully converges in (d) with a very small learning rate. In contrast, in (a), \fednest (using the global inverse Hessian) achieves better test accuracy in the same communication round with a larger learning rate.



\end{document}